\documentclass[11pt]{article}

\usepackage{etoolbox}
\newtoggle{colt}
\togglefalse{colt}
\newcommand{\colt}[1]{\iftoggle{colt}{#1}{}}
\newcommand{\arxiv}[1]{\iftoggle{colt}{}{#1}}

\arxiv{\usepackage[margin=1.0in]{geometry}}
\usepackage{graphicx}
\arxiv{\usepackage{amsmath,amssymb,amsfonts,amsthm}}
\arxiv{\usepackage{xspace}}
\usepackage{mathrsfs}
\usepackage{color}
\usepackage{enumitem}
\usepackage[small,bf]{caption}
\usepackage{makecell}
\newcommand{\mylabel}[2]{#2\def\@currentlabel{#2}\label{#1}}
\makeatother
\usepackage{xcolor}
\definecolor{dgreen}{rgb}{0,0.5,0}
\arxiv{\usepackage[colorlinks = true,
            linkcolor = blue,
            urlcolor  = blue,
            citecolor = dgreen,
            anchorcolor = blue]{hyperref}}

\usepackage{enumitem}
\setlist{nosep}

\newcommand{\nc}{\newcommand}
\newcount\Comments
\nc{\noah}[1]{\ifnum\Comments=1\textcolor{purple}{[Noah: #1]}\fi}
\nc{\DMO}{\DeclareMathOperator}
\arxiv{\nc\todo[1]{\textcolor{red}{[TODO: #1]}}}
\nc{\ReduceTree}{\texttt{ReduceTree}\xspace}
\nc{\PolyPriLearn}{\texttt{PolyPriLearn}\xspace}
\nc{\PPPLearn}{\texttt{PolyPriPropLearn}\xspace}
\nc{\GenericLearner}{\texttt{GenericLearner}\xspace}
\nc{\ReduceTreeReg}{\texttt{ReduceTreeReg}\xspace}
\nc{\RegLearn}{\texttt{RegLearn}\xspace}
\nc\coarse[2]{\{ {#1} \}_{#2}}
\nc\m[2]{m_{#1}(#2)}
\nc{\BR}{\mathbb{R}}
\nc{\BN}{\mathbb{N}}
\nc{\BZ}{\mathbb{Z}}
\nc{\bw}{\mathbf{w}}
\nc{\lng}{\langle}
\nc{\rng}{\rangle}
\DMO{\LFT}{left}
 \DMO{\RGT}{right}
\DMO{\Lap}{Lap}
\nc{\ep}{\varepsilon}
\nc{\ra}{\rightarrow}
\nc{\st}{\star}

\nc{\disc}[2]{\lfloor {#1} \rfloor_{#2}}
\nc{\MD}{\mathcal{D}}
\nc{\MM}{\mathcal{M}}
\nc{\MZ}{\mathcal{Z}}
\nc{\MU}{\mathcal{U}}
\nc{\MP}{\mathcal{P}}
\nc{\taumax}{(2+2\chi)(d+1)}
\nc{\rmax}{(d+1)}
\nc{\poly}{\mathrm{poly}}
\DMO{\treesum}{TreeSum}
\DMO{\lapsum}{LapSum}
\DMO{\checksum}{CheckSum}
\DMO{\height}{depth}
\DMO{\final}{final}
\DMO{\good}{good}

\nc{\MDts}{\MD_{\treesum}}
\nc{\MDls}{\MD_{\lapsum}}
\nc{\MDcs}{\MD_{\checksum}}
\nc{\MC}{\mathcal{C}}
\nc{\MT}{\mathcal{T}}
\nc{\MS}{\mathcal{S}}
\nc{\MX}{\mathcal{X}}
\nc{\MY}{\mathcal{Y}}
\nc{\MA}{\mathcal{A}}
\nc{\MB}{\mathcal{B}}
\nc{\MJ}{\mathcal{J}}
\nc{\MF}{\mathcal{F}}
\nc{\MG}{\mathcal{G}}
\nc{\hG}{\hat{\MG}}
\nc{\hg}{\hat{g}}
\nc{\MR}{\mathcal{R}}
\nc{\ML}{\mathcal{L}}
\nc{\MQ}{\mathcal{Q}}
\nc{\p}{\mathbb{P}}
\nc{\E}{\mathbb{E}}
\nc{\tablesize}{s}
\DMO{\Hist}{hist}
\nc{\hist}{\mathrm{hist}}
\nc{\soafilter}{\texttt{SOAFilter}\xspace}
\nc{\filterstep}{\texttt{FilterStep}\xspace}

\nc{\ba}{\mathbf{A}}
\nc{\bb}{\mathbf{B}}
\nc{\bx}{\mathbf{x}}
\nc{\by}{\mathbf{y}}
\nc{\bz}{\mathbf{z}}
\nc{\bs}{\mathbf{s}}
\nc{\bt}{\mathbf{t}}
\nc{\br}{\mathbf{r}}

\DMO{\sr}{sr}
\DMO{\Med}{Med}
\DMO{\Ber}{Ber}
\DMO{\Bin}{Bin}
\DMO{\Had}{Had}

\nc{\ME}{\mathcal{E}}
\DMO{\View}{View}
\nc{\B}{B}
\nc{\M}{M}
\nc{\ha}{\alpha}
\nc{\hk}{k}
\DMO{\pre}{pre}
\nc{\MH}{\mathcal{H}}
\DMO{\FO}{FO}
\DMO{\hb}{\beta}
\nc{\MW}{\mathcal{W}}
\nc{\MV}{\mathcal{V}}
\nc{\CS}{\mathscr{S}}
\nc{\CI}{\mathscr{I}}
\nc{\CQ}{\mathscr{Q}}
\nc{\CL}{\mathscr{L}}
\nc{\CM}{\mathscr{M}}
\nc{\CG}{\mathscr{G}}
\nc{\CR}{\mathscr{R}}
\nc{\irred}[3]{\CI_{{#1},{#2}}({#3})}

\DMO{\sfat}{sfat}
\DMO{\Ldim}{Ldim}
\DMO{\vc}{VCdim}
\DMO{\fat}{fat}

\nc{\K}{K}
\DMO{\ERR}{err}
\DMO{\NERR}{nerr}
\nc{\err}[2]{\ERR_{#1}({#2})}
\nc{\nerr}[2]{\NERR_{#1}({#2})}
\nc{\hMLp}{\hat\ML'}
\DMO{\soaa}{SOA}
\nc{\soa}[2]{\soaa_{#1}({#2})}
\nc{\soaf}[1]{\soaa_{#1}}
\nc{\gRes}[2]{\hat\MG({#1},{#2})}
\nc{\emp}{\hat P_{S_n}}
\DMO{\Red}{red}
\DMO{\Irred}{irred}
\nc{\Ired}{I^{\Red}}
\nc{\Iirred}{I^{\Irred}}

\DMO{\repp}{R}
\nc{\rep}[2]{R_{#1}({#2})}
\nc{\repf}[1]{\CR_{#1}}
\DMO{\REP}{rep}
\nc{\repl}[1]{\ML_{\REP}({#1})}
\nc{\mrepl}[2]{\ML_{#2}({#1})}
\arxiv{\nc{\dfn}[1]{{\it {#1}}}}
\colt{\nc{\dfn}[1]{\emph{{#1}}}}
\renewcommand{\^}[1]{^{(#1)}}
\DeclareMathOperator*{\argmax}{arg\,max}

\arxiv{
  \newtheorem{theorem}{Theorem}[section]
\newtheorem{corollary}[theorem]{Corollary}
\newtheorem{proposition}[theorem]{Proposition}
\newtheorem{lemma}[theorem]{Lemma}
\theoremstyle{definition}
\newtheorem{definition}{Definition}[section]
}
\newtheorem{claim}[theorem]{Claim}

\arxiv{
  \newcommand{\citep}{\cite}
  }

\makeatletter
\newtheorem*{rep@theorem}{\rep@title}
\newcommand{\newreptheorem}[2]{%
\newenvironment{rep#1}[1]{%
 \def\rep@title{#2 \ref{##1}}%
 \begin{rep@theorem}}%
 {\end{rep@theorem}}}
\makeatother

\newreptheorem{lemma}{Lemma}

\arxiv{
\usepackage[boxruled,vlined,nofillcomment,linesnumbered]{algorithm2e}
  }

\colt{\title[Differentially private regression]{Differentially Private Nonparametric Regression \\ Under a Growth Condition}}
\arxiv{\title{Differentially Private Nonparametric Regression \\ Under a Growth Condition}}
\arxiv{\author{Noah Golowich\thanks{MIT EECS, Cambridge, MA. Supported by a Fannie \& John Hertz Foundation Fellowship and an NSF Graduate Fellowship.}}}
\arxiv{\date{August 15, 2021}}
\colt{\usepackage{times}}

\colt{
\coltauthor{%
 \Name{Noah Golowich}\thanks{ Supported by a Fannie \& John Hertz Foundation Fellowship and an NSF Graduate Fellowship.} \Email{nzg@mit.edu}\\
 \addr MIT EECS
}}

\begin{document}

\maketitle

\begin{abstract}%
  Given a real-valued hypothesis class $\MH$, we investigate under what conditions there is a differentially private algorithm which learns an optimal hypothesis from $\MH$ given i.i.d.~data. Inspired by recent results for the related setting of binary classification \citep{alon_private_2019,bun_equivalence_2020}, where it was shown that \emph{online learnability} of a binary class is necessary and sufficient for its private learnability, \cite{jung_equivalence_2020} showed that in the setting of regression, online learnability of $\MH$ is \emph{necessary} for private learnability. Here online learnability of $\MH$ is characterized by the finiteness of its \emph{$\eta$-sequential fat shattering dimension, $\sfat_\eta(\MH)$}, for all $\eta > 0$. In terms of \emph{sufficient} conditions for private learnability, \cite{jung_equivalence_2020} showed that $\MH$ is privately learnable if $\lim_{\eta \downarrow 0} \sfat_\eta(\MH)$ is finite, which is a fairly restrictive condition. We show that under the relaxed condition $\liminf_{\eta \downarrow 0} \eta \cdot \sfat_\eta(\MH) = 0$, %
  $\MH$ is privately learnable, establishing the first nonparametric private learnability guarantee for classes $\MH$ with $\sfat_\eta(\MH)$ \emph{diverging} as $\eta \downarrow 0$. Our techniques involve a novel filtering procedure to output stable hypotheses for nonparametric function classes.
\end{abstract}

\colt{
}

\section{Introduction}
In recent years there has been an increased focus on the importance of protecting the privacy of potentially sensitive users' data on which machine learning algorithms are trained \citep{kearns_ethical_2019,nissim_bridging_2016}. The model of \emph{differentially private learning}~\citep{dwork2006calibrating, dwork_algorithmic_2013,vadhan_complexity_2017} provides a way to formalize the accuracy-privacy tradeoffs encountered. The vast majority of work in this area focuses on the setting of private classification, namely where we must predict a $\{0,1\}$-valued label for each data point $x$  \citep{kasiviswanathan2008what,beimel_bounds_2014,bun_differentially_2015,feldman_sample_2015,beimel_private_2014,bun_composable_2018,beimel_characterizing_2019,alon_private_2019,kaplan_privately_2020,bun_equivalence_2020,neel_heuristics_2019,bun_computational_2020}. Many natural machine learning problems, however, in application domains ranging from ecology to medicine \citep{dua_2019}, %
are phrased more naturally as \emph{regression} problems, where for each data point $x$ we must predict a real-valued label. In this paper we study this problem of differentially private regression for nonparametric function classes.

In the setting of differentially private binary classification, a major recent development \citep{alon_private_2019,bun_equivalence_2020} is the result that a hypothesis class $\MF$ consisting of binary classifiers is learnable with approximate differential privacy (Definition \ref{def:dp}) if and only if it is \emph{online learnable}, which is known to hold in turn if and only if the \emph{Littlestone dimension} of $\MF$ is finite \citep{littlestone_learning_1987,ben-david_agnostic_2009}. Such an equivalence, however, remains open for the setting of differentially private regression (this question was asked in \cite{bun_equivalence_2020}). 
The combinatorial parameter characterizing online learnability for regression %
is the \emph{sequential fat-shattering dimension} \citep{rakhlin_sequential_2015} (Definition \ref{def:sfat-dim}), which may be viewed as a scale-sensitive analogue of the Littlestone dimension. In one direction, \cite{jung_equivalence_2020} recently showed that if a class $\MF$ consisting of bounded real-valued functions is privately learnable, then it is online learnable, i.e., the sequential fat-shattering dimension of $\MF$ is finite at all scales. The other direction, namely whether online learnability of $\MF$ in the regression setting implies private learnability, remains open. %

\subsection{Results}
\label{sec:results}
In this paper, we make progress towards the question of whether online learnability in the regression setting implies private learnability by exhibiting a sufficient condition for private learnability in terms of the growth of the sequential fat-shattering dimension of a class. For input space $\MX$, a class $\MH$ consisting of hypotheses $h : \MX \ra [-1,1]$, and $\eta > 0$, let $\sfat_\eta(\MH)$ denote the $\eta$-sequential fat-shattering dimension of $\MH$ (Definition \ref{def:sfat-dim}). As in \cite{jung_equivalence_2020,rakhlin_sequential_2015}, we work with the \emph{absolute loss} to measure the error of a hypothesis $h : \MX \ra [-1,1]$: for a  distribution $Q$ supported on $\MX \times [-1,1]$, write $\err{Q}{h} := \E_{(x,y) \sim Q} \left[ |h(x) - y|\right]$. 
Our main result is as follows:
\begin{theorem}[Private nonparametric regression; informal version of Theorem \ref{thm:reglearn}]
  \label{thm:reglearn-informal}
  Let $\MH$ be a class of hypotheses $h : \MX \ra [-1,1]$. For any $\ep, \delta, \eta \in (0,1)$, for some $n = \frac{2^{\tilde O(\sfat_\eta(\MH))}}{\ep \eta^4}$, there is an $(\ep, \delta)$-differentially private algorithm which, given $n$ i.i.d.~samples from any distribution $Q$ on $\MX \times [-1,1]$, with high probability outputs a hypothesis $\hat h : \MX \ra [-1,1]$ so that
  $$
\err{Q}{\hat h} \leq \inf_{h \in \MH} \err{Q}{h} + O \left( \eta \cdot \sfat_\eta(\MH) \right).
  $$
\end{theorem}

As an immediate consequence, we obtain the following sufficient condition for \emph{private learnability} (Definition \ref{def:priv-learnability}) of a real-valued hypothesis class:
\begin{corollary}
  \label{cor:sfat-priv-learnability}
Suppose $\MH$ is a class of hypotheses $h : \MX \ra [-1,1]$ satisfying $\liminf_{\eta \downarrow 0} \eta \cdot \sfat_\eta(\MH) = 0$. Then $\MH$ is privately learnable.
\end{corollary}

Prior to our work, essentially the strongest private learnability guarantee for a nonparametric real-valued function class was \cite[Theorem 15]{jung_equivalence_2020}, which established that if the \emph{sequential pseudo-dimension} of a class $\MH$ is finite, then $\MH$ is privately learnable. However, the sequential pseudo-dimension of $\MH$ is lower-bounded by $\sfat_\eta(\MH)$ for all $\eta > 0$ (and in fact may be defined as $\lim_{\eta \downarrow 0} \sfat_\eta(\MH)$), and thus its boundedness implies that $\sfat_\eta(\MH)$ is bounded uniformly over $\eta > 0$. Thus Corollary \ref{cor:sfat-priv-learnability} is the first result to establish a private learnability result for a nonparametric family of classes $\MH$ with the property that $\sfat_\eta(\MH)$ can \emph{diverge} as $\eta \downarrow 0$. %
Even very simple function classes may have $\sfat_\eta(\MH)$ diverging as $\eta \downarrow 0$: for instance, the class of all single-dimensional linear functions $\MH = \{ x \mapsto ax + b : x,a,b \in \BR, \ \ |x| \leq 1, |a| \leq 1, |b| \leq 1\}$ satisfies $\sfat_\eta(\MH) = \Theta(\log(1/\eta))$. 

\paragraph{Techniques: new filtering procedure}
The proof of Theorem \ref{thm:reglearn-informal} proceeds in two stages. The first, fairly straightforward, step extends the algorithm \ReduceTree of \colt{\cite{ghazi_sample_2020}}\arxiv{\cite[Algorithm 1]{ghazi_sample_2020}} which was used to construct a private learner in the setting of binary classification for a class of finite Littlestone dimension; our analogue for regression is \ReduceTreeReg (Algorithm \ref{alg:reduce-tree}). From a technical standpoint, this involves extending the notion of \emph{irreducibility} to real-valued classes (Section \ref{sec:irreducibility}). However, unlike for the case of classification, \ReduceTreeReg alone is not sufficient for our purposes. In particular, \ReduceTreeReg leads, roughly speaking, to the following guarantee, which we informally call \dfn{weak stability}. Given any distribution $Q$ on $\MX \times [-1,1]$, there is a hypothesis $\sigma^\st : \MX \ra [-1,1]$ with low \arxiv{population }error on $Q$ so that given some number $n_0$ of i.i.d.~samples from $Q$, we can output a collection of hypotheses $\hat g_1, \ldots, \hat g_M$ so that for some $1 \leq j \leq M$ we have $\| \hat g_j - \sigma^\st \|_\infty \leq \eta$ with some not-to-small probability. Here $\eta > 0$ is a small value representing a lower bound on the desired error. In the setting of classification \cite{ghazi_sample_2020} showed the stronger guarantee (which we informally call \dfn{strong stability}) that $\hat g_j = \sigma^\st$ for some $j$. The guarantee of strong stability allowed them to perform multiple draws of $n_0$ samples and use a \emph{private sparse selection procedure} (an analogue of the stable histograms procedure of \cite{bun_simultaneous_2016} for the selection problem; see Section \ref{sec:sparse-selection}) to privately output a hypothesis with low population error.

The guarantee of weak stability is, however, insufficient to apply the sparse selection procedure. Thus we introduce a new procedure, called \soafilter (Algorithm \ref{alg:soafilter}) to upgrade the guarantee of weak stability provided by \ReduceTreeReg to one of strong stability; this is our main technical contribution. At a high level, \soafilter first ``filters out'' many candidate hypotheses $h : \MX \ra [-1,1]$ which are well-approximated by some hypothesis which is not filtered out (\filterstep, Algorithm \ref{alg:filterstep}). It then assigns each hypothesis $\hat g_j$, $1 \leq j \leq M$, as above, to some not-too-large collection of hypotheses which are not filtered out in a careful way that can ensure strong stability. Further details are provided in Section \ref{sec:soafilter}.
\subsection{Related work}
\paragraph{Differentially private regression} As discussed in the previous sections, the most closely related work to ours is \cite{jung_equivalence_2020}, which showed that finiteness of sequential pseudo-dimension (namely, $\lim_{\eta \downarrow 0} \sfat_\eta(\MH)$) is sufficient for private learnability. A number of other papers have studied special cases of regression: for instance, \cite{chaudhuri_privacy_2009} studied differentially private logistic regression, \cite{chaudhuri_differentially_2011,kifer_private_2012,bassily_private_2014} proved upper and lower bounds on the minimax rate of empirical misk minimization, which includes linear regression with general loss functions as a special case, \cite{wang_revisiting_2018} showed improved adaptive linear regression algorithms, \cite{cai_cost_2019} showed improved bounds on the minimax rate of linear regression with $\ell_2$ loss, \cite{bernstein_differentially_2019} studied differentially private Bayesian linear regression, and \cite{alabi_differentially_2020} studied differentially private linear regression in one dimension with the goal of optimizing performance on certain empirical datasets. Our work may be viewed as orthogonal to these papers, which study linear models in finite-dimensional spaces. While the growth condition $\lim_{\eta \downarrow 0} \eta \cdot \sfat_\eta(\MH) = 0$ is generally satisfied for such models,\footnote{For instance, if $\MX$ is the unit ball in $\BR^d$ with respect to the $\ell_2$ norm, and $\MH = \{ x \mapsto \lng w, x \rng : \| w \|_2 \leq 1 \}$, then $\sfat_\eta(\MH) \leq O(d \log 1/\eta)$ since $\MH$ has a pointwise (i.e., sup-norm) $\eta$-cover of size $O(1/\eta^d)$, i.e., pointwise metric entropy $O(d \log 1/\eta)$.} Theorem \ref{thm:reglearn-informal} does not improve upon any existing sample complexity bounds in these specialized settings (where in most cases optimal minimax rates are known). On the other hand, these existing works do not address the nonparametric setting where essentially no structure is imposed on the hypothesis class. %

\paragraph{Online learnability for nonparametric classes}
The sequential fat-shattering dimension was introduced by \cite{rakhlin_sequential_2015} and shown to characterize online learnability of a real-valued hypothesis class in \cite{rakhlin_online_2015}. It is a sequential analogue of the fat-shattering dimension, which was introduced in \cite{alon_scale_1997,kearns_efficient_1994} and was shown to characterize learnability in the i.i.d.~setting. A substantial amount of work has established bounds on the complexity of various learning tasks in terms of the fat-shattering dimension in the i.i.d.~setting (e.g., \cite{anthony_neural_2009,mendelson_rademacher_2002,bartlett_fat_1996}), and in terms of the sequential fat-shattering dimension and related complexity measures in the online setting (e.g., \cite{rakhlin_online_2014,rakhlin_equivalence_2017,foster_contextual_2018}). Our work begins such a study in the setting of differentially private learning (with i.i.d.~data).

\arxiv{\subsection{Overview of the paper}}
\colt{\paragraph{Overview of the paper}}
In Section \ref{sec:prelim} we give preliminaries. In Section \ref{sec:irreducibility} we introduce the notion of \emph{irreducibility} for the setting of regression. In Section \ref{sec:reducetree} we state the weak stability guarantee of the \ReduceTreeReg algorithm, which we then upgrade to one of strong stability in Section  \ref{sec:soafilter} using our ``filtering'' algorithm. Section \ref{sec:combine-informal} describes how to combine the components of the previous sections to prove Theorem \ref{thm:reglearn-informal}. Finally, we discuss some directions for future work in Section \ref{sec:conclusion}. Several lemma statements in the main body are stated informally; full and rigorous statements and proofs of all lemmas and theorems are given in the appendix.

\section{Preliminaries}
\label{sec:prelim}
\subsection{PAC learning \& discretization of hypothesis classes}
\label{sec:pac}
For a positive integer $K$, let $[K] := \{ 1, 2, \ldots, K\}$. 
Let $\MX$ denote an input space and $\MY$ denote an output space, which will always be a subset of the real line. We let $\MY^\MX$ denote the space of \dfn{hypotheses} on $\MX$, namely functions $h : \MX \ra \MY$. We are given a known \dfn{hypothesis class} $\MH \subset \MY^\MX$. For a distribution $Q$ on $\MX \times \MY$ and $h \in \MY^\MX$, let $\err{Q}{h} := \E_{(x,y) \sim Q} [|h(x) - y|]$ denote the \dfn{population error} of $h$.\footnote{Following \cite{jung_equivalence_2020,rakhlin_online_2015}, we work with the absolute loss; the results may readily be generalized to any other Lipschitz loss function.} A \dfn{dataset} $S_n \in (\MX \times \MY)^n$ is a tuple of $n$ elements of $\MX \times \MY$; for $Q$ as above, let $Q^n$ be the distribution of $S_n \in (\MX \times \MY)^n$ consisting of $n$ i.i.d.~draws from $Q$. For $(x,y) \in \MX \times \MY$, let $\delta_{(x,y)}$ denote the point measure at $(x,y)$, and for a dataset $S_n$ write $\hat Q_{S_n} := \frac 1n \sum_{i=1}^n \delta_{(x_i, y_i)}$ to denote the \dfn{empirical measure} for $S_n$. The \dfn{empirical error} of a hypothesis $h$ with respect to a dataset $S_n$ is defined to be $\err{\hat Q_{S_n}}{h}$. To avoid having to make technical measurability assumptions on $\MH, \MX$, we will assume throughout the paper that $\MH, \MX$ are countable (or finite).

Ultimately we aim to solve the following problem: for $\MY = [-1,1]$ and some small error $\eta_0$, find some $\hat h$ so that $\err{Q}{\hat h} \leq \inf_{h \in \MH} \left\{\err{Q}{h}\right\} + \eta_0$ given a sample $S_n \sim Q^n$. To streamline the analysis, though, we will often work with the \dfn{discretization} of the class $\MH$ \dfn{at scale $\eta$}, for some $\eta < \eta_0$: it is denoted $\disc{\MH}{\eta}$ and is obtained by dividing the interval $[-1,1]$ into $\lceil 2/\eta \rceil$ intervals each of length $2/\lceil 2/\eta \rceil \leq \eta$, and rounding $h(x)$, for each $h \in \MH, x \in \MX$, to the interval containing $h(x)$. A formal definition of $\disc{\MH}{\eta}$ is as follows: first, for a real number $y \in [-1,1]$, define $\disc{y}{\eta} \in [\lceil 2/\eta \rceil]$ as follows: %
\colt{
$\disc{y}{\eta} :=
 1 + \left\lfloor \frac{(y + 1)}{2} \cdot \lceil 2/\eta \rceil\right\rfloor$ for $y < 1$ and $\disc{y}{\eta} :=
  \lceil 2/\eta \rceil$ for $y = 1$.
  }
\arxiv{$$
\disc{y}{\eta} := \begin{cases}
 1 + \left\lfloor \frac{(y + 1)}{2} \cdot \lceil 2/\eta \rceil\right\rfloor &: y < 1 \\
  \lceil 2/\eta \rceil  &: y = 1 .
\end{cases}
$$}

Next, for $h \in \MH$, define $\disc{h}{\eta} \in [\lceil 2/\eta \rceil]^\MX$ by $\disc{h}{\eta}(x) = \disc{h(x)}{\eta}$, for $x \in \MX$. 
Then the discretization $\disc{\MH}{\eta} \subset {\{ 1,2, \ldots, \lceil 2/\eta \rceil \}}^\MX$ is defined as
$
\disc{\MH}{\eta} :=  \{ \disc{h}{\eta} : h \in \MH \}.
$
Moreover, the \dfn{discretization of a distribution $Q$} on $\MX \times [-1,1]$ \dfn{at scale $\eta$}, denoted $\disc{Q}{\eta}$, is defined to be the distribution of $(x, \disc{y}{\eta})$, where $(x,y) \sim Q$. In Appendix \ref{sec:uc-prelim-disc}, we show that for $h \in [-1,1]^\MX$, $\err{Q}{h}$ is roughly $\eta$ times  $\err{\disc{Q}{\eta}}{\disc{h}{\eta}}$, up to an additive error of $\pm O(\eta)$ (see (\ref{eq:disc-cont-rel})), and that we have the bound $\sfat_2(\disc{\MH}{\eta}) \leq \sfat_\eta(\MH)$ on the sequential fat-shattering dimension of $\disc{\MH}{\eta}$ at scale 2 (Lemma \ref{lem:fat-disc}). We will often write $\K := \lceil 2/\eta \rceil$ when considering the discretization of classes.

For any $h \in \BR^\MX$ write $\| h \|_\infty := \sup_{x \in \MX} |h(x)|$. 
\subsection{Differential privacy}
In this paper we study algorithms which satisfy \dfn{approximate differential privacy}, defined as follows:
\begin{definition}[Differential privacy, \cite{dwork2006calibrating}]
  \label{def:dp}
Fix sets $\MZ, \MW$, $n \in \BN$, $\ep, \delta \in (0,1)$, and suppose $\MW$ is countable. A randomized algorithm $A : \MZ^n \ra \MW$ is $(\ep, \delta)$-differentially private if the following holds: for any datasets $S_n, S_n' \in \MZ^n$ differing in a single example\footnote{Written out, we have $S_n = (z_1, \ldots, z_n)$ and $S_n' = (z_1, \ldots, z_{n-1}, z_n')$ for some $z_1, \ldots, z_n, z_n' \in \MZ$.} and for all subsets $\ME \subset \MW$, 
$
\Pr[A(S_n) \in \ME] \leq e^\ep \cdot \Pr[A(S_n') \in \ME] + \delta.
$
\end{definition}
Our goal is to solve the PAC learning problem (as introduced in Section \ref{sec:pac}) with an algorithm that is $(\ep, \delta)$-differentially private as a function of $S_n$. Typically in the differential privacy literature it is assumed that $\delta = n^{-\omega(1)}$. %
To this end, we make the following definition:
\begin{definition}[Private learnability]
  \label{def:priv-learnability}
  A class $\MH \subset [-1,1]^\MX$ is \dfn{privately (PAC) learnable} if for all $\ep, \delta, \eta, \beta \in (0,1)$, there is a bound $n = n_\MH(\ep, \delta, \eta, \beta)$ so that the following holds:
  \begin{itemize}
  \item There is an $(\ep, \delta)$-differentially private algorithm $A$ that takes as input a dataset $S_n \in (\MX \times [-1,1])^n$ and outputs some $A(S_n) \in [-1,1]^\MX$ so that: for any distribution $Q$ on $\MX \times [-1,1]$, with probability at least $1-\beta$ over $S_n \sim Q^n$, $\err{Q}{A(S_n)} \leq \inf_{h \in \MH} \left\{ \err{Q}{h} \right\} + \eta$. 
  \item For fixed $\ep, \eta, \beta$, the mapping $\delta\mapsto n_\MH(\ep, \delta, \eta, \beta)$ is $\delta^{-o(1)}$, i.e., for any constant $c > 0$ there is $\delta_0 > 0$ so that for $0 < \delta < \delta_0$ we have $n_\MH(\ep, \delta, \eta, \beta) \leq 1/\delta^c$. 
    \end{itemize}
  \end{definition}
  Our algorithms will satisfy the stronger guarantee that for fixed $\eta$ and $\MH$, the bound $n_\MH(\ep, \delta, \eta, \beta)$ grows polynomially in $1/\ep, \log(1/\delta), \log(1/\beta)$.

\subsection{Sequential fat-shattering dimension}
For a positive integer $\K$, we begin by defining $\K$-ary $\MX$-valued trees. For a positive integer $t$ and a sequence $k_1, k_2, \ldots, \in [\K]$, write $k_{1:t} = (k_1, \ldots, k_t)$. Let $k_{1:0}$ denote the empty sequence.
\begin{definition}[$\MX$-valued tree]
  \label{def:xv-tree}
  For $d, \K \in \BN$, a {\it $\K$-ary $\MX$-valued tree} of \dfn{depth} $d$ is a collection of partial functions $\bx_t : [\K]^{t-1} \ra \MX$, for $1 \leq t \leq d$, each with nonempty domain, so that for all $k_{1:t}$ that lie in the domain of $\bx_{t+1}$:
  \begin{enumerate}
  \item \label{it:parent} The sequence $k_{1:t-1}$ lies in the domain of $\bx_t$ (i.e., a node's parent is a node);
  \item \label{it:sibling} For all $k_t' \in [\K]$ the sequence $(k_1, \ldots, k_{t-1}, k_t')$ lies in the domain of $\bx_{t+1}$ (i.e., each non-root node has $\K-1$ siblings).
  \end{enumerate}
  We write $\bx := (\bx_1, \ldots, \bx_d)$. We say that the tree $\bx$ is \dfn{complete} if for each $t$ the domain of $\bx_t$ is all of $[\K]^{t-1}$. The tree $\bx$ is \dfn{binary} if it is 2-ary (i.e., $\K = 2$ in the above).

\end{definition}
  Associated with each sequence $k_{1:t} \in [\K]^t$ for which $k_{1:t-1}$ is in the domain of $\bx_t$, for some $1 \leq t \leq d$, is a \dfn{node} of the tree. We say that this node is a \dfn{leaf} if $k_{1:t}$ is not in the domain of $\bx_{t+1}$ (or if $t = d$). Moreover, for any non-leaf node associated with $k_{1:t} \in [\K]^t$, we say that it is \dfn{labeled} by the point $\bx_{t+1}(k_{1:t}) \in \MX$. For any such node $v$, the nodes associated with $(k_1, \ldots, k_t, k_{t+1}')$, for each choice of $k_{t+1}' \in [\K]$ are the \dfn{children} of $v$; we say that the coresponding edge between $v$ and each child is \dfn{labeled} by $k_{t+1}'$. Note that a node is a leaf if and only if it has no children. Note also that any non-leaf node has exactly $\K$ children.

  Fix $\alpha > 0$. A complete binary (i.e., 2-ary) $\MX$-valued tree $\bx$ of depth $d$ is \dfn{$\alpha$-shattered} by a class $\MF \subset \BR^\MX$ if there is a complete $\BR$-valued binary tree $\bs$ of depth $d$ so that for all $k_{1:d} \in \{1,2\}^{d}$, there is some $f \in \MF$ so that
  $
  \left(3 - 2k_t \right) \cdot (f(\bx_t(k_{1:t-1})) - \bs_t(k_{1:t-1})) \geq \alpha/2
$
for all $1 \leq t \leq d$.
The tree $\bs$ is called the \dfn{witness to shattering}. 
\begin{definition}[Sequential fat-shattering dimension]
  \label{def:sfat-dim}
The \dfn{$\alpha$-sequential fat shattering dimension} of a class $\MF$, denoted $\sfat_\alpha(\MF)$, is the greatest positive integer $d$ so that there is an $\MX$-valued binary tree of depth $d$ which is $\alpha$-shattered by $\MX$. As a convention, if $\MF$ is empty, we write $\sfat_\alpha(\MF) = -1$.
\end{definition}

\subsection{Sparse selection procedure}
\label{sec:sparse-selection}
A key building block in our private learning protocols is a differentially private algorithm for the following sparse selection problem from \cite{ghazi_differentially_2020}. For $m, s \in \BN$, the \dfn{$(m, s)$-sparse selection} problem is defined as follows: there is some (possibly infinite) universe $\MU$, and $m$ users. Each user $i \in [m]$ is given some set $\MS_i \subset \MU$ of size $|\MS_i| \leq s$. An algorithm is said to solve the $(m,s)$-sparse selection problem \dfn{with additive error $\eta > 0$} if, given as input the sets $\MS_1, \ldots, \MS_m$, it outputs some universe element $\hat u \in \MU$ so that $| \{ i : \hat u \in \MS_i \}| \geq \max_{u \in \MU} | \{ i : u \in \MS_i \}| - \eta$. We will use the following proposition, which shows that the sparse selection problem can be solved privately with error independent of the size of the universe $\MU$:
\begin{proposition}[\cite{ghazi_differentially_2020}, Lemma 36]
  \label{prop:sparse-selection}
For $\ep ,\delta, \beta \in (0,1)$, there is an $(\ep, \delta)$-differentially private algorithm that, given an input dataset to the $(m,s)$-sparse selection problem, outputs a universe element $\hat u$ such that with probability at least $1-\beta$, the (additive) error of $\hat u$ is $O\left( \frac{1}{\ep} \log \left( \frac{ms}{\ep \delta \beta} \right) \right)$. 
\end{proposition}
In our application of Proposition \ref{prop:sparse-selection}, the universe $\MU$ will be the set of hypotheses $[\K]^\MX$ and so the output of the sparse selection procedure will be a private hypothesis; see Section \ref{sec:combine-informal}.
\section{Irreducibility for real-valued classes}
\label{sec:irreducibility}
In this section we introduce the concept of \emph{irreducibility} in the context of regression, extending the work of \cite{ghazi_sample_2020}, which defined irreducibility for $\{0,1\}$-valued classes in the context of classification. Throughout this section, we will fix a positive integer $\K$ and an input space $\MX$, and consider a class $\MF \subset [\K]^\MX$ so that $\sfat_2(\MF)$ is finite. As discussed in Section \ref{sec:pac}, $\MF$ will arise in the proof of Theorem \ref{thm:reglearn-informal} as the $\eta$-discretization of a real-valued class $\MH \subset [-1,1]^\MX$, where $\K = \lceil 2/\eta \rceil$. %
We begin with the following definition which will simplify our notation.
\begin{definition}[Ancestor set, depth of a node]
  Let $\bx$ be a $\MX$-valued tree of depth $d$, and $v$ be a node of $\bx$ corresponding to the tuple $(k_1, \ldots, k_t) \in [\K]^t$. The \dfn{ancestor set} of $v$, denoted $\ba(v)$, is the subset of $\MX \times [\K]$ given by
\arxiv{$$
  \ba(v) := \{ (\bx_1, k_1), (\bx_2(k_1), k_2), \ldots, (\bx_t(k_{1:t-1}), k_t) \}.
  $$}
\colt{$
  \ba(v) := \{ (\bx_1, k_1), (\bx_2(k_1), k_2), \ldots, (\bx_t(k_{1:t-1}), k_t) \}.
$}
The integer $t$ is referred to as the \dfn{depth} of the node $v$ and is denoted as $t = \height(v)$. 
\end{definition}
In the context of the above definition, note that $t$ is an upper bound on the size of $\ba(v)$. It is possible that for some distinct $s,s'$ we could have $(\bx_s(k_{1:s-1}), k_s) = (\bx_{s'}(k_{1:s'-1}), k_{s'})$ and hence the size of $\ba(v)$ could be strictly less than $t$. Note that $\ba(v)$ depends on the tree $\bx$, though we do not explicitly notate this dependence since the tree $\bx$ will always be clear from the node $v$.

For any $x \in \MX, k \in [\K]$, set
$
\MF|_{(x,k)} := \{ f \in \MF : f(x) = k \}.
$
For a set $S = \{ (x_1, k_1), \ldots, (x_\ell, k_\ell) \}$, similarly set
$
\MF|_S := \bigcap_{i \in [\ell]} \MF|_{(x_i, k_i)} = \{ f \in \MF : f(x_i) = k_i \ \  \forall i \in [\ell]\}.
$
\begin{definition}[Irreducibility]
  \label{def:irreducibility}
  For an integer $\ell \geq 1$, a class $\MF \subset [\K]^\MX$ is \dfn{$\ell$-irreducible} if for any $K$-ary $\MX$-valued tree $\bx$ of depth at most $\ell$,
the tree $\bx$ has some leaf $v$ so that $\sfat_2(\MF|_{\ba(v)}) = \sfat_2(\MF)$.
\end{definition}
We say $\MF$ is \dfn{irreducible} if it is 1-irreducible. For convenience we will say that all classes are {\it 0-irreducible} (i.e., 0-irreducibility is vacuous); thus $\ell$-irreducibility makes sense for all non-negative integers $\ell$.
Note that $\ell$-irreducibility implies $\ell'$-irreducibility for $\ell' < \ell$. The following simple, though fundamental, lemma forms the basis of a number of the stability-type results we show:
\begin{lemma}
   \label{lem:consec-k}
  Suppose $\MG \subset [\K]^\MX$ is irreducible. Then there are at most 2 values of $k \in [\K]$ so that $\sfat_2(\MG|_{(x,k)}) = \sfat_2(\MG)$, and if there are 2 values, they differ by 1.
\end{lemma}
Using Lemma \ref{lem:consec-k}, we next define the \emph{SOA hypothesis} associated to an irreducible hypothesis class $\MG \subset [\K]^\MX$, which assigns to each $x$ some element $k \in [\K]$ maximizing $\sfat_2(\MG|_{(x,k)})$. Such SOA hypotheses were crucial in the development of private learning algorithms for classification \citep{ghazi_sample_2020,bun_equivalence_2020}, and they will likewise play a major role in this paper. %

\begin{definition}[SOA hypothesis]
  \label{def:soa}
  Fix an irreducible class $\MG \subset [\K]^\MX$. Define $\soaf{\MG} \in [\K]^\MX$ as follows: for each $x \in \MX$, $\soa{\MG}{x}$ is equal to some $k \in [\K]$ so that $\sfat_2(\MG|_{(x,k)}) = \sfat_2(\MG)$. By Lemma \ref{lem:consec-k}, there are at most 2 such values of $k$. If there are 2 such values of $k$, i.e., there is some $k \in [\K-1]$ so that $\sfat_2(\MG|_{(x,k)}) = \sfat_2(\MG|_{(x,k+1)}) = \sfat_2(\MG)$, the tie is broken as follows:
  \begin{itemize}
  \item If there is some $\ell$ so that $\MG|_{(x,k)}$ is $\ell$-irreducible but $\MG|_{(x,k+1)}$ is not, then set $\soa{\MG}{x} = k$; vise versa, if $\MG|_{(x,k+1)}$ is $\ell$-irreducible but $\MG|_{(x,k)}$ is not, then $\soa{\MG}{x} = k+1$.
  \item If the previous item does not hold, then set $\soa{\MG}{x} = k$.
  \end{itemize}
\end{definition}

Lemma \ref{lem:soa-stability} below is similar to \cite[Lemma 4.3]{ghazi_sample_2020} proved in the setting of classification and is the basis for the ``weak stability'' results presented in Section \ref{sec:reducetree}. The key difference between Lemma \ref{lem:soa-stability} and \cite[Lemma 4.3]{ghazi_sample_2020} is that in the setting of classification, it can be established that $\soaf{\MH} = \soaf{\MG}$, whereas for the setting of regression we only get ``approximate equality'', i.e., $\| \soaf{\MH} - \soaf{\MG} \|_\infty \leq 1$.
\begin{lemma}
  \label{lem:soa-stability}
Suppose $\MH \subset \MG$, $\sfat_2(\MH) = \sfat_2(\MG)$, and that $\MH$ is irreducible. Then it holds that $\| \soaf{\MH} - \soaf{\MG} \|_\infty \leq 1$. %
\end{lemma}
Following \cite{ghazi_sample_2020}, we say that $\MG \subset \MF$ is a \dfn{finite restriction subclass} if it holds that $\MG = \MF|_{(x_1, y_1), \ldots, (x_M, y_M)}$ for some $(x_1, y_1), \ldots, (x_M, y_M) \in \MX \times [\K]$. Note that if $\MX$ is countable, the set of finite restriction subclasses of $\MF$ is countable. (The set of all subclasses of $\MF$ may be uncountable; thus, by considering finite restriction subclasses we avoid having to deal with uncountable sets.)

\section{The \ReduceTreeReg algorithm: obtaining weak stability}
\label{sec:reducetree}
In this section we state the weak stability guarantee afforded by the algorithm \ReduceTreeReg (Algorithm \ref{alg:reduce-tree}). Overall the algorithm and its analysis is very similar to that of the \ReduceTree algorithm of \cite{ghazi_sample_2020}, so all details are given in the appendix. (Some modifications from \cite{ghazi_sample_2020} are necessary, though, for instance because a class with finite sequential fat-shattering dimension does not immediately give rise to one of comparable Littlestone dimension; thus we cannot use the results of \cite{ghazi_sample_2020} in a black-box manner.) As in Section \ref{sec:irreducibility} we work with the discretized problem: given $\MX, \K$ a class $\MF \subset [\K]^\MX$ with $d := \sfat_2(\MF) \ll \K$, $n \in \BN$, and a distribution $P$ on $\MX \times [\K]$, the algorithm \ReduceTreeReg receives a dataset $S_n \in (\MX \times [\K])^n$ drawn from $P^n$. It also takes as input a parameter $\alpha_1$, for which it is assumed that $\alpha_1 - 3d \geq \inf_{f \in \MF} \err{P}{f}$. The guarantee of \ReduceTreeReg is stated (informally) as follows:
\begin{lemma}[Weak stability; informal version of Lemmas \ref{lem:approx-stability} and \ref{lem:s-size}]
  \label{lem:weak-stability-informal}
  Suppose $\MF, P, \alpha_1$ are given as described above. Then there are $d+1$ hypotheses $\sigma_1^\st, \ldots, \sigma_{d+1}^\st : \MX \ra [\K]$, depending only on $\MF, P$,\footnote{Each of the hypotheses $\sigma_i^\st$  is of the form $\soaf{\MG}$ for some $\MG \subset \MF$ which is $\ell'$-irreducible for sufficiently large $\ell'$. } so that, for sufficiently large $n$, given as input a dataset $S_n \sim P^n$, \ReduceTreeReg outputs a set $\hat \MS \subset [\K]^\MX$ of size $|\hat \MS| \leq \K^{2^{\tilde O(d)}}$ so that:
  \begin{itemize}
  \item With high probability, for some $t \in [d+1]$ and $\hat g \in \hat \MS$, it holds that $\| \hat g - \sigma_t^\st \|_\infty \leq 5$.
  \item With high probability, all $\hat g \in \hat \MS$ satisfy $\err{P}{\hat g} \leq \alpha_1$.
  \end{itemize}
\end{lemma}
Note that Lemma \ref{lem:weak-stability-informal} only guarantees that $\| \hat g - \sigma_t^\st \|_\infty \leq 5$ with high probability, which we informally refer to as \emph{weak stability}; in order to apply Proposition \ref{prop:sparse-selection} to obtain a private learning algorithm, we would need that $\| \hat g - \sigma_t^\st \|_\infty = 0$ (which we refer to as \emph{strong stability}).  In the following section we discuss how to upgrade the guarantee of weak stability to one of strong stability.

\section{The algorithm \soafilter: from weak to strong stability}
\label{sec:soafilter}
In this section we introduce the algorithm \soafilter and state its main guarantee. As in Section \ref{sec:irreducibility}, we continue on working with the discretized version of the problem, i.e., $\MX, \K$ are fixed, $\MX$ is countable, and we are given some countable hypothesis class $\MF \subset [\K]^\MX$, known to the algorithm, distribution $P$ on $\MX \times [\K]$, unknown to the algorithm, and the goal is to find $f \in \MF$ minimizing $\err{P}{f}$. We will write $d := \sfat_2(\MF)$ throughout this section. The error bounds we establish in this section will grow as $O(d)$ (see, e.g., item \ref{it:informal-l-close} below); thus, if $\MF$ arises as a discretization $\MF = \disc{\MH}{\eta}$, in order to ensure the error in the non-discretized version of the problem, which is $O(d) / \K$, is small, we work in the regime $d \ll \K$. %
Recalling that $\K = \lceil 2/\eta \rceil$ for a discretization scale $\eta$ (Section \ref{sec:pac}) and so $d / \K = O(\eta \cdot d) \leq O(\eta \cdot \sfat_\eta(\MH))$ (Lemma \ref{lem:fat-disc}), the growth condition $O(\sfat_\eta(\MH) \cdot \eta) \ra 0$  arises as a sufficient condition for $d/\K \ra 0$. 

We address the following problem: suppose there is some class $\MG \subset \MF$ which is $\ell$-irreducible for some large $\ell \in \BN$, and for which $\err{P}{\soaf{\MG}}$ is known to be small. Unfortunately, the algorithm does not know $\soaf{\MG}$; instead, we only know of some procedure (formalized as part of \ReduceTreeReg described in Section \ref{sec:reducetree}) to produce, given i.i.d.~samples from $P$, a collection of hypotheses $\hat g_1, \hat g_2, \ldots, \hat g_M \in [\K]^\MX$, so that with some positive probability (lower bounded by $1/O(d)$) at least one such hypothesis $\hat g_i$ satisfies $\| \soaf{\MG} - \hat g_i \|_\infty \leq \chi$ for some small positive constant $\chi$.\footnote{We were able to establish such a guarantee for $\chi = 5$ in Section \ref{sec:reducetree} (see Lemma \ref{lem:weak-stability-informal}).} Recall that we call this guarantee \dfn{weak stability}. We can repeat this procedure many times with disjoint samples from $P$, thus generating many hypotheses $\hat g_i$ satisfying $\| \soaf{\MG} - \hat g_i \|_\infty \leq \chi$, with the goal of applying the sparse selection procedure of Proposition \ref{prop:sparse-selection}.  However, in order to do so, we would need that for a given draw of $(\hat g_1, \ldots, \hat g_M)$, some hypothesis $\hat g_i$ is {\it equal} to $\soaf{\MG}$ with positive probability, i.e., $\chi = 0$. Since we wish to avoid dependence on $|\MX|$ in our sample complexity bounds (e.g., if $\MX$ is infinite), given only the guarantee that $\| \soaf{\MG} - \hat g_i \|_\infty \leq \chi$ for some $\chi > 0$, %
it is nontrivial to privately output some hypothesis close to $\soaf{\MG}$. %

In this section we overcome this challenge as follows: given $\MG$ as above and $\hat g \in [\K]^\MX$ with $\| \soaf{\MG} - \hat g \|_\infty \leq \chi$, we introduce an algorithm, \soafilter (Algorithm \ref{alg:soafilter}), which outputs some set $\CR_{\hat g}$ consisting of many subclasses $\ML \subset [\K]^\MX$, of size bounded above as a function of $d$ and $\K$ (in particular, $|\CR_{\hat g}| \leq \K^{d^{O(d)}}$), so that the following two properties hold, which we refer to informally as \emph{strong stability} (see Lemma \ref{lem:soafilter-lstar} for a formal statement):
\begin{enumerate}
\item Each $\ML \in \CR_{\hat g}$ is irreducible and satisfies $\| \soaf{\ML} - \hat g \|_\infty \leq O(\chi \cdot d)$.\label{it:informal-l-close}
\item For some irreducible $\ML^\st \subset \MF$ depending only on $\MG$, we have $\ML^\st \in \CR_{\hat g}$.\label{it:informal-lstar-belongs}
\end{enumerate}
Given a collection of hypotheses $\hat g_1, \ldots, \hat g_M \in [\K]^\MX$ as above, if we run \soafilter on each of the hypotheses $\hat g_i$, then the set $\hat \CR := \CR_{\hat g_1} \cup \cdots \cup \CR_{\hat g_M}$ is of bounded size (namely, at most $M \cdot \K^{d^{O(d)}}$), and as long as $\| \soaf{\MG} - \hat g_i \|_\infty \leq \chi$ for some $i \in [M]$ we have that $\ML^\st \in \hat \CR$ (item \ref{it:informal-lstar-belongs}) and $\| \soaf{\ML^\st} - \soaf{\MG} \|_\infty \leq O(\chi \cdot d)$ (item \ref{it:informal-l-close}). These properties (in particular, that $\hat \CR$ contains \emph{exactly} the class $\ML^\st$) are sufficient to apply the sparse selection procedure of Proposition \ref{prop:sparse-selection}, and thus obtain a private learning algorithm for $\MF$.
In Section \ref{sec:filterstep}, we describe a subroutine of \soafilter, which we call \filterstep; we then describe \soafilter in Section \ref{sec:redtree-soafilter}. %

\subsection{\filterstep algorithm}
\label{sec:filterstep}
A challenge in achieving a strong stability guarantee as explained in the above paragraphs is that the class $\MF$ could consist of too many functions with small oscillatory behavior: in particular, suppose that $\MF = \{ f : f(x) \in \{1,2\} \ \forall x \in \MX \}$, so that $\sfat_2(\MF) = 0$. Suppose that $\soaf{\MG}$ and $\hat g$ are arbitrary functions taking values in  $\{1,2\}$; then $\| \hat g - \soaf{\MG} \|_\infty \leq 1$. Moreover, each irreducible subclass $\ML \subset \MF$ satisfies $\| \soaf{\ML} - \hat g \|_\infty \leq 1$. Since we aim to have $| \CR_{\hat g} | \leq \K^{d^{O(d)}}$, and yet the number of irreducible subclasses $\ML \subset \MF$ could be much larger than this quantity, we will have to narrow down the set of subclasses $\ML$ which can be added to $\CR_{\hat g}$; this is done in the algorithm \filterstep, which ``filters out'' many $\MH \subset \MF$, and assigns to each $\MH$ which is filtered out some $\ML \subset \MF$ which is not filtered out that is a good $\ell_\infty$ approximation of $\MH$. 

To describe the algorithm \filterstep, fix a class $\MF \subset [\K]^\MX$.
For $\ell \geq 0$ and $0 \leq b \leq d$, set
\colt{
  $$
\irred{\ell}{b}{\MF} := \left\{ \MH \subset \MF: \substack{\text{$\MH$ is a finite restriction subclass of $\MF$} \\ \text{which is $\ell$-irreducible, and $\sfat_2(\MH) = b$}} \right\}.
$$}
\arxiv{
    $$
\irred{\ell}{b}{\MF} := \left\{ \MH \subset \MF: \parbox{6.5cm}{\centering\text{$\MH$ is a finite restriction subclass of $\MF$} \\ \text{which is $\ell$-irreducible, and $\sfat_2(\MH) = b$}} \right\}.
$$}
\begin{algorithm}[!htp]
  \caption{\bf \filterstep}
  \label{alg:filterstep}
\KwIn{A class $\MF$ with $d := \sfat_2(\MF)$, and a sequence $(\ell_{r,t})_{r,t \geq 0}$ of positive integers that is non-decreasing in $r$, a parameter $r_{\max}$.}
\arxiv{\begin{enumerate}[leftmargin=14pt,rightmargin=20pt,itemsep=1pt,topsep=1.5pt]}
\colt{\begin{enumerate}[leftmargin=20pt,rightmargin=20pt,itemsep=1pt,topsep=1.5pt]}
\item For each $t \in \{0, 1, \ldots, d \}$, set $\CL_{t} \gets \emptyset$. %
\item For $0 \leq t \leq d$ and $0 \leq r \leq r_{\max}$, define $\CI_{r,t} := \irred{\ell_{r,t}}{d-t}{\MF}$. Also set $\CI_{r_{\max} + 1, t} := \emptyset$ for $0 \leq t \leq d$.
\item For $t \in \{0, 1, \ldots, d\}$:
  \begin{enumerate}
  \item For $r \in \{r_{\max}, r_{\max}-1, \ldots, 0\}$:
    \begin{enumerate}
      \item \label{it:for-mh} For each $\MH \in \CI_{r,t} \backslash \CI_{r+1,t}$: {\it (Since the sequence $\ell_{r,t}$ is non-decreasing in $r$, we have $\CI_{r+1,t} \subset \CI_{r,t}$ for all $r,t$. Note that this step makes sense since $\CI_{r,t}$ is countable; an arbitrary enumeration of $\CI_{r,t}$ may be used.)} %
  \begin{enumerate}
  \item \label{it:found-l} If there is some $\ML \in \CL_{d-t}$ and $\ba \subset \MX \times [\K]$ with $|\ba| \leq \ell_{r,t} - 1$ so that $\sfat_2(\MF|_{\ba}) = d-t$ and for all $(x,y) \in \ba$, $\soa{\ML}{x} = \soa{\MH}{x} = y$, then set $\repl{\MH} \gets \ML$. %
  \item \label{it:set-l} Else, add $\MH$ to $\CL_{d-t}$, and set $\repl{\MH} \gets \MH$.
  \end{enumerate}
\end{enumerate}
\end{enumerate}
\item Output the sets $\CL_t$, $0 \leq t \leq d$, as well as the mapping $\repl{\cdot}$. %
  \end{enumerate}
\end{algorithm}
The algorithm \filterstep is presented in Algorithm \ref{alg:filterstep}. For an input positive integer $r_{\max}$ and a sequence $(\ell_{r,t})_{r,t}$ defined for $0 \leq r \leq r_{\max}, 0 \leq t \leq d$, the algorithm defines a mapping $\repl{\cdot}$, which maps each $\MH \in \irred{\ell_{r,t}}{d-t}{\MF}$, for $0 \leq r \leq r_{\max}$ and $0 \leq t \leq d$, into some ``filtered set'' $\CL_{d-t}$. For $\MH \in \irred{\ell_{r,t}}{d-t}{\MF}$, the class $\repl{\MH}$ should be interpreted as a representative of $\MH$ which approximates it well, in the sense of the following lemma:
\begin{lemma}
  \label{lem:close-reps}
  Fix inputs $\MF, (\ell_{r,t})_{r,t \geq 0}, r_{\max}$ to \filterstep. For any $0 \leq r \leq r_{\max}, 0 \leq t \leq d$, and any $\MH \in \irred{\ell_{r,t}}{d-t}{\MF}$, we have that $\| \soaf{\MH} - \soaf{\repl{\MH}} \|_\infty \leq 1$.
\end{lemma}

The algorithm \filterstep is designed so that its output sets $\CL_{d-t}$, $0 \leq t \leq d$, satisfy the following sparsity-type property:
\begin{lemma}
  \label{lem:at-most-one}
Fix inputs $\MF, (\ell_{r,t})_{r,t \geq 0}, r_{\max}$ to \filterstep. For any $0 \leq t \leq d$ and $0 \leq r \leq r_{\max}$, and any $\ba \subset \MX \times [\K]$ with $|\ba| \leq \ell_{r,t} - 1$ so that $\sfat_2(\MF|_{\ba}) = d-t$, there is at most one element $\ML \in \CL_{d-t} \cap \irred{\ell_{r,t}}{d-t}{\MF}$ so that for all $(x,y) \in \ba$, $\soa{\ML}{x} = y$.
\end{lemma}

\subsection{Reducing trees and \soafilter}
\label{sec:redtree-soafilter}
In this section we describe the algorithm \soafilter in full; before doing so, we introduce the notion of \emph{reducing tree} in the following two definitions:
\begin{definition}[Augmented tree]
  \label{def:augmented}
  For $d \geq 1, \K \in \BN$, an \dfn{augmented $\K$-ary $\MX$-valued tree} of depth $d$ is defined exactly the same as a $K$-ary $\MX$-valued tree (Definition \ref{def:xv-tree}), with the exception that there is a unique value of $k_1 \in [\K]$ so that the sequence $(k_1)$ lies in the domain of $\bx_2$ (in particular, requirement \ref{it:sibling} in Definition \ref{def:xv-tree} is dropped for $t=1$). Moreover, the only node associated with a sequence of length 1 is the node associated with $(k_1)$. We will say that the augmented tree $\bx$ is \dfn{rooted by the pair $(\bx_1, k_1)$.}
\end{definition}
One should think of an augmented $\MX$-labeled tree $\bx$ of depth $d$ which is rooted by the pair $(x, k)$ as an $\MX$-labeled tree $\bx'$ of depth $d-1$ for which we created a new root labeled by $x$ and attached to it a single child (labeled by $k$), which is the root of the tree $\bx'$. (Note that we have $\bx_1 = x$ here.)

\begin{definition}[Reducing tree]
    \label{def:reducing-tree}
  Suppose $\MH \subset [\K]^\MX$, and let $d := \sfat_2(\MH)$. Fix an increasing sequence $(\ell_t)_{t \geq 0}$ of positive integers. Given a point $(x,y) \in \MX \times [\K]$ so that $\sfat_2(\MH|_{(x,y)}) < \sfat_2(\MH)$, we say that an augmented $\K$-ary $\MX$-labeled tree $\bx$ rooted by the pair $(x,y)$ is a \dfn{reducing tree for the pair $(x,y)$ and the sequence $(\ell_t)_{t \geq 0}$} if any leaf $v$ of the tree satisfies:
  \begin{itemize}
  \item $\MH|_{\ba(v)}$ is either empty or is $\ell_t$-irreducible, where $t := d - \sfat_2(\MH|_{\ba(v)})$. %
  \item $\height(v) \leq \sum_{t'=0}^{t-1} \ell_{t'}$. Moreover, for any $1 \leq \tilde t < t$, there is some node $v'$ which is an ancestor of $v$ so that  $\sfat_2(\MH|_{\ba(v')}) \leq d-\tilde t$ and $\height(v') \leq \sum_{t'=0}^{\tilde t-1} \ell_{t'}$. 
  \end{itemize}
\end{definition}
Lemma \ref{lem:make-tree} in the appendix shows that reducing trees exist.

The algorithm \soafilter is presented in Algorithm \ref{alg:soafilter}. It takes as input some hypothesis $\hat g : \MX \ra [\K]$ and a class $\MF \subset [\K]^\MX$, as well as parameters $\tau_{\max}, r_{\max} \in \BN$. Its output is a set $\repf{\hat g}$, consisting of sub-classes of $\MF$. The set $\repf{\hat g}$ should be interpreted as a set of ``representatives'' of $\hat g$ in the sense that for $\ML \in \repf{\hat g}$, under appropriate conditions, we will have that $\soaf{\ML}$ is a good $\ell_\infty$-approximation of $\hat g$ (i.e., $\| \soaf{\ML} - \hat g \|_\infty$ is small); see Lemma \ref{lem:soafilter-lstar} below. 

The algorithm $\soafilter$ proceeds as follows. It first runs the algorithm $\filterstep$ for the class $\MF$, which produces ``filtered sets'' $\CL_{d-t}, 0 \leq t \leq d$, of sub-classes of $\MF$; each element of $\CR_{\hat g}$ will belong to some set $\CL_{d-t}$. \soafilter then tries to find finite sets $\ba \subset \MX \times [\K]$ so that both (a) $\| \soaf{\MF|_{\ba}} - \hat g \|_\infty$ is small and (b) so that for some $\ML$ in one of the ``filtered sets'' $\CL_{d-t}$ produced by \filterstep, it holds that $\soa{\ML}{x} = y$ for each $(x,y) \in \ba$; such sets $\ML$ will be added to $\repf{\hat g}$ (step \ref{it:found-close}). The sets $\ba$ are built up gradually as follows: if some set $\ba$ in the process of being built up is so that $\| \soaf{\MF|_{\ba}} - \hat g \|_\infty$ is large, then we may choose some $x_\ba \in \MX$ so that $|\soa{\MH}{x_\ba} - \hat g(x_\ba)|$ is large (step \ref{it:choose-xs}). For $y$ not too far from $\hat g(x_\ba)$, it will follow that we can construct a reducing tree with respect to the class $\MF|_{\ba}$ at the point $(x_\ba, y)$ (step \ref{it:make-tree}). For some of the leaves $v$ of this reducing tree, we will then add $\ba(v)$ to $\ba$ to create a new set $\ba'$ (one for each such leaf $v$), and continue to process each of these new sets $\ba'$ (step \ref{it:bav-added}). Intuitively, adding $\ba(v)$ to $\ba$ ``restricts'' the class of functions $\MF|_{\ba}$ under consideration so that all functions in it (and therefore its SOA hypothesis $\soaf{\MF|_{\ba}}$) well-approximates $\hat g(x_\ba)$ at $x_\ba$. Since for all leaves $v$ of the reducing tree we must have that $\sfat_2(\MF|_{\ba(v) \cup \ba}) < \sfat_2(\MF|_{\ba})$, this process must eventually terminate.  We will show that some sequence of restrictions, corresponding to a choice of leaf of the reducing tree created at each step, will create some set $\ba$ with our desired properties (a) and (b) above. All rigorous details of the algorithm are presented in Algorithm \ref{alg:soafilter}.  %
\begin{algorithm}[ht]
  \caption{\bf \soafilter} \label{alg:soafilter}
  \KwIn{Class $\MF \subset [\K]^\MX$, $d := \sfat_2(\MF)$, sequence $(\ell_{r,t})_{r,t \geq 0}$, $r_{\max} \in \BN$, tolerance parameter $\tau_{\max} \in \BN$, $\chi \in \BN$, $\hg \in [\K]^\MX$. %
    It is assumed that $r_{\max}, \tau_{\max}$ are multiples of $d+1$; let $r_0 := r_{\max} / (d+1), \tau_0 := \tau_{\max} / (d+1)$. \colt{Initialize $\repf{\hg} \gets \emptyset$.}}
\arxiv{\begin{enumerate}[leftmargin=14pt,rightmargin=20pt,itemsep=1pt,topsep=1.5pt]}
\colt{\begin{enumerate}[leftmargin=20pt,rightmargin=20pt,itemsep=1pt,topsep=1.5pt]}
\item \label{it:call-filterstep} Run the algorithm \filterstep (Algorithm \ref{alg:filterstep}) with $\MF$, $(\ell_{r,t})_{r,t \geq 0}$, and $r_{\max}$ as input, and let the output sets be denoted $(\CL_t)_{0 \leq t \leq d}$.  
\arxiv{\item Set $\repf{\hg} \gets \emptyset$.}
\item For each $0 \leq s \leq d, 0 \leq j \leq d$, set $\CQ_{j,s} \gets \emptyset$. {\it ($\CQ_{j,s}$ will be a collection of finite subsets $\ba \subset \MX \times [\K]$ defined for each index pair $s,j$.)} %
\item Set $\CQ_{j,0} \gets \{ \emptyset\}$ for each $j$ (i.e., $\CQ_{j,0}$ has a single element, which is the empty set).\label{it:cq0}
    \item \label{it:for-j} For $j \in \{0, 1, \ldots, d\}$: \colt{let $r \gets r_{\max} - jr_0 - 1, \ \tau \gets j\tau_0 + 2 + \chi$:}
      \begin{enumerate}
        \arxiv{  \item Let $r \gets r_{\max} - jr_0 - 1, \ \tau \gets j\tau_0 + 2 + \chi$. %
          }
          \item For $s \in \{0, 1, \ldots, d \}$:\label{it:iterate-s}
      \begin{itemize}
  \item For each $\ba \in \CQ_{j,s}$, letting $\MH := \MF|_\ba$ :%
    \begin{enumerate}
    \item \label{it:check-empty} If $\MH$ is empty, continue on with the next $\ba \in \CQ_{j,s}$. 
    \item \label{it:found-close} If $\| \soaf{\MH} - \hg \|_\infty \leq \tau$: %
      \begin{itemize}
      \item If there is some $\ML \in \irred{\ell_{r,t}}{d-t}{\MF} \cap \CL_{d-t}$ so that for all $(x,y) \in \ba$, $\soa{\ML}{x} = y$, then add any such $\ML$ to $\repf{\hg}$.
      \item %
        Continue (i.e., go to step \ref{it:check-empty} with the next $\ba \in \CQ_{j,s}$). %
      \end{itemize}
    \item \label{it:choose-xs} Else, we have $\| \soaf{\MH} - \hg \|_\infty > \tau$; then choose some $x_{\ba} \in \MX$ so that $| \soa{\MH}{x_{\ba}} - \hg({x_{\ba}}) | \geq \tau + 1$.
    \item \label{it:choose-y} Let $k \gets \hg({x_{\ba}})$. For $y \in \{k-\tau+1 \vee 0, k-\tau+2 \vee 0, \ldots, k+\tau-1 \wedge \K\}$: %
      \begin{enumerate}
      \item \label{it:make-tree} Let $t_\ba := d - \sfat_2(\MH)$, and let $\bx\^{\MH, (x_{\ba},y)}$ be a reducing tree with respect to $\MH$ for the point $(x_{\ba},y)$ and the sequence $(\ell_{r,t + t_{\ba}})_{0 \leq t \leq d - t_{\ba}}$, as constructed per Lemma \ref{lem:make-tree}. %
        {\it (Note that the reducing tree is well-defined since $|k - \soa{\MH}{x_\ba}| \geq \tau+1$ and so any $y$ with $|y-k| \leq \tau-1$ must satisfy $\sfat_2(\MH|_{(x_\ba,y)}) < \sfat_2(\MH)$.)}
      \item \label{it:bav-added} For each leaf $v$ of the tree $\bx\^{\MH, (x_{\ba},y)}$, if it is the case that (a) $\MF|_{\ba \cup \ba(v)}$ is nonempty, and (b) for each $(x,y) \in \ba(v)$, $|\hg({x}) - y | \leq \tau-1$, then add $\ba \cup \ba(v)$ to $\CQ_{j,s+1}$. %
      \end{enumerate}
    \end{enumerate}
  \end{itemize}
    \end{enumerate}
    \item \label{it:remove-faraway-filter} Remove all $\ML \in \CR_{\hat g}$ from $\CR_{\hat g}$ with $\| \soaf{\ML} - \hat g \|_\infty > \taumax$\arxiv{.}\colt{, then output $\repf{\hg}$.}
   \arxiv{ \item Output $\repf{\hg}$.}
    \end{enumerate}
\end{algorithm}
Lemma \ref{lem:soafilter-lstar} provides the main guarantee for \soafilter.
\begin{lemma}[``Strong stability'']
  \label{lem:soafilter-lstar}
  Fix any positive integer $\bar\ell$. Suppose that $\MG \subset \MF$ is nonempty, $\hat g \in [\K]^\MX$, that %
  $ \| \soaf{\MG} - \hg \|_\infty \leq \chi$ for some $\chi > 0$, and that $\MG$ is $(\bar\ell \cdot (d+3)^d)$-irreducible. %
  Then there is some $\bar\ell$-irreducible $\ML^\st \subset \MF$, depending only on $\MG$, so that $\| \soaf{\ML^\st} - \soaf{\MG} \|_\infty \leq \taumax + 1$ and so that $\ML^\st \in \repf{\hg}$, where $\repf{\hg}$ is the output of \soafilter when given as inputs $\MF$, $\hat g$, $r_{\max} = \rmax, \ \tau_{\max} = \taumax$ and the sequence $\ell_{r,t} := \bar\ell \cdot (r+2)^t$ for $0 \leq r \leq \rmax$, $0 \leq t \leq d$.

  Moreover, all $\ML \in \CR_{\hat g}$ satisfy $\| \soaf{\ML} - \hat g \|_\infty \leq \taumax$ and are $\bar \ell$-irreducible.
\end{lemma}
We provide a brief sketch of the proof of Lemma \ref{lem:soafilter-lstar}; the full proof is given in the appendix. The final statement of the lemma follows from step \ref{it:remove-faraway-filter} of \soafilter. To prove the remainder of the lemma, for $0 \leq \tau \leq \taumax$ and $2 \leq r \leq \rmax$, define
$ 
\mu(r,\tau) := \max_{(\MH, \ell) \in \CG_{r,\tau}} \left\{ \sfat_2(\MH) \right\},
$ 
where
\begin{equation}
  \CG_{r,\tau} := \left\{ (\MH, \ell_{r,t}) : \substack{\text{$\MH \subset \MF$ is $\ell_{r,t}$-irreducible and a finite restriction subclass of $\MF$, } \\ \text{ where $t = d - \sfat_2(\MH)$, %
      and $\| \soaf{\MH} - \soaf{\MG} \|_\infty \leq \tau$.}} \right\}.\nonumber
  \end{equation}
  Since $\MG$ is $\ell_{\rmax,d}$-irreducible, and for all $t,r$ we have $\ell_{r,t} \leq \ell_{\rmax,d}$, we have that $(\MG, \ell_{r,t}) \in \CG_{r,\tau}$ for $t = d - \sfat_2(\MG)$ and all $0 \leq r \leq \rmax, 0 \leq \tau \leq \taumax$, i.e., $\CG_{r,\tau}$ is nonempty and so $\mu(r,\tau)$ is well-defined. %
  It is straightforward to show, using that $\mu$ is non-decreasing in $\tau$ and non-increasing in $r$, that we can find some $r^\st, \tau^\st$ so that $\mu(r^\st, \tau^\st) = \mu(r^\st - 1, \tau^\st + 2 + 2\chi)$. Informally, this property of $r^\st, \tau^\st$ provides a source of ``stability'' which may be exploited to find some $\ML^\st$ and show that it satisfies the claimed properties in Lemma \ref{lem:soafilter-lstar}. 

  We next explain how $\ML^\st$ is defined: choose some $(\MH^\st, \ell^\st)$ which achieves the maximum in (\ref{eq:mrtau}) for $r=r^\st, \ \tau = \tau^\st$; letting $t^\st = d - \sfat_2(\MH^\st)$, we have $\ell^\st = \ell_{r^\st, t^\st}$. Let $\repl{\cdot}$ be the mapping defined as the output of \filterstep with the input class $\MF$, the sequence $(\ell_{r,t})_{0\leq r \leq r_{\max}, 0 \leq t \leq d}$, and $r_{\max} = d+1$ (these are the parameters used in Step \ref{it:call-filterstep} of \soafilter). Now set $\ML^\st = \repl{\MH^\st} \in \CL_{d-t^\st}\cap \irred{\ell_{r^\st, t^\st}}{d-t^\st}{\MF}$; this is well-defined since $\MH^\st \in \irred{\ell_{r^\st, t^\st}}{d - t^\st}{\MF}$. It can be shown that $\ML^\st$ satisfies the claimed properties of Lemma \ref{lem:soafilter-lstar}; full details are given in the appendix.

Finally, in Lemma \ref{lem:rep-size} we show that $|\CR_{\hat g}| \leq \K^{\bar \ell \cdot (d+4)^d}$ for the parameter settings in Lemma \ref{lem:soafilter-lstar}.

\section{Putting it all together with \RegLearn: on the proof of Theorem \ref{thm:reglearn-informal}}
\label{sec:combine-informal}
Theorem \ref{thm:reglearn-informal} may be obtained as a reasonably straightforward consequence of the results presented in the previous sections; the full algorithm (\RegLearn; Algorithm \ref{alg:reglearn}) is presented in the appendix. For positive integers $n_0, m$, we will draw $n := n_0 m$ samples $(x,y)$ from some distribution $P$ on $\MX \times [\K]$, and partition them into $m$ groups of $n_0$ samples. For $1 \leq j \leq m$, the $j$th group of $n_0$ samples will be fed to the algorithm \ReduceTreeReg, which outputs some $\{\hat g_1\^j, \ldots, \hat g_{M_j}\^j \}$ of candidate hypotheses, satisfying the weak stability guarrantee of Lemma \ref{lem:weak-stability-informal}. Then each of $\hat g_1\^j, \ldots, \hat g_{M_j}\^j$ will be fed to \soafilter, which produces an output set $\CR_{\hat g_i\^j}$ for each $1 \leq i \leq M_j$, consisting of hypotheses all of which have low population error. The combination of Lemma \ref{lem:weak-stability-informal} and the strong stability property of Lemma \ref{lem:soafilter-lstar} gives that there is some hypothesis $h^\st : \MX \ra [\K]$, depending only on $\MF, P$, so that with probability $1/O(d)$ over the $n_0$ samples, $h^\st \in \hat \CR\^j := \bigcup_{i=1}^{M_j} \CR_{\hat g_i\^j}$. We will also be able to bound $|\CR\^j|$ by $\K^{2^{\tilde O(d)}}$. Then we will apply Proposition \ref{prop:sparse-selection} with $m$ users whose sets are $\hat \CR\^1, \ldots, \hat \CR\^m$. 
By choosing the number of groups $m$ to be large enough, we may ensure that some $h^\st$ occurs in a number of groups greater than the additive error in Proposition \ref{prop:sparse-selection}, which ensures that the private sparse selection algorithm outputs some such $h^\st$ with high probability. Full details of the proof are presented in Appendix \ref{sec:reglearn}.
\section{Conclusion and future work}
\label{sec:conclusion}
In this paper we showed that the condition $\liminf_{\eta \downarrow 0} \eta \cdot \sfat_\eta(\MH) = 0$ is sufficient for the class $\MH \subset [-1,1]^\MX$ to be privately learnable. A natural question is whether this growth condition can be relaxed; it seems that new techniques will be required even to prove that all classes $\MH$ with $\eta \cdot \sfat_\eta(\MH) \leq 1$ for all $\eta > 0$ are privately learnable, if this is even true (such classes are all online learnable since $\sfat_\eta(\MH)$ is {necessarily }finite). An example of a natural hypothesis class for which our growth condition is not satisfied is infinite-dimensional $\ell_2$ regression: in particular, set $\MX = \ell_2^\infty = \{ (x_1, x_2, \ldots ) : \ x_i \in \BR, \sum_{i=1}^\infty x_i^2 \leq 1\}$ and $\MH = \ell_2^\infty = \{ (w_1, w_2, \ldots ) : \ w_i \in \BR, \sum_{i=1}^\infty w_i^2 \leq 1 \}$, and then for $h = (w_1, w_2, \ldots)$ and $x = (x_1, x_2, \ldots)$, define $h(x) := \lng w, x \rng = \sum_{i=1}^\infty w_i x_i$. It can be shown that $\sfat_\eta(\MH) \asymp 1/\eta^2 \gg 1/\eta$ as $\eta \ra 0$. 

Another interesting question is whether the sample complexity bound of Theorem \ref{thm:reglearn-informal} can be improved to one that is polynomial in $\sfat_\eta(\MH)$; for the setting of binary classification, it \emph{is} possible to obtain sample complexity bounds polynomial in the appropriate complexity parameter for online learnability, namely the Littlestone dimension \citep{ghazi_sample_2020}.

\colt{\acks{ I am grateful to Sasha Rakhlin and Roi Livni for helpful suggestions.}}
\arxiv{\section*{Acknowledgements}
  I am grateful to Sasha Rakhlin and Roi Livni for helpful suggestions.}

\arxiv{\bibliographystyle{alpha}}
\bibliography{privacy.bib}

\appendix

\section{Additional preliminaries}
In this section we introduce some additional preliminaries which will be useful in our proofs.
\subsection{Fat-shattering dimension and uniform convergence}
In this section we overview some uniform convergence properties of real-valued classes and their discretizations. For a class $\MH \subset \BR^\MX$ and $\eta > 0$, the \dfn{$\eta$-fat shattering dimension} of $\MH$ is defined as the largest positive integer $d$ so that there are $d$ points $x_1, \ldots, x_d \in \MX$ and real numbers $s_1, \ldots, s_d \in \BR$ so that for each $b = (b_1, \ldots, b_d) \in \{0,1\}^d$, there is a function $h \in \MH$ so that, for $1 \leq i \leq d$, $h(x_i) \geq s_i + \eta$ if $b_i = 1$, and $h(x_i) \leq s_i - \eta$ if $b_i = 0$.

We will use the following result showing that finiteness of the fat-shattering dimension of $\MH \subset [-1,1]^\MX$ implies that it exhibits uniform convergence.
\begin{theorem}[Uniform convergence; \cite{mendelson_entropy_2003}]
  \label{thm:fat-uc}
  There are constants $C_0 \geq 1$ and $0 < c_0 \leq 1$ so that the following holds. For any $\MH \subset [-1,1]^\MX$, any distribution $Q$ on $\MX \times [-1,1]$, and any $\gamma \in (0,1)$, it holds that
  \begin{align}
    \label{eq:dudley-uc}
\Pr_{S_n \sim Q^n} \left[ \sup_{h \in \MH} \left| \err{Q}{h} - \err{\hat Q_{S_n}}{h} \right| > C_0 \cdot \left( \inf_{\eta \geq 0} \left\{ \eta + \frac{1}{\sqrt{n}} \int_\eta^1 \sqrt{\fat_{c_0 \eta'}(\MH) \log(1/\eta')}\ d\eta' \right\}+ \sqrt{\frac{\log(1/\gamma)}{n}} \right) \right] \leq \gamma.
  \end{align}
\end{theorem}
The specific form of Theorem \ref{thm:fat-uc} may be derived from \cite[Corollary 12.8]{rakhlin_statistical_2014} (which is a corollary of \cite[Theorem 1]{mendelson_entropy_2003}) by applying the symmetrization lemma together with McDiarmid's inequality (see the proof of Theorem 8 in \cite{bartlett_rademacher_2003}). By upper bounding the integral in (\ref{eq:dudley-uc}) by $\sqrt{\fat_{c_0 \eta}(\MH) \log(1/\eta)}$ for some choice of $\eta \in (0,1)$, we obtain the following consequence, which only depends on the fat-shattering dimension of $\MH$ at a single scale $c_0\eta$, yet may be weaker than Theorem \ref{thm:fat-uc}.
\begin{corollary}[Uniform convergence, simplified]
  \label{cor:fat-uc}
  There are constants $C_0 \geq 1$ and $0 < c_0 \leq 1$ so that the following holds. For any $\MH \subset [-1,1]^\MX$, any distribution $Q$ on $\MX \times [-1,1]$, and any $\gamma \in (0,1/2), \eta \in (0,1/2)$, it holds that, for any $$n \geq C_0 \cdot \frac{ \fat_{c_0 \eta}(\MH) \log(1/\eta) + \log(1/\gamma)}{\eta^2},$$
  we have
  \begin{align}
    \label{eq:fat-uc}
\Pr_{S_n \sim Q^n} \left[ \sup_{h \in \MH} \left| \err{Q}{h} - \err{\hat Q_{S_n}}{h} \right| > \eta \right] \leq \gamma.
  \end{align}
\end{corollary}

\subsection{Uniform convergence for discretized classes}
\label{sec:uc-prelim-disc}
Recall that we defined discretized classes and distributions in Section \ref{sec:pac}. In this section we state (straightforward) consequences of Corollary \ref{cor:fat-uc} for such discretized classes.

For $y,y' \in [-1,1]$, note that
$$
\frac{\lceil 2/\eta \rceil \cdot |y-y'|}{2} - 1 \leq \left| \disc{y}{\eta} - \disc{y'}{\eta} \right| \leq \frac{\lceil 2/\eta \rceil \cdot |y-y'|}{2} + 1,
$$
Therefore, for $\MH \subset [-1,1]^\MX$, a distribution $Q$ on $\MX \times [-1,1]$, and any $h \in \MH$, we have that
\begin{equation}
  \label{eq:disc-cont-rel}
\frac{\lceil 2/\eta \rceil \cdot \err{Q}{h}}{2} - 1 \leq  \err{\disc{Q}{\eta}}{\disc{h}{\eta}}  \leq \frac{\lceil 2/\eta \rceil \cdot \err{Q}{h}}{2} + 1 .
\end{equation}
Using (\ref{eq:disc-cont-rel}), we have the following corollary of Corollary \ref{cor:fat-uc} showing a uniform convergence result for the discretized class corresponding to a class of finite fat-shattering dimension.
\begin{corollary}
  \label{cor:fat-uc-disc}
  There are constants $C_0 \geq 1$ and $0 < c_0 \leq 1$ so that the following holds. For any $\MH \subset [-1,1]^\MX$, any distribution $Q$ on $\MX \times [-1,1]$, and any $\gamma \in (0,1/2), \alpha \in (0,1/2)$, it holds that, for any
  \begin{equation}
    \label{eq:fat-uc-disc-n-lb}
    n \geq C_0 \cdot \frac{ \fat_{c_0 \alpha}(\MH) \log(1/\alpha) + \log(1/\gamma)}{\alpha^2},
    \end{equation}
  we have
  \begin{align}
    \label{eq:fat-uc-disc}
\Pr_{S_n \sim Q^n} \left[ \sup_{h \in \MH} \left| \err{\disc{Q}{\alpha}}{\disc{h}{\alpha}} - \err{\disc{\hat Q_{S_n}}{\alpha}}{\disc{h}{\alpha}} \right| > 3 \right] \leq \gamma.
  \end{align}
\end{corollary}
\begin{proof}[Proof of Corollary \ref{cor:fat-uc-disc}]
  We first upper bound the probability that $\sup_{h \in \MH} \left\{\err{\disc{Q}{\alpha}}{\disc{h}{\alpha}} - \err{\disc{\hat Q_{S_n}}{\alpha}}{\disc{h}{\alpha}}\right\} > 3$. With probability at least $1-\gamma/2$ over $S_n \sim Q^n$, as long as $C_0$ and $c_0$ in (\ref{eq:fat-uc-disc-n-lb}) are sufficiently large and small, respectively, %
  \begin{align}
    & \sup_{h \in \MH}\left\{ \err{\disc{Q}{\alpha}}{\disc{h}{\alpha}} - \err{\disc{\hat Q_{S_n}}{\alpha}}{\disc{h}{\alpha}} \right\} \nonumber\\
    & \leq \sup_{h \in \MH} \left( \frac{\lceil 2/\alpha \rceil \cdot \err{Q}{h}}{2} + 1 \right) - \left(\frac{\lceil 2/\alpha \rceil \cdot \err{\hat Q_{S_n}}{h}}{2} - 1 \right) \label{eq:use-cont-disc}\\
    & = \frac{\lceil 2/\alpha\rceil}{2} \cdot \sup_{h \in \MH}( \err{Q}{h} - \err{\hat Q_{S_n}}{h} ) + 2 \nonumber\\
    & \leq \frac{\lceil 2/\alpha \rceil}{2} \cdot 2/\lceil 2/\alpha \rceil + 2 \label{eq:use-unif-conv}\\
    & = 3,
  \end{align}
  where (\ref{eq:use-cont-disc}) follows from (\ref{eq:disc-cont-rel}), and (\ref{eq:use-unif-conv}) follows from Corollary \ref{cor:fat-uc} with $\eta = 2/\lceil 2 / \alpha \rceil = \Theta(\alpha)$ (and holds with probability at least $1-\gamma/2$ over $S_n \sim Q^n$). The fact that $\sup_{h \in \MH} \left\{ - \left(\err{\disc{Q}{\alpha}}{\disc{h}{\alpha}} - \err{\disc{\hat Q_{S_n}}{\alpha}}{\disc{h}{\alpha}}\right) \right\} > 3$ with probability at least $1-\gamma/2$ is established similarly.
\end{proof}

The following result, also a consequence of Corollary \ref{cor:fat-uc}, is similar to Corollary \ref{cor:fat-uc-disc}, but it states the sample complexity bound in terms of the quantity $\sfat_2(\MF)$ of a discretized class $\MF$, at the expense of having a larger constant in (\ref{eq:unif-conv-2fat}) (not explicitly computed here; compare to (\ref{eq:fat-uc-disc})). Strictly speaking, Corollary \ref{cor:fat-uc-disc} is not necessary for our purposes, but we use it to improve certain constants in our bounds. %
\begin{corollary}
  \label{cor:2fat-uc-disc}
  There are constant $C_0, C_1 \geq 1$ so that the following holds. For any $\K \in \BN$, $\MF \subset [\K]^\MX$, any distribution $P$ on $\MX \times [\K]$, and any $\gamma \in (0, 1/2)$, it holds that,  for any
  $$
n \geq C_0 \K^2 \cdot \left( \fat_2(\MF) \log(\K) + \log(1/\gamma) \right),
$$
we have
\begin{equation}
  \label{eq:unif-conv-2fat}
\Pr_{S_n \sim P^n} \left[ \sup_{f \in \MF} \left| \err{P}{f} - \err{\emp}{f} \right| > C_1 \right] \leq \gamma.
\end{equation}
\end{corollary}
\begin{proof}
  Define the class $\tilde \MF \subset [-1,1]^\MX$ as follows: for each $f \in \MF$, there is a function $\tilde f \in \tilde \MF$, defined as $\tilde f(x) := \frac{2}{\K} f(x) - 1$. Note that $\fat_2(\MF) = \fat_{1/\K}(\tilde \MF)$. Let $c_0, C_0$ be the constants of Corollary \ref{cor:fat-uc}. Using Corollary \ref{cor:fat-uc} with $\eta = \frac{1}{c_0 \K}$, we have that for $n \geq \frac{\K^2C_0}{c_0^2} \cdot \left( \fat_2(\MF) \log(C_0 \K) + \log(1/\gamma) \right)$, it holds that for any distribution $Q$ on $\MX \times [-1,1]$, $$\Pr_{S_n \sim Q^n} \left[ \sup_{\tilde f \in \tilde \MF} \left| \err{Q}{\tilde f} - \err{\hat Q_{S_n}}{\tilde f} \right| > \frac{1}{c_0 \K} \right] \leq \gamma.$$
  The claimed statement (\ref{eq:unif-conv-2fat}) follows by setting $C_1 = 1/(2c_0)$ and increasing $C_0$ by a sufficiently large amount. 
\end{proof}

We may upper bound the fat-shattering dimension and the sequential fat-shattering dimension of $\disc{\MH}{\eta}$ in terms of the corresponding quantities for $\MH$:
\begin{lemma}
  \label{lem:fat-disc}
Suppose $\MH \subset [-1,1]^\MX$, and $\eta > 0$. Then it holds that $\fat_2(\disc{\MH}{\eta}) \leq \fat_\eta(\MH)$, and $\sfat_2(\disc{\MH}{\eta}) \leq \sfat_\eta(\MH)$.
\end{lemma}
\begin{proof}
  This follows from the fact that for any $\disc{h}{\eta} \in \disc{\MH}{\eta}$ and any $s \in \BR$, if it holds that $|\disc{h}{\eta}(x) - s| \geq 1$, then since $$\left| \left(\frac{2(\disc{h}{\eta}(x) - 1)}{\lceil 2/\eta \rceil} - 1\right) -  h(x) \right| \leq \eta/2$$
  for all $x \in \MX$, and
  $$
\left| \left(\frac{2(\disc{h}{\eta}(x) - 1)}{\lceil 2/\eta \rceil} - 1\right)  - \left(\frac{2(s - 1)}{\lceil 2/\eta \rceil} - 1\right)\right| \geq \frac{2}{\lceil 2/\eta \rceil} \geq \eta,
$$
we must have that
$$
\left| \left(\frac{2(s - 1)}{\lceil 2/\eta \rceil} - 1\right)  - h(x) \right| \geq \eta/2.
$$
Thus, if we have a 2-shattered tree (or dataset) for the class $\disc{\MH}{\eta}$, we may obtain a corresponding $\eta$-shattered one for the class $\MH$ by replacing each witness $s$ with $\frac{2(s-1)}{\lceil 2/\eta \rceil} - 1$. 
\end{proof}

\subsection{Attaching a tree via a node}
The following definition will be useful when arguing about trees in the context of irrecucibility:
\begin{definition}[Attaching a tree via a node]
  \label{def:attach}
  Suppose that $\bx, \bx'$ are $\K$-ary $\MX$-valued trees of depths $d$ and $d' \geq 1$, respectively, and that $v$ is a leaf of $\bx$, corresponding to some tuple $(\bar k_1, \ldots, \bar k_{t_0}) \in [\K]^t$ (in particular, the depth of $v$ is $t_0$). We say that the tree $\bx''$ is obtained by \dfn{attaching the tree $\bx'$ to $\bx$ via the leaf $v$}, where $\bx''$ is the depth-$(d' + t_0)$ tree defined as follows: for all $1 \leq t \leq d' + t_0$, and $k_1, \ldots, k_{t-1} \in [\K]$,
  \begin{align*}
    \bx''_t(k_1, \ldots, k_{t-1}) = \begin{cases}
      \bx_t(k_1, \ldots, k_{t-1}) \quad & : \ t \leq t_0 \text{ or } (k_1, \ldots, k_{t-1}) \neq (\bar k_1, \ldots, \bar k_{t-1})  \\
      \bx'_{t-t_0}(k_{t_0+1}, k_{t_0+2}, \ldots, k_{t-1}) \quad & : \ t > t_0 \text{ and } (k_1, \ldots, k_{t-1}) = (\bar k_1, \ldots, \bar k_{t-1}).
    \end{cases}
  \end{align*}
  (If, in either case above, either $\bx_t(k_1, \ldots, k_{t-1})$ or $\bx'_{t-t_0}(k_{t_0+1}, \ldots, k_{t-1})$ is not defined, then $\bx_t''(k_1, \ldots, k_{t-1})$ is not defined, i.e., $(k_1, \ldots, k_{t-1})$ is not in the domain of $\bx_t''$.)
  
  In words, $\bx''$ is obtained as follows: the node $v$ is given the label of $\bx_1'$, and the sub-tree of $\bx''$ rooted at $v$ is identical to $\bx'$ (and otherwise is identical to $\bx$).
\end{definition}

\subsection{Laplace distribution}
For a positive real number $b > 0$, write $\Lap(b)$ to denote the random variable $X \in \BR$ with probability density function $\Pr[X =  x] = \frac{1}{2b} \exp(-|x|/b)$. A straightforward computation gives that for any $t > 0$, $\Pr[|X| \geq t \cdot b] = \exp(-t)$.

\section{Proofs for Section \ref{sec:irreducibility}: irreducibility}
This section presents basic properties of the notion of irreducibility from Definition \ref{def:irreducibility}. Some of the results are analogous to those in the setting for classification \citep{ghazi_sample_2020}; this is indicated where it is the case.
\subsection{Basic properties of irreducibility}
  \begin{replemma}{lem:consec-k}
  Suppose $\MG \subset [\K]^\MX$ is irreducible. Then there are at most 2 values of $k \in [\K]$ so that $\sfat_2(\MG|_{(x,k)}) = \sfat_2(\MG)$, and if there are 2 values, they differ by 1.
\end{replemma}
\begin{proof}
  Let $d:= \sfat_2(\MG)$, and suppose without loss of generality that $k > k'$. 
Suppose for the purpose of contradiction that for some $k, k' \in [\K]$ with $|k-k'| \geq 2$, we have $\sfat_2(\MG|_{(x,k)}) = \sfat_2(\MG|_{(x,k')}) = \sfat_2(\MG) = d$. Let $\bx, \by$ be complete binary trees of depth $d$ shattered by $\MG|_{(x,k)}, \MG|_{(x,k')}$, respectively, witnessed by trees $\bs, \bt$, respectively. We construct a tree $\bz$ of depth $d+1$ shattered by $\MG$, as follows: for any $k_2, \ldots, k_{d+1} \in \{1,2\}$, set $\bz_{t+1}(1, k_2, \ldots, k_{t}) = \bx_t(k_2, \ldots, k_t)$, $\bz_{t+1}(2, k_2, \ldots, k_t) = \by_t(k_2, \ldots, k_t)$ for $1 \leq t \leq d$ and $\bz_1 = x$. (In words, we are setting $\bx, \by$ to be the left and right subtrees of a node labeled by $x$.) We claim that $\bz$ is 2-shattered by $\MG$: indeed, a witness $\br$ may be defined as follows: define $\br_{t+1}(1, k_2, \ldots, k_t) = \bs_t(k_2, \ldots, k_t)$, $\br_{t+1}(2, k_2, \ldots, k_t) = \bt_t(k_2, \ldots, k_t)$ for $1 \leq t \leq d$, and $\br_1 = \frac{k + k'}{2}$. (In words, $\br$ is the tree rooted by a node labeled by $\frac{k+ k'}{2}$, whose left and right subtrees are given by $\bs, \bt$, respectively.) That $\br$ witnesses the shattering follows from the fact that $\bs, \bt$ are witnesses to the shattering of $\MG|_{(x,k)}, \MG|_{(x,k')}$ by $\bx, \by$, respectively, and the fact that for any $f \in \MG|_{(x,k)}, f' \in \MG|_{(x,k')}$, we have that $f(\bz_1) - \br_1 \geq 1$ and $-(f'(\bz_1) - \br_1) \geq 1$. 
\end{proof}

\begin{lemma}
  \label{lem:0dim-irred}
Suppose $\MG \subset [\K]^\MX$ has $\sfat_2(\MG) = 0$. Then $\MG$ is $\ell$-irreducible for all $\ell \in \BN$.
\end{lemma}
\begin{proof}
Let $\bx$ be a $K$-ary $\MX$-valued tree of depth $\ell$. Since $\bigcup_{\text{leaves } v \text{ of } \bx} \MG|_{\ba(v)} = \MG$, the tree $\bx$ must have some leaf $v$ so that $\MG|_{\ba(v)}$ is nonempty. For such $v$, we must have that $\sfat_2(\MG|_{\ba(v)}) \geq 0$, and since $\MG|_{\ba(v)} \subset \MF$, we have $\sfat_2(\MG|_{\ba(v)}) = 0$, as desired.
\end{proof}

\begin{lemma}
  \label{lem:lm1-irred}
Suppose  $\MG \subset [\K]^\MX$ is $\ell$-irreducible for $\ell \geq 1$. Then for any $x \in \MX$, there is some $k \in [\K]$ so that $\MG|_{(x,k)}$ is $(\ell-1)$-irreducible and $\sfat_2(\MG|_{(x,k)}) = \sfat_2(\MG)$.
\end{lemma}
\begin{proof}
The statement of the lemma follows immediately from Definition \ref{def:irreducibility} if $\ell = 1$, so we may assume from here on that $\ell \geq 2$. 
  
  Fix any $x \in \MX$. Our goal is to show that there is some $k \in [\K]$ so that the following holds: for any $\K$-ary $\MX$-valued tree $\bx$, of depth $\ell-1$, $\bx$ has some leaf $v$ so that $\sfat_2(\MG|_{\{(x,k)\} \cup \ba(v)}) = \sfat_2(\MG)$.  We now consider two cases:

  {\bf Case 1.} There is a unique $k' \in [\K]$ so that $\sfat_2(\MG|_{(x,k')}) = \sfat_2(\MG)$. In this case, we set $k = k'$. Now consider any $\K$-ary $\MX$-valued tree $\bx$ of depth $\ell-1$. Let $\tilde \bx$ be the tree of depth $\ell$ whose root is given by $x$ and so that each child of the root is a copy of the tree $\bx$; formally, for $k_1, \ldots, k_\ell \in [\K]$, $\tilde \bx_{t+1}(k_1, \ldots, k_t) = \bx_t(k_2, \ldots, k_t)$ for $1 \leq t \leq \ell-1$, and $\tilde \bx_1 = x$. The $\ell$-irreducibility of $\MG$ guarantees the existence of some tuple $k_1, \ldots, k_\ell$ so that $\sfat_2(\MG|_{(x, k_1), (\tilde \bx_2(k_1), k_2), \ldots, (\tilde \bx_\ell(k_{1:\ell-1}), k_\ell)}) = \sfat_2(\MG)$. Since for all $k' \neq k$, we have $\sfat_2(\MG|_{(x,k')}) < \sfat_2(\MG)$, it holds that $k_1 = k$. Letting $v$ be the leaf of $\bx$ associated to the tupe $(k_2, \ldots, k_\ell)$, we see that $\sfat_2(\MG|_{\{ (x,k) \} \cup \ba(v)}) = \sfat_2(\MG)$, as desired.

  {\bf Case 2.} For some $k_0 \in [\K]$, it holds that $\sfat_2(\MG|_{(x,k_0)}) = \sfat_2(\MG|_{(x,k_0+1)}) = \sfat_2(\MG)$, and for all $k' \neq k_0$, $\sfat_2(\MG|_{(x,k')}) < \sfat_2(\MG)$ (see Lemma \ref{lem:consec-k}). Suppose for the purpose of contradiction that there did not exist a choice of $k \in \{ k_0, k_0 + 1\}$ so that $\MG|_{(x,k)}$ is $(\ell-1)$-irreducible. Then for each $k \in \{ k_0, k_0 + 1\}$, there is some tree $\bx\^k$ of depth $\ell - 1$ so that for any choice of $k_2, \ldots,k_\ell \in [\K]$ we have $\sfat_2(\MG|_{(x,k), (\bx\^k_1, k_2), \ldots, (\bx\^k_{\ell-1}(k_{2:\ell-1}), k_{\ell-1})}) < \sfat_2(\MG|_{(x,k)}) = \sfat_2(\MG)$.

  Now let $\tilde \bx$ be the tree of depth $\ell$ whose root is given by $x$, so that the $k'$-th child of the root, for $k' \neq k_0+1$, is a copy of the tree $x\^{k_0}$, and so that the $(k_0 +1)$-th child of the root is a copy of the tree $x\^{k_0+1}$. (The $k'$-th children of the root for $k' \in \{k_0, k_0+1\}$ can in fact be arbitrary.) Formally, for $k_1, \ldots, k_\ell \in [\K]$, we have $\tilde \bx_1 = x$ and
  $$
  \tilde \bx_{t+1}(k_1, \ldots, k_t) = \begin{cases}
    \bx\^{k_0}(k_2, \ldots, k_t) \quad : \quad k_1 \neq k_0 + 1 \\
    \bx\^{k_0+1}(k_2, \ldots, k_t) \quad : \quad k_1 = k_0 + 1.
  \end{cases}
  $$
  Now consider any sequence $(k_1, \ldots, k_\ell)$, and let its associated leaf in $\tilde \bx$ be denoted $v$. If $k_1 \not \in \{k_0, k_0 + 1\}$, then $\sfat_2(\MG|_{\ba(v)}) \leq \sfat_2(\MG|_{(x,k_1)}) < \sfat_2(\MG)$. If $k_1 \in \{k_0, k_0 + 1 \}$, then $$\sfat_2(\MG|_{\ba(v)}) = \sfat_2(\MG|_{(x, k_1), (\bx_1\^{k_1}, k_2), \ldots, (\bx_{\ell-1}\^{k_1}(k_2 : k_{\ell-1}), k_\ell)}) < \sfat_2(\MG|_{(x,k_1)}) = \sfat_2(\MG).$$
  This contradicts the $\ell$-irreducibility of $\MG$, completing the proof.
\end{proof}

The following lemma is analogous to \cite[Lemma 4.2]{ghazi_sample_2020}:
\begin{lemma}
  \label{lem:irred-hg}
Suppose $\MH \subset \MG \subset [\K]^\MX$, and that $\sfat_2(\MG) = \sfat_2(\MH)$. If $\MH$ is $\ell$-irreducible, then so is $\MG$.
\end{lemma}
\begin{proof}
  The $\ell$-irreducibility of $\MH$ implies that for any $\K$-ary $\MX$-valued tree $\bx$ of depth $\ell$, there is some choice of $k_1, \ldots, k_\ell \in [\K]$ so that
  $$
\sfat_2(\MG|_{(\bx_1, k_1), \ldots, (\bx_\ell(k_{1:\ell-1}), k_\ell)}) \geq \sfat_2(\MH|_{(\bx_1, k_1), \ldots, (\bx_\ell(k_{1:\ell-1}), k_\ell)}) = \sfat_2(\MH) = \sfat_2(\MG).
$$
But since $\MG|_{(\bx_1, k_1), \ldots, (\bx_\ell(k_{1:\ell-1}), k_\ell)} \subset \MG$, the inequality above must be an equality, and this ensures that $\MG$ is $\ell$-irreducible.
\end{proof}

\subsection{Properties of SOA hypotheses}
\begin{replemma}{lem:soa-stability}
Suppose $\MH \subset \MG$, $\sfat_2(\MH) = \sfat_2(\MG)$, and that $\MH$ is irreducible. Then  for all $x \in \MX$, $|\soa{\MH}{x} - \soa{\MG}{x}| \leq 1$. 
\end{replemma}
\begin{proof}
Fix any $x \in \MX$, and let $k := \soa{\MH}{x}$. Then $\sfat_2(\MG|_{(x,k)}) \geq \sfat_2(\MH|_{(x,k)}) = \sfat_2(\MH) = \sfat_2(\MG)$, and so $\sfat_2(\MG|_{(x,k)}) = \sfat_2(\MG)$. By Lemma \ref{lem:consec-k} and Definition \ref{def:soa}, we have that $\soa{\MG}{x} \in \{k-1, k, k+1\}$, as desired.
\end{proof}

\begin{lemma}
  \label{lem:many-irred}
  Suppose $\MG \subset [\K]^\MX$ is $\ell$-irreducible. Consider any $\ell' \leq \ell$, and any set $\ba \subset \MX \times [\K]$ of size $|\ba| \leq \ell'$, so that each $(x,y) \in \ba$ satisfies $y = \soa{\MG}{x}$. Then $\MG' := \MG|_\ba$
is $(\ell - \ell')$-irreducible and satisfies $\sfat_2(\MG') = \sfat_2(\MG)$. %
\end{lemma}
\begin{proof}
We first prove the statement for the case $\ell' = 1$. Consider some $(x,y) \in \MX \times [\K]$, so that $y = \soa{\MG}{x}$. By Definition \ref{def:soa}, for $\MG' := \MG|_{(x,y)}$, we have $\sfat_2(\MG') = \sfat_2(\MG)$. By Lemma \ref{lem:lm1-irred}, there is some $y' \in [\K]$ so that $\sfat_2(\MG|_{(x,y')}) = \sfat_2(\MG)$ and so that $\MG|_{(x,y')}$ is $(\ell-1)$-irreducible. By Definition \ref{def:soa} we must have $y \in \{y'-1,y', y'+1\}$ and $\MG|_{(x,y)}$ is $(\ell-1)$-irreducible as well.
  
We now prove the statement for general $\ell'$ by induction. Suppose the statement holds for some value $\ell' < \ell$. Consider some set $\ba \subset \MX \times [\K]$ of size $|\ba| = \ell'+1$, and write $\ba = \tilde \ba \cup \{(x,y) \}$, for $|\tilde \ba| = \ell'$ and some $(x,y) \in \MX \times [\K]$. By the inductive hypothesis we have that $\MG|_{\tilde \ba}$ is $(\ell - \ell')$-irreducible and satisfies $\sfat_2(\MG|_{\tilde \ba}) = \sfat_2(\MG)$. By the case $\ell' = 1$ proven above we have that $(\MG|_{\tilde \ba})|_{(x,y)} = \MG|_{\ba}$ is $(\ell - \ell' - 1)$-irreducible and satisfies $\sfat_2(\MG|_{\ba}) = \sfat_2(\MG|_{\tilde \ba}) = \sfat_2(\MG)$, as desired.
  \end{proof}

The below lemma is analogous to \cite[lemma 4.4]{ghazi_sample_2020}.
\begin{lemma}
  \label{lem:irred-sfat-bound}
  For a class $\MF \subset [\K]^\MX$ with $\sfat_2(\MF) = d$, set
  $$
\tilde \MF_{d+1} := \left\{ \soaf{\MG} \ : \ \MG \subset \MF,\ \text{$\MG$ is nonempty and $(d+1)$-irreducible} \right\}.
$$
Then $\sfat_2(\tilde \MF_{d+1}) = d$ as well.
\end{lemma}
\begin{proof}
  Note that $\MF \subset \tilde \MF_{d+1}$, since for any $f \in \MF$, $\{ f \}$ is $\ell$-irreducible for all $\ell \in \BN$, and $\soaf{\{ f \}} = f$. Thus $\sfat_2(\tilde \MF_{d+1}) \geq d$. To see the upper bound on $\sfat_2(\tilde \MF_{d+1})$, suppose for the purpose of contradiction that $\tilde \MF_{d+1}$ shatters an $\MX$-valued binary tree $\bx$ of depth $d+1$. Let $\bs$ be a witness tree to this shattering. We will show that $\MF$ also shatters $\bx$ (witnessed by $\bs$), which leads to the desired contradiction.

  Fix any sequence $(k_1, \ldots, k_{d+1}) \in \{1,2\}^{d+1}$. Since $\bx$ is shattered by $\tilde \MF_{d+1}$, there must be some $\MG \subset \MF$ that is $(d+1)$-irreducible so that for $1 \leq t \leq d+1$,
  $$
(3 - 2k_t) \cdot (\soa{\MG}{\bx_t(k_{1:t-1})} - \bs_t(k_{1:t-1})) \geq 1.
$$
For $1 \leq t \leq d+1$, set $y_t := \soa{\MG}{\bx_t(k_{1:t-1})}$. Since $\MG$ is $(d+1)$-irreducible, by Lemma \ref{lem:many-irred}, we have that
$$
\sfat_2(\MG|_{(\bx_1, y_1), (\bx_2(k_1), y_2), \ldots, (\bx_{d+1}(k_{1:d}), y_{d+1})}) = \sfat_2(\MG) \geq 0.
$$
Thus there must be some $f \in \MG \subset \MF$ so that for $1 \leq t \leq d+1$, $f(\bx_t(k_{1:t-1})) = y_t$. Since the above argument holds for any choice of $(k_1, \ldots, k_{d+1}) \in \{1,2\}^{d+1}$, it follows that $\bx$ is shattered by $\MF$, witnessed by $\bs$. 
\end{proof}

\section{Proofs for the \ReduceTreeReg algorithm (Section \ref{sec:reducetree})}
In this section we introduce the \ReduceTreeReg algorithm reference in Section \ref{sec:reducetree} and state its main guarantee of weak stability reference in Lemma \ref{lem:weak-stability-informal} (the informal version of Lemmas \ref{lem:approx-stability} and Lemma \ref{lem:s-size}). The algorithm and its analysis is very similar to that in \cite{ghazi_sample_2020}; we provide all proofs for completeness, but indicate the corresponding results in \cite{ghazi_sample_2020} where appropriate.
  
Suppose $\MF \subset \BR^\MX$ (e.g., $\MF \subset [\K]^\MX$); for each $\alpha > 0$, and a distribution $P$ on $\MX \times \BR$, define
$$
\MF_{P,\alpha} := \{ f \in \MF : \err{P}{f} \leq \alpha \}.
$$

For a dataset $S_n \in (\MX \times \BR)^n$, note that, under the event $\sup_{f \in \MF} \left| \err{P}{f} - \err{\emp}{f} \right| \leq \alpha_0$, for each $\alpha \in [0,1]$ it holds that
\begin{align}
  \label{eq:sandwich}
\MF_{\emp,\alpha-2\alpha_0} \subset \MF_{P, \alpha-\alpha_0} \subset \MF_{\emp, \alpha}.
\end{align}

The below lemma is analogous to Lemma 4.7 of \cite{ghazi_sample_2020}; the proof is almost identical to that in \cite{ghazi_sample_2020}, but we provide it for completeness.
\begin{lemma}
  \label{lem:equal-to-soa}
  Fix some $\ell, \ell' \in \BN$ with $\ell > \ell'$ and hypothesis classes $\MH \subset \MG \subset [\K]^\MX$. Suppose we are given $S^\st \in (\MX \times [\K])^{\ell - \ell'}$ so that $\MH|_{S^\st}$ is $\ell$-irreducible, and that
  \begin{equation}
    \label{eq:gh-equal}
    \sfat_2(\MG|_{S^\st}) = \sfat_2(\MH|_{S^\st}) =: q^\st \geq 0.
  \end{equation}
  Suppose that $\bx$ is a $\K$-ary $\MX$-valued tree of depth at most $\ell - \ell'$, and that for all leaves $v$ of $\bx$, $\sfat_2(\MG|_{\ba(v)}) \leq q^\st$. Then there is some leaf $\hat v$ of $\bx$ so that $\|\soaf{\MJ|_{S^\st}} - \soaf{\MJ'|_{\ba(\hat v)}}\|_\infty \leq 4$ for all hypothesis classes $\MJ', \MJ$ satisfying $\MH \subset \MJ' \subset \MG$ and $\MH \subset \MJ \subset \MG$.

  Moveover, the leaf $\hat v$ satisfies:
  \begin{enumerate}
  \item \label{it:gh-q}$\sfat_2(\MG|_{\ba(\hat v)}) = \sfat_2(\MH|_{\ba(\hat v)}) = q^\st$.
  \item \label{it:h-irred} $\MH|_{\ba(\hat v)}$ is $\ell'$-irreducible.
  \end{enumerate}
\end{lemma}
\begin{proof}
  The fact that $\MH|_{S^\st}$ is $\ell$-irreducible together with (\ref{eq:gh-equal}) and Lemma \ref{lem:irred-hg} gives that $\MG|_{S^\st}$ and $\MJ|_{S^\st}$ are $\ell$-irreducible for any $\MJ$ satisfying $\MH \subset \MJ \subset \MG$.

  We now define a leaf $\hat v$ of $\bx$ as follows: first choose $k_1 := \soa{\MH|_{S^\st}}{\bx_1}$, then for $t > 2$, if the node corresponding to the sequence $(k_1, \ldots, k_{t-1})$ is not a leaf of $\bx$, set $k_t := \soa{\MH|_{S^\st}}{\bx_t(k_{1:t-1})}$. This process will stop (i.e., the node corresponding to $(k_1, \ldots, k_t)$ will be a leaf for some $t$) after at most $\ell - \ell'$ steps (since $\height(\bx) \leq \ell - \ell'$), and we let the resulting leaf be $\hat v$. Since $|\ba(\hat v)| \leq \height(\bx) \leq \ell - \ell'$ and for each $(x,y) \in \ba(\hat v)$ we have $y = \soa{\MH|_{S^\st}}{x}$, by Lemma \ref{lem:many-irred}, it holds that $\MH|_{S^\st \cup \ba(\hat v)}$ is $\ell'$-irreducible and satisfies $\sfat_2(\MH|_{S^\st \cup \ba(\hat v)}) = \sfat_2(\MH|_{S^\st}) = q^\st$.  

  Next, using the assumption that $\sfat_2(\MH|_{S^\st}) = q^\st \geq \sfat_2(\MG|_{\ba(\hat v)})$ (as $\hat v$ is a leaf of $\bx$) together with the $\ell'$-irreducibility of $\MH|_{S^\st \cup \ba(\hat v)}$, we see that for any $\K$-ary $\MX$-valued tree $\by$ of depth at most $\ell'$, there is some leaf $u$ of $\by$ so that
  \begin{align}
    \sfat_2(\MH|_{\ba(\hat v) \cup \ba(u)}) & \geq \sfat_2(\MH|_{S^\st \cup \ba(\hat v) \cup \ba(u)}) \label{eq:remove-sst}\\
                                            & = \sfat_2(\MH|_{S^\st \cup \ba(\hat v)}) \nonumber\\
                                            & = \sfat_2(\MH|_{S^\st}) \label{eq:only-sst}\\
                                            &\geq \sfat_2(\MG|_{\ba(\hat v)}) \geq \sfat_2(\MH|_{\ba(\hat v)}).\label{eq:gh-av}
  \end{align}
  Since $\MH|_{\ba(\hat v) \cup \ba(u)} \subset \MH|_{\ba(\hat v)}$, it follows that the inequalities in (\ref{eq:remove-sst}) and (\ref{eq:gh-av}) are equalities. For any $x \in \MX$, interpret it as a depth-0 tree $\by$ whose root node is labeled by $x$, set $k(x) \in [\K]$ to be the value ensuring that (\ref{eq:remove-sst}) through (\ref{eq:gh-av}) holds. It then follows from Lemma \ref{lem:consec-k} that
  \begin{equation}
    \label{eq:hv-k-close}
    | \soa{\MH|_{\ba(\hat v)}}{x} - k(x)| \leq 1
  \end{equation}
  for all $x \in \MX$.

  From equalities (\ref{eq:remove-sst}) through (\ref{eq:only-sst}), we have that for all $x \in \MX$ (again letting the tree $\by$ be the depth-0 tree whose root is labeled by $x$), $\sfat_2(\MH|_{S^\st}) = \sfat_2(\MH|_{S^\st \cup \{ (x, k(x)) \}})$. Thus
  \begin{equation}
    \label{eq:hsst-k-close}
    | \soa{\MH|_{S^\st}}{x}  - k(x) | \leq 1
  \end{equation}
  for all $x \in \MX$.

  Since $\MH|_{S^\st}$ is irreducible, by Lemma \ref{lem:soa-stability}, we have that for all $x \in \MX$ and $\MJ$ satisfying $\MH \subset \MJ \subset \MG$,
  \begin{align}
    \label{eq:hj-sst-close}
| \soa{\MH|_{S^\st}}{x} - \soa{\MJ|_{S^\st}}{x} | \leq 1.
  \end{align}
  From (\ref{eq:hv-k-close}), (\ref{eq:hsst-k-close}), (\ref{eq:hj-sst-close}) and the triangle inequality we see that $\| \soaf{\MJ|_{S^\st}} - \soaf{\MH|_{\ba(\hat v)}} \|_\infty \leq 3$. This establishes the desired closeness of SOA hypotheses for $\MJ' = \MH$. Before establishing this for all $\MH'$ satisfying $\MH \subset \MJ' \subset \MG$, we first show items \ref{it:gh-q} and \ref{it:h-irred}. 

  Using (\ref{eq:only-sst}) and (\ref{eq:gh-av}) (which, as we argued above, are all equalities) gives that $\sfat_2(\MH|_{\ba(\hat v)}) = \sfat_2(\MG|_{\ba(\hat v)}) = q^\st$, establishing item \ref{it:gh-q}. Item \ref{it:h-irred} is a consequence of the fact that $\MH|_{S^\st \cup \ba(\hat v)}$ is $\ell'$-irreducible, $\sfat_2(\MH|_{S^\st \cup \ba(\hat v)}) = \sfat_2(\MH|_{\ba(\hat v)})$ (by (\ref{eq:remove-sst}) through (\ref{eq:gh-av})), and Lemma \ref{lem:irred-hg}.
  
  Items \ref{it:gh-q} and \ref{it:h-irred} together with Lemma \ref{lem:soa-stability} imply that for any hypothesis class $\MJ'$ satisfying $\MH \subset \MJ' \subset \MG$, we have that
  \begin{equation}
    \label{eq:jp-v-close}
\| \soaf{\MJ'|_{\ba(\hat v)}} - \soaf{\MH'|_{\ba(\hat v)}} \|_\infty \leq 1.
\end{equation}
Then (\ref{eq:hv-k-close}), (\ref{eq:hsst-k-close}), (\ref{eq:hj-sst-close}), and (\ref{eq:jp-v-close}) together with the triangle inequality give that $\| \soaf{\MJ'|_{\ba(\hat v)}} - \soaf{\MJ|_{S^\st}} \|_\infty \leq 4$ for all $\MJ', \MJ$ satisfying $\MH \subset \MJ' \subset \MG,\ \MH \subset \MJ \subset \MG$.
\end{proof}

\subsection{\ReduceTreeReg algorithm}
\begin{algorithm}[!htp]
  \caption{\bf \ReduceTreeReg}\label{alg:reduce-tree}
  \KwIn{Parameters $n, \ell' \in \BN$, $\alpha_\Delta, \alpha_1 \in \BR_+$. Distribution $\emp$ over $\MX$. Hypothesis class $\MF$, with $d := \sfat_2(\MF)$.}
 \arxiv{\begin{enumerate}[leftmargin=14pt,rightmargin=20pt,itemsep=1pt,topsep=1.5pt]}
 \colt{\begin{enumerate}[leftmargin=20pt,rightmargin=20pt,itemsep=1pt,topsep=1.5pt]}
\item Initialize a counter $t = 1$ ($t$ counts the depth of the tree constructed at each step of the algorithm).
\item For $1 < t \leq d+1$, set $\alpha_t := \alpha_1 - (t-1) \cdot \alpha_\Delta$. 
\item For $1 \leq t \leq d$, set $\ell_t := \ell' \cdot 2^{t}$.\label{it:define-kt}
\item Initialize $\hat \bx\^0 = \{ v_0 \}$ to be a tree with a single (unlabeled) leaf $v_0$. (In general $\hat \bx\^t$ will be the tree produced by the algorithm after step $t$ is completed.)
\item Initialize $\hat \ML_1 = \{ v_0 \}$. (In general $\hat \ML_t$ will be the set of leaves of the tree before step $t$ is started.)
\item For $t \in \{1,2, \ldots, d\}$:
  \begin{enumerate}
  \item For each leaf $v \in \hat \ML_{t}$ and $\alpha \geq 0$, set $\gRes{\alpha}{v} := \MF_{\emp, \alpha}|_{\ba(v)}$. (Note that since the only way the tree changes from round to round is by adding children to existing nodes, $\ba(v)$ will never change for a node $v$ that already exists.)
  \item \label{it:sup-ldim-alg} Let $\hat w_t^\st := \max_{v \in \hat \ML_{t}} \sfat_2(\gRes{\alpha_t}{v})$ be the maximum sequential fat-shattering dimension of any of the classes $\gRes{\alpha_t}{v}.$ 

    Also let $\hat \ML_{t}' := \{ v \in \hat \ML_{t} : \sfat_2(\gRes{\alpha_t}{v}) = \hat w_t^\st\}$.
  \item If $\hat w_t^\st < 0$, halt and output ERROR. {\it (We show that this never occurs under appropriate assumptions in Lemma \ref{lem:exists-irred-node}.)}\label{it:halt-error}
  \item \label{it:ldim-break-step} 
  If there is some $v \in \hat \ML_{t}'$ so that $\sfat_2(\hat{\MG}(\alpha_t-\alpha_\Delta, v)) = \sfat_2(\hat{\MG}(\alpha_t, v)) \geq 0$ and $\hat{\MG}(\alpha_t - \alpha_\Delta, v)$ is $\ell_t$-irreducible, then break out of the loop and go to step~\ref{it:tfinal-define}.
  \item \label{it:red-loop} Else, for each node $v \in \hat\ML_{t}'$:
    \begin{enumerate}
    \item If $\gRes{\alpha_t}{v}$ is empty or $\sfat_2(\hat{\MG}(\alpha_t-\alpha_\Delta, v)) < \sfat_2(\hat{\MG}(\alpha_t, v))$, move on to the next $v$.
    \item \label{it:decrease-ldim} Else, we must have that $\hat{\MG}(\alpha_t - \alpha_\Delta, v)$ is not $\ell_t$-irreducible. Let $ \ell_v$ be chosen as small as possible so that $\gRes{\alpha_t - \alpha_\Delta}{v}$ is not $\ell_v$-irreducible; then $ \ell_v \leq \ell_t$. %
      Then there is some $\K$-ary $\MX$-valued  tree $\bx$ of depth $ \ell_v$, so that for any choice of $k_1, \ldots, k_{ \ell_v} \in [\K]$, we have
      \begin{equation}
        \label{eq:split-reduce}
\sfat_2(\gRes{\alpha_t - \alpha_\Delta}{v}|_{(\bx_1, k_1), \ldots, (\bx_{ \ell_v}(k_{1: \ell_v-1}),k_{ \ell_v})}) < \sfat_2(\gRes{\alpha_t - \alpha_\Delta}{v}).
      \end{equation}
      
      \item Attach the tree $\bx$ to $\hat \bx\^{t-1}$ via the leaf $v$ (per Definition \ref{def:attach}).\label{it:attach-trees}
      \end{enumerate}
      \item Let the current tree (with the additions of the previous step) be denoted by $\hat \bx\^{t}$, and let $\hat \ML_{t+1}$ be the list of the leaves of $\hat \bx\^{t}$, i.e., the nodes which have not (yet) been assigned labels or children.
  \end{enumerate}
\item \label{it:tfinal-define} Let $t_{\final}$ be the final value of $t$ the algorithm {\it completed} the loop of step~\ref{it:red-loop} for before breaking out of the above loop (i.e., if the break at step~\ref{it:ldim-break-step} was taken at step $t$, then $t_{\final} = t-1$; if the break was never taken, then $t_{\final} = d$). Let $\hat w_{t_{\final}+1}^\st$ and $\hat \ML_{t_{\final}+1}'$ be defined as in Step~\ref{it:sup-ldim-alg}.
\item
Output the set $\hMLp := \hat \ML_{t_{\final}+1}'$ of leaves of the tree $\hat \bx\^{t_{\final}}$, and the tree $\hat \bx := \hat \bx\^{t_{\final}}$.
  Finally, output the set 
  \begin{equation}
    \label{eq:redtree-output}
\hat \MS := \{ \soaf{\gRes{\alpha_{t_{\final} + 1} - 2\alpha_\Delta/3}{v}} : \text{$v \in \hMLp$ and $\gRes{\alpha_{t_{\final}+1} - 2\alpha_\Delta/3}{v}$ is $\ell'$-irreducible \& nonempty} \}.
  \end{equation}
  \end{enumerate}
\end{algorithm}

Throughout this section we fix a positive integer $\K$, a distribution $P$ on $\MX \times [\K]$, a function class $\MF \subset [\K]^\MX$, and write $d := \sfat_2(\MF)$. The algorithm \ReduceTreeReg takes as input some parameters $k' \in \BN$, $\alpha_1, \alpha_\Delta > 0$, as well as some dataset $S_n \in (\MX \times [\K])^n$ consisting of $n$ samples $(x,k) \in \MX \times [\K]$, which is accessed through its empirical distribution $\emp$. Given these parameters, define the event $E_{\good}$ to be
\begin{equation}
  \label{eq:def-egood}
E_{\good} := \left\{ \sup_{f \in \MF} \left| \err{P}{f} - \err{\emp}{f} \right| \leq \frac{\alpha_\Delta}{6} \right\}.
\end{equation}
Though the algorithm \ReduceTreeReg is well-defined regardless of whether $E_{\good}$ holds, several of the lemmas in this section regarding correctness of \ReduceTreeReg will rely on $E_{\good}$ holding; in Section \ref{sec:reglearn} we will show that when the dataset $S_n$ is drawn according to an appropriate distribution, $E_{\good}$ will hold with high probability with respect to this draw. 

The below lemma is analogous to Lemma 5.1 of \cite{ghazi_sample_2020}.
\begin{lemma}
  \label{lem:exists-irred-node}
  Suppose the inputs $S_n, \alpha_1, \alpha_\Delta$ of \ReduceTreeReg are chosen so that $\MF_{\emp, \alpha_1 - (d+1) \cdot \alpha_\Delta}$ is nonempty. Then \ReduceTreeReg never halts and outputs ERROR at step \ref{it:halt-error}. Moreover, the set $\hat \ML'$ output by \ReduceTreeReg satisfies the following property: letting $t = t_{\final} + 1 \in [d+1]$, there is some leaf $v \in \hat \ML'$ so that $\sfat_2(\gRes{\alpha_t - \alpha_\Delta}{v}) = \sfat_2(\gRes{\alpha_t}{v}) \geq 0$ and $\gRes{\alpha_t - \alpha_\Delta}{v}$ is $\ell_t$-irreducible. 
\end{lemma}
\begin{proof}
  If, for some $t$, the algorithm \ReduceTreeReg breaks at step \ref{it:ldim-break-step}, then the inclusion of the lemma is immediate: the condition to break in step \ref{it:ldim-break-step} gies that for some $v \in \hat \ML_t' = \hat \ML_{t_{\final}+1}'$, we have that $\sfat_2(\gRes{\alpha_t - \alpha_\Delta}{v}) = \sfat_2(\gRes{\alpha_t}{v}) \geq 0$ and $\gRes{\alpha_t - \alpha_\Delta}{v}$ is $\ell_t$-irreducible.

  Next we show that the algorithm never halts and outputs ERROR at step \ref{it:halt-error}. Note that for each $1 \leq t \leq d+1$, the tree $\hat \bx\^{t-1}$ has the property that each non-leaf node has exactly $\K$ children, one corresponding to each label in $[\K]$ (this is by Definition \ref{def:xv-tree}); thus, we have that, for each $t$, and each $\alpha \geq 0$,
  \begin{equation}
    \label{eq:leaves-decompose}
\MF_{\emp, \alpha} = \bigcup_{v \in \hat \ML_t} \MF_{\emp, \alpha}|_{\ba(v)} = \bigcup_{v \in \hat \ML_t} \gRes{\alpha}{v}.
\end{equation}
Since $\MF_{\emp, \alpha_1 - (d+1) \cdot \alpha_\Delta}$ is nonempty (by assumption), $\MF_{\emp, \alpha_t} \supset \MF_{\emp, \alpha_1 - (d+1) \cdot \alpha_\Delta}$ is nonempty for $1 \leq t \leq d+1$. Thus there is some $v \in \hat \ML_t$ so that $\gRes{\alpha_t}{v}$ is nonempty, i.e., $\hat w_t^\st = \max_{v \in \hat \ML_t} \sfat_2(\gRes{\alpha_t}{v}) \geq 0$. 

Otherwise, the algorithm performs a total of $d$ iterations. We claim that $\hat w_{d+1}^\st = 0$. We first show that for all $t \geq 1$, $\hat w_{t+1}^\st < \hat w_t^\st$. To see this, note that each leaf $v$ in $\hat \ML_{t+1}$ belongs to one of the following categories:
\begin{itemize}
\item $v \in \hat \ML_t \backslash \hat \ML_t'$. (This includes the case that $\gRes{\alpha_t}{v}$ is empty.) In this case, we have
  $$
\sfat_2(\gRes{\alpha_{t+1}}{v}) \leq \sfat_2(\gRes{\alpha_t}{v}) < \hat w_t^\st.
$$
\item $v \in \hat \ML_t'$ and $\sfat_2(\gRes{\alpha_t - \alpha_\Delta}{v}) < \sfat_2(\gRes{\alpha_t}{v})$. Using that $\alpha_{t+1} = \alpha_t - \alpha_\Delta$, we obtain
  $$
  \sfat_2(\gRes{\alpha_{t+1}}{v}) = \sfat_2(\gRes{\alpha_t - \alpha_\Delta}{v}) < \sfat_2(\gRes{\alpha_t}{v}) \leq \hat w_t^\st.
  $$
\item $v$ corresponds to some leaf $u$ of some $\K$-ary $\MX$-valued tree $\bx$ which is attached to $\hat \bx\^{t-1}$ via some leaf $v_0$ of $\hat \bx\^{t-1}$ (as constructed in steps \ref{it:decrease-ldim} and \ref{it:attach-trees} of the algorithm). Then $\ba(v) = \ba(v_0) \cup \ba(u)$, and so
  \begin{align*}
    \sfat_2(\gRes{\alpha_{t+1}}{v}) & \leq \sfat_2(\MF_{\emp, \alpha_t - \alpha_\Delta}|_{\ba(v)}) \\
                                    & =\sfat_2(\MF_{\emp, \alpha_t - \alpha_\Delta}|_{\ba(v_0) \cup \ba(u)}) \\
                                    &< \sfat_2(\MF_{\emp, \alpha_t - \alpha_\Delta}|_{\ba(v_0)}) \\
    & \leq \hat w_t^\st,
  \end{align*}
  where the strict inequality follows from (\ref{eq:split-reduce}) (the set $\{ (\bx_1, k_1), \ldots, (\bx_{\ell_v}(k_{1:\ell_v-1}), k_{\ell_v})\}$ is exactly $\ba(u)$), and the last inequality follows from the fact that $v_0 \in \hat \ML_t$. 
\end{itemize}
Thus all leaves $v$ in $\hat \ML_{t+1}$ satisfy $\sfat_2(\gRes{\alpha_{t+1}}{v}) < \hat w_t^\st$, i.e., $\hat w_{t+1}^\st < \hat w_t^\st$. 
Since $\hat w_1^\st \leq d$ as $\gRes{\alpha_t}{v} \subset \MF$, we obtain that $\hat w_{d+1}^\st \leq 0$. We have already shown that $\hat w_{d+1}^\st \geq 0$, and so $\hat w_{d+1}^\st = 0$.

By assumption, $\MF_{\emp, \alpha_{d+1} - \alpha_\Delta} = \MF_{\emp, \alpha_1 - (d+1) \cdot \alpha_\Delta}$ is nonempty, and therefore, by (\ref{eq:leaves-decompose}), and therefore, for some leaf $v \in \hat \ML_{d+1}' = \hat \ML'$, we have $\sfat_2(\gRes{\alpha_{d+1}}{v}) = \sfat_2(\gRes{\alpha_{d+1} - \alpha_\Delta}{v}) = 0$. Moreover, $\gRes{\alpha_{d+1} - \alpha_\Delta}{v}$ is $\ell_{d+1}$-irreducible since a class with sequential fat-shattering dimension 0 is $\ell$-irreducible for all $\ell \in \BN$ (Lemma \ref{lem:0dim-irred}).
\end{proof}

The below lemma is analogous to Lemma 5.2 of \cite{ghazi_sample_2020}.
\begin{lemma}
  \label{lem:depth-ub}
For all $t$ the tree $\hat \bx\^t$ of Algorithm \ref{alg:reduce-tree} has depth at most $\ell_{t+1} - \ell'$. In particular, the tree $\hat \bx$ has depth at most $\ell_{t_{\final} + 1} - \ell'$. 
\end{lemma}
\begin{proof}
  We prove by induction that the depth of $\hat \bx\^t$, denoted $\height(\hat \bx\^t)$, satisfies $\height(\hat \bx\^t) \leq \ell_{t+1} - \ell' = \ell' \cdot 2^{t+1} - \ell'$. For the base case, note that $\height(\hat \bx\^0) = 0 < 2\ell' - \ell' = \ell' \cdot 2^t - \ell'$. For any $t > 0$, The only step of \ReduceTreeReg at which $\hat \bx\^{t-1}$ is modified (to produce $\hat \bx\^{t}$) is step \ref{it:decrease-ldim}, when some trees of depth at most $\ell_t$ are attached to $\hat \bx\^{t-1}$ via some leaves. Thus we have
  $$
\height(\hat \bx\^t) \leq \height(\hat \bx\^{t-1}) + \ell_t \leq \ell_t - \ell' + \ell_t = \ell_{t+1} - \ell'.
$$
\end{proof}

For each $\alpha > 0$ and $t \in [d+1]$, define the set:
\begin{equation}
  \label{eq:define-m}
\MM_{\alpha,t} := \left\{ S \in ( \MX \times [\K])^{\leq (\ell_t - \ell')} : \substack{
\text{$\MF_{P,\alpha-\alpha_\Delta/3}|_{S}$ is $\ell_t$-irreducible and nonempty,}  \\ 
\text{and $\sfat_2(\MF_{P,\alpha-\alpha_\Delta/3}|_{S}) = \sfat_2(\MF_{P,\alpha+\alpha_\Delta/3}|_{S})$}
}  \right\}.
\end{equation}
Notice that $\MM_{\alpha, t}$ depends on $\MF, P$. 
The below lemma is analogous to Lemma 5.3 of \cite{ghazi_sample_2020}.
\begin{lemma}
  \label{lem:m-nonempty}
Suppose that $E_{\good}$ holds. Then for $t = t_{\final} + 1$, the set $\MM_{\alpha_t - \alpha_\Delta/2, t}$ is nonempty.
\end{lemma}
\begin{proof}
  Set $t = t_{\final} + 1$. 
  Let $v$ be a node in the set $\hat \ML'$ (so that $v$ is a leaf of $\hat \bx\^{t_{\final}} = \hat \bx\^{t-1}$) produced by \ReduceTreeReg as guaranteed by Lemma \ref{lem:exists-irred-node}, i.e., so that $\sfat_2(\gRes{\alpha_t - \alpha_\Delta}{v}) = \sfat_2(\gRes{\alpha_t}{v}) \geq 0$ and so that $\gRes{\alpha_t - \alpha_\Delta}{v}$ is $\ell_t$-irreducible. Since the event $E_{\good}$ holds,
  $$
\gRes{\alpha_t - \alpha_\Delta}{v} = \MF_{\emp, \alpha_t - \alpha_\Delta}|_{\ba(v)} \subset \MF_{P, \alpha_t - 5\alpha_\Delta/6}|_{\ba(v)} \subset \MF_{P, \alpha_t - \alpha_\Delta/6}|_{\ba(v)} \subset \MF_{\emp, \alpha_t}|_{\ba(v)} = \gRes{\alpha_t}{v}.
$$
It follows from Lemma \ref{lem:irred-hg} that $\MF_{P, \alpha_t - 5\alpha_\Delta/6}|_{\ba(v)}$ is $\ell_t$-irreducible and that $\sfat_2(\MF_{P, \alpha_t - \alpha_\Delta/6}|_{\ba(v)}) = \sfat_2(\MF_{P, \alpha_t - 5\alpha_\Delta/6}|_{\ba(v)}) \geq 0$. Since the depth of the tree $\hat \bx\^{t-1} = \hat \bx\^{t_{\final}}$ is at most $\ell_t - \ell'$ (Lemma \ref{lem:depth-ub}), it follows that the number of tuples in $\ba(v)$ is at most $\ell_t - \ell'$; thus $\ba(v) \in \MM_{\alpha_t - \alpha_\Delta/2,t}$. 
\end{proof}

For any $\alpha > 0, t \in [d+1]$ for which $\MM_{\alpha,t}$ is nonempty, define:
\begin{align}
  \label{eq:m-argmaxes}
S_{\alpha,t}^\st \in \argmax_{S \in \MM_{\alpha,t}} \left\{ \sfat_2(\MF_{P,\alpha}|_S) \right\} , \qquad q_{\alpha,t}^\st := \max_{S \in \MM_{\alpha,t}} \left\{ \sfat_2(\MF_{P,\alpha}|_S) \right\} \geq 0.
\end{align}
Also set
\begin{equation}
  \label{eq:define-sigma-at}
\sigma_{\alpha,t}^\st := \soaf{\MF_{P,\alpha}|_{S_{\alpha,t}^\st}}.
\end{equation}

The below lemma is analogous to Lemma 5.4 of \cite{ghazi_sample_2020}.
\begin{lemma}[``Weak stability'']
  \label{lem:approx-stability}
  Suppose that $E_{\good}$ holds and $\MF_{\emp, \alpha_1 - d \cdot \alpha_\Delta} = \MF_{\emp, \alpha_{d+1}}$ is nonempty. Then the following holds: for $t = t_{\final} + 1 \in [d+1]$ and some leaf $\hat v\in \hat \ML'$, we have $\| \sigma_{\alpha_t - \alpha_\Delta/2,t}^\st - \soaf{\gRes{\alpha_t - 2\alpha_\Delta/3}{\hat v}} \|_\infty \leq 5$. (In particular, for this $t$, $\sigma_{\alpha_t - \alpha_\Delta/2, t}^\st$ is well-defined, i.e., $\MM_{\alpha_t - \alpha_\Delta/2,t}$ is nonempty.)

  Moreover, $\gRes{\alpha_t - 2\alpha_\Delta/3}{\hat v}$ is $\ell'$-irreducible and nonempty, and $\sfat_2(\gRes{\alpha_t - 2\alpha_\Delta/3}{\hat v}) = q_{\alpha_t - \alpha_{\Delta}/2, t}^\st \geq 0$.
\end{lemma}
\begin{proof}
  By Lemma \ref{lem:exists-irred-node}, for $t := t_{\final} + 1 \in [d+1]$, there is some leaf $v' \in \hat \ML'$ so that $\sfat_2(\gRes{\alpha_t - \alpha_\Delta}{v}) = \sfat_2(\gRes{\alpha_t}{v}) \geq 0$ and $\gRes{\alpha_t - \alpha_\Delta}{v}$ is $\ell_t$-irreducible. Since the event $E_{\good}$ holds, for each node $v$ of the tree $\hat \bx$ output by \ReduceTreeReg, we have that
    \begin{align}
  \MF_{\emp, \alpha_t - \alpha_\Delta}|_{\ba(v)} \subset
  \MF_{P,\alpha_t - 5 \alpha_\Delta/6}|_{\ba(v)} \subset
    \MF_{\emp, \alpha_t - 4\alpha_\Delta/6}|_{\ba(v)} \nonumber\\
    \label{eq:6-inclusions}
    \subset
    \MF_{P, \alpha_t - 3\alpha_\Delta/6}|_{\ba(v)} \subset
  \MF_{\emp, \alpha_t - 2\alpha_\Delta/6}|_{\ba(v)} \subset
  \MF_{P, \alpha_t - \alpha_\Delta/6}|_{\ba(v)} \subset \MF_{\emp, \alpha_t}|_{\ba(v)}.
  \end{align}
  Now we apply Lemma \ref{lem:equal-to-soa} with $\MJ = \MJ' = \MF_{P, \alpha_t - \alpha_\Delta/2}, \MH = \MF_{P, \alpha_t - 5\alpha_\Delta/6}, \MG = \MF_{P, \alpha_t - \alpha_\Delta/6}, \ell = \ell_t, \ell' = \ell'$, $\bx$ equal to the tree $\hat \bx = \hat \bx\^{t_{\final}}$ output by \ReduceTreeReg, and $S^\st = S^\st_{\alpha_t - \alpha_\Delta/2,t}$. Since $t = t_{\final} + 1$, Lemma \ref{lem:m-nonempty} guarantees that $S^\st_{\alpha_t - \alpha_\Delta/2,t}$ is well-defined (i.e., $\MM_{\alpha_t - \alpha_\Delta/2,t}$ is nonempty). We check that the preconditions of Lemma \ref{lem:equal-to-soa} hold: First, note that (\ref{eq:gh-equal}) holds by definition of $\MM_{\alpha_t - \alpha_\Delta/2,t}$ in (\ref{eq:define-m}) and since $S^\st \in \MM_{\alpha_t - \alpha_\Delta/2,t}$. Moreover, $\MH|_{S^\st} = \MF_{P, \alpha_t - \alpha_\Delta/2 - \alpha_\Delta/3}|_{S^\st}$ is $\ell_t$-irreducible, again by (\ref{eq:define-m}) and since $S^\st \in \MM_{\alpha_t - \alpha_\Delta/2, t}$. By definition of $q_{\alpha,t}^\st$ in (\ref{eq:m-argmaxes}), we have
  $$
q_{\alpha_t - \alpha_\Delta/2,t}^\st = \sfat_2(\MF_{P, \alpha_t - 5\alpha_\Delta/6}|_{S^\st}) = \sfat_2(\MF_{P, \alpha_t - \alpha_\Delta/6}|_{S^\st}).
$$
Lemma \ref{lem:depth-ub} establishes that the depth of $\hat \bx$ is at most $\ell_t - \ell'$, so $|\ba(v')| \leq \ell_t - \ell'$. Next, from the guarantee on $v'$ in Lemma \ref{lem:exists-irred-node} (i.e., that $\gRes{\alpha_t - \alpha_\Delta}{v'} = \MF_{\emp, \alpha_t - \alpha_\Delta}|_{\ba(v')}$ is $\ell_t$-irreducible), the fact that $\MF_{\emp, \alpha_t - \alpha_\Delta}|_{\ba(v')} \subset \MF_{P, \alpha_t - 5 \alpha_\Delta/6}|_{\ba(v')}$ (by (\ref{eq:6-inclusions})), and Lemma \ref{lem:irred-hg}, we have that $\MF_{P, \alpha_t - 5 \alpha_\Delta/6}|_{\ba(v')}$ is $\ell_t$-irreducible. (To apply Lemma \ref{lem:irred-hg} here, we need that $\sfat_2(\MF_{P, \alpha_t - 5\alpha_\Delta/6}|_{\ba(v')}) = \sfat_2(\MF_{\emp, \alpha_t - \alpha_\Delta}|_{\ba(v')})$, which follows from $\sfat_2(\MF_{\emp, \alpha_t - \alpha_\Delta}|_{\ba(v')}) = \sfat_2(\MF_{\emp, \alpha_t - \alpha_\Delta}|_{\ba(v')})$ and (\ref{eq:6-inclusions}).) Since also $\sfat_2(\MF_{P, \alpha_t - 5\alpha_\Delta/6}) = \sfat_2(\MF_{P, \alpha_t - \alpha_\Delta/6})$, we have that $\ba(v') \in \MM_{\alpha_t - \alpha_\Delta/2, t}$, so the definition of $q_{\alpha,t}^\st$ gives
$$
q_{\alpha_t - \alpha_\Delta/2,t}^\st \geq \sfat_2(\MF_{P, \alpha_t - \alpha_\Delta/2}|_{\ba(v')}).
$$
Moreover, for any other leaf $u$ of the tree $\hat \bx$, we have, by definition of $\hat \ML' = \hat \ML_{t_{\final} + 1}'$,
$$
\sfat_2(\MF_{P, \alpha_t - \alpha_\Delta/6}|_{\ba(u)}) \leq \sfat_2(\MF_{\emp, \alpha_t}|_{\ba(u)}) \leq \sfat_2(\MF_{\emp, \alpha_t}|_{\ba(v')}) = \sfat_2(\MF_{P, \alpha_t - \alpha_\Delta/2}|_{\ba(v')}) \leq q_{\alpha_t - \alpha_\Delta/2,t}^\st, 
$$
(The first inequality above holds due to (\ref{eq:6-inclusions}), the second inequality is due to the fact that $v' \in \hat \ML_{t_{\final}+1}'$ (see step \ref{it:sup-ldim-alg} of \ReduceTreeReg), and the equality holds due to (\ref{eq:6-inclusions})  and $\sfat_2(\MF_{\emp, \alpha_t - \alpha_\Delta}|_{\ba(v')}) = \sfat_2(\MF_{\emp, \alpha_t}|_{\ba(v')})$.) This completes the verification that all hypotheses of Lemma \ref{lem:equal-to-soa} hold. Then Lemma \ref{lem:equal-to-soa} with $\MJ' = \MJ = \MF_{P, \alpha_t - \alpha_\Delta/2}$, we get that for some leaf $\hat v$ of $\hat \bx$, we have
$$
\| \soaf{\MF_{P, \alpha_t - \alpha_\Delta/2}|_{S^\st}} - \soaf{\MF_{P, \alpha_t - \alpha_\Delta/2}|_{\ba(\hat v)}} \| = \|\sigma_{\alpha_t - \alpha_\Delta/2,t}^\st - \soaf{\MF_{P, \alpha_t - \alpha_\Delta/2}|_{\ba(\hat v)}} \| \leq 4.
$$
Moreover, item \ref{it:gh-q} of Lemma \ref{lem:equal-to-soa} gives that $\sfat_2(\MF_{P, \alpha_t - 5 \alpha_\Delta/6}|_{\ba(\hat v)}) = \sfat_2(\MF_{P, \alpha_t - \alpha_\Delta/6}|_{\ba(\hat v)}) = q_{\alpha_t - \alpha_\Delta/2,t}^\st$, and item \ref{it:h-irred} gives that $\MF_{P, \alpha_t - 5 \alpha_\Delta/6}|_{\ba(\hat v)}$ is $\ell'$-irreducible. From (\ref{eq:6-inclusions}), it follows that $\sfat_2(\MF_{\emp, \alpha_t - 4 \alpha_\Delta/6}|_{\ba(\hat v)}) = \sfat_2(\MF_{P, \alpha_t - \alpha_\Delta/2}|_{\ba(\hat v)}) = \sfat_2(\MF_{\emp, \alpha_t - 2 \alpha_\Delta/6}|_{\ba(\hat v)}) = q_{\alpha_t - \alpha_\Delta/2,t}^\st \geq 0$, and that $\MF_{\emp, \alpha_t - 4\alpha_\Delta/6}|_{\ba(\hat v)} = \gRes{\alpha_t - 2\alpha_\Delta/3}{\hat v}$ is $\ell'$-irreducible (from Lemma \ref{lem:irred-hg}). Then by (\ref{eq:6-inclusions}) and Lemma \ref{lem:soa-stability}, we have
\begin{align*}
  & \| \sigma_{\alpha_t - \alpha_\Delta/2,t}^\st - \soaf{\gRes{\alpha_t-2\alpha_\Delta/3}{\hat v}} \|_\infty \\
  & \leq \| \sigma_{\alpha_t - \alpha_\Delta/2,t}^\st - \soaf{\MF_{P, \alpha_t - \alpha_\Delta/2|_{\ba(\hat v)}}} \|_\infty + \| \soaf{\MF_{P, \alpha_t - \alpha_\Delta/2}}|_{\ba(\hat v)} - \soaf{\gRes{\alpha_t - 2\alpha_\Delta/3}{\hat v}} \|_\infty \leq 4 + 1 = 5.
\end{align*}
Finally, we check that $\hat v \in \hat \ML' = \hat \ML_{t}' = \hat \ML_{t_{\final}+1}'$, i.e., all leaves $u$ of the tree $\hat \bx$ satisfy $\sfat_2(\MF_{\emp, \alpha_t}|_{\ba(u)}) \leq \sfat_2(\MF_{\emp,\alpha_t}|_{\ba(\hat v)})$. This is a consequence of the fact that for all such $u$,
$$
\sfat_2(\MF_{\emp,\alpha_t}|_{\ba(\hat v)}) \geq \sfat_2(\MF_{\emp, \alpha_t - 2\alpha_\Delta/6}|_{\ba(\hat v)}) = q_{\alpha_t - \alpha_\Delta/2,t}^\st  \geq \sfat_2(\MF_{\emp, \alpha_t}|_{\ba(v')}) \geq \sfat_2(\MF_{\emp,\alpha_t}|_{\ba(u)}),
$$
since $v' \in \hat \ML'$ (by definition).
\end{proof}

\begin{lemma}
  \label{lem:s-size}
  The set $\hat \MS$ output by \ReduceTreeReg has size $| \hat \MS | \leq \K^{\ell' \cdot 2^{d+1}}$.
\end{lemma}
\begin{proof}
  We show that for $t \in [d]$, the tree $\hat \bx\^t$ has at most $\prod_{t'=1}^t \K^{\ell_{t'}}$ leaves. This statement is a simple consequence of the fact that $\bx\^0$ has a single leaf, and the tree $\hat \bx\^t$ is formed by attaching a trees of depth at most $\ell_t$ to some of the leaves of $\hat \bx\^{t-1}$. Thus the number of leaves of $\hat \bx\^t$ is at most
  $$
\prod_{t'=1}^d \K^{\ell_{t'}} = \K^{\ell_1 + \cdots + \ell_d} \leq \K^{\ell' \cdot 2^{d+1}}.
  $$
\end{proof}

\section{Proofs for Section \ref{sec:soafilter}: the algorithm \soafilter}
In this section we give proofs for all results in Section \ref{sec:soafilter}, and state several additional lemmas which will be useful in our proofs. 
Throughout we suppose that we are given a hypothesis class $\MF \subset [\K]^\MX$ and write $d := \sfat_2(\MF)$.

\subsection{Existence of reducing trees}
Recall the definition of \emph{reducing tree} from Definition \ref{def:reducing-tree}. Lemma \ref{lem:make-tree} shows that such trees exist.
\begin{lemma}
  \label{lem:make-tree}
  For any class $\MH \subset [\K]^\MX$ with $d := \sfat_2(\MH)$, any sequence $(\ell_t)_{t \geq 0}$ of positive integers, and any $(x,y) \in \MX \times [\K]$ for which $\sfat_2(\MH|_{(x,y)}) < \sfat_2(\MH)$, there is a reducing tree $\bx$ (of depth at least 1) for the pair $(x,y)$, the sequence $(\ell_t)$, and the class $\MH$.

  Moreover, $\bx$ may be chosen so that for each $1 \leq t \leq d$, $\bx$ has at most $\K^{\sum_{t'=0}^{t-1} \ell_{t'}}$ %
  leaves $v$ so that $\sfat_2(\MH|_{\ba(v)}) = d-t$. %
\end{lemma}
\begin{proof}
  We define a sequence $\bx\^0, \bx\^1, \ldots$ of augmented $\MX$-labeled trees. We begin by defining the tree $\bx\^0$, which is of depth 1 and consists of a root, labeled by $x$, together with a single child  (which is its only leaf), for which the edge to the root is labeled by $y$. Now, suppose we are given the tree $\bx\^s$, for some $s \geq 0$. To define the tree $\bx\^{s+1}$, we begin with the tree $\bx\^s$, and then add some subtrees below some of the leaves of $\bx\^s$; we will say that each node of $\bx\^s$ {\it corresponds} to its copy in this copy of $\bx\^s$ in $\bx\^{s+1}$, as well as to its copies in $\bx\^{s+2}, \bx\^{s+3}, \ldots$. %
  In particular, for each leaf $v$ of $\bx\^s$:
  \begin{itemize}
  \item If $\MH|_{\ba(v)}$ is empty or $\ell_t$-irreducible, where $t = d - \sfat_2(\MH|_{\ba(v)})$, we move onto the next leaf.
  \item Otherwise, by the definition of irreducibility, there is some $\K$-ary $\MX$-valued tree $\bx'$ of depth at most $\ell_t$ (again, with $t = d - \sfat_2(\MH|_{\ba(v)})$) so that for each leaf $v'$ of $\bx'$, it holds that $\sfat_2(\MH|_{\ba(v) \cup \ba(v')}) < \sfat_2(\MH|_{\ba(v)})$. Then we attach $\bx'$ to $\bx$ via the leaf $v$, i.e., we label the leaf $v$ with $\bx_1'$ and add a copy of the tree $\bx'$ to $\bx$ rooted at the leaf $v$ (Definition \ref{def:attach}).
  \end{itemize}
  We claim that $\bx\^d = \bx\^{d-1}$, namely that for any leaf $v$ of $\bx\^{d-1}$, we have that $\MH|_{\ba(v)}$ is $\ell_t$-irreducible, where $t = d - \sfat_2(\MH|_{\ba(v)})$. To do this, we introduce the following notation: for $s \geq 0$, let $\MB\^s$ denote the set of leaves of $\bx\^s$ so that  $\MH|_{\ba(v)}$ is not empty or $\ell_t$-irreducible for $t = d - \sfat_2(\MH|_{\ba(v)})$. We now prove the following claim:
  \begin{claim}
    \label{clm:lt-irred}
    For $0 \leq s \leq d-1$, for each leaf $v \in \MB\^s$, %
    $\sfat_2(\MH|_{\ba(v)}) \leq d - s - 1$.
\end{claim}
\begin{proof}[Proof of Claim \ref{clm:lt-irred}]
We use induction on $s$. The base case $s = 0$ is immediate since the tree $\bx\^0$ has a single leaf $v$ which satisfies $\ba(v) = \{ (x,y) \}$, and $\sfat_2(\MH|_{(x,y)}) < \sfat_2(\MH) = d$ is assumed.

  To establish the inductive step, note that any leaf $v \in \MB\^{s+1}$ does not correspond to a leaf $v'$ of $\bx\^s$. Rather, there is some leaf $\tilde v$ of $\bx\^s$ and some tree $\bx'$, as well as some leaf $\tilde v'$ of $\bx'$ so that $v$ is the leaf $\tilde v'$ attached to $\bx\^s$ via $\tilde v$. In particular, we have $\ba(v) = \ba(\tilde v) \cup \ba(\tilde v')$ and $\sfat_2(\MH|_{\ba(\tilde v) \cup \ba(\tilde v')}) < \sfat_2(\MH|_{\ba(\tilde v)})$. By the inductive hypothesis, $\sfat_2(\MH|_{\ba(\tilde v)}) \leq d - s - 1$, and so $\sfat_2(\MH|_{\ba(v)}) \leq d - s - 2$, completing the inductive step.
\end{proof}
We now set $\bx = \bx\^{d-1}$. It follows from Claim \ref{clm:lt-irred} that for all leaves $v$ of $\bx$, either $v \not \in \MB\^s$, in which case $\MH|_{\ba(v)}$ is empty or $\ell_t$-irreducible for $t = d - \sfat_2(\MH|_{\ba(v)})$, or $\sfat_2(\MH|_{\ba(v)}) \leq 0$, i.e., $\MH|_{\ba(v)}$ is empty or $\ell$-irreducible for all $\ell \in \BN$ (Lemma \ref{lem:0dim-irred}).

To establish that $\bx$ is a reducing tree, we need to establish the second item in Definition \ref{def:reducing-tree} regarding $\height(v)$ for leaves $v$ of $\bx$. To do so, we establish the following claim:
\begin{claim}
  \label{clm:bs-depth}
Fix any $0 \leq s \leq d-1$. For each leaf $v$ of $\bx\^s$, letting $t = d - \sfat_2(\MH|_{\ba(v)})$, we have that $\height(v) \leq \sum_{t'=0}^{t-1} \ell_{t'}$.
\end{claim}
\begin{proof}
  We establish the claim using induction on $s$. For the base case $s = 0$, the only leaf $v$ of $\bx\^0$ satisfies $\height(v) = 1$, which is bounded above by $\sum_{t'=0}^{t-1} \ell_{t'}$ (Note that we have $t = d - \sfat_2(\MH|_{\ba(v)}) \geq 1$ here.)

  To establish the inductive step, consider any leaf $v$ of $\bx\^{s+1}$ for some $0 \leq s \leq d-2$, and let $t = d - \sfat_2(\MH|_{\ba(v)})$. If $v$ corresponds to some leaf $v'$ of $\bx\^s$ then certainly $\height(v) \leq \sum_{t'=0}^{t-1} \ell_{t'}$, by the inductive hypothesis. Otherwise (as in the proof of Claim \ref{clm:lt-irred}), there is some leaf $\tilde v$ of $\bx\^s$, some tree $\bx'$ of depth at most $\ell_{\tilde t}$ (where $\tilde t := d - \sfat_2(\MH|_{\ba(\tilde v)})$), as well as some leaf $\tilde v'$ of $\bx'$, so that $v$ is the leaf $\tilde v'$ attached to $\bx\^s$ via $\tilde v$. Moreover, it holds that $\sfat_2(\MH|_{\ba(v)}) = d -t < \sfat_2(\MH|_{\ba(\tilde v)}) = d - \tilde t$, i.e., $t > \tilde t$. It follows that
  \begin{equation}
    \label{eq:depthv-indstep}
\height(v) \leq \height(\tilde v) + \ell_{\tilde t} \leq \sum_{t'=0}^{\tilde t - 1} \ell_{t'} + \ell_{\tilde t} \leq \sum_{t'=0}^{t - 1} \ell_{t'},
\end{equation}
as desired.
\end{proof}
Applying Claim \ref{clm:bs-depth} for $s = d-1$, we get that for each leaf $v$ of $\bx$, $\height(v) \leq \sum_{t'=0}^{t-1} \ell_{t'}$ for $t = d - \sfat_2(\MH|_{\ba(v)})$. Moreover, again fixing a leaf $v$ of $\bx$, let $s_v$ denote the minimum value of $s' \geq 0$ so that $v$ corresponds to a leaf $v'$ in $\bx\^{s'}$. For each $0 \leq s' < s_v$, there is a unique leaf $w_{s'}$ of $\bx\^{s'}$ (in fact, $w_{s'} \in \MB\^{s'}$) so that $w_{s'}$ is an ancestor of the leaf $v'$ in $\bx\^{s'}$. Also let $w_{s_v} = v$. For any given $1 \leq \tilde t < t$, choose $s' \leq s_v$ as small as possible so that $\sfat_2(\MH|_{\ba(w_{s'})}) \leq d - \tilde t$. We must have $\sfat_2(\MH|_{\ba(w_{s'-1})}) > d - \tilde t$, and so, letting $\hat t := d - \sfat_2(\MH|_{\ba(w_{s'-1})})$, (similarly to (\ref{eq:depthv-indstep})) it follows that
$$
\height(w_{s'}) \leq \height(w_{s'-1}) + \ell_{\hat t} \leq \sum_{t'=0}^{\hat t - 1} \ell_{t'} + \ell_{\hat t} \leq \sum_{t'=0}^{\tilde t - 1} \ell_{t'},
$$
which completes the verification that $\bx$ is a reducing tree.

To establish the last claim of the lemma, note that Claim \ref{clm:bs-depth} with $s = d-1$ implies that to specify a leaf $v$ of $\bx$ with $\sfat_2(\MH|_{\ba(v)}) = d-t$, we need to specify a sequence of at most $\sum_{t'=0}^{t-1} \ell_{t'}$ integers in $[\K]$ (as the tree $\bx$ is $\K$-ary). Moreover, the set of such sequences, taken over all leaves $v$ with $\sfat_2(\MH|_{\ba(v)}) = d-t$, must be prefix-free (as a leaf cannot be an ancestor of another leaf). Thus the number of leaves $v$ with $\sfat_2(\MH|_{\ba(v)}) = d-t$ is at most $\K^{\sum_{t'=0}^{t-1} \ell_{t'}}$.
\end{proof}

\subsection{Proofs for the \filterstep algorithm}
\begin{lemma}
  \label{lem:reduced-equal}
  Suppose $\MF \subset [\K]^\MX$ and $\ba \subset \MX \times [\K]$ is a subset of $\MX \times [\K]$ of size at most $\ell - 1$, for some positive integer $\ell$. Suppose $\MG, \MG' \subset \MF$ are $\ell$-irreducible and satisfy, for each $(x,y) \in \ba$, $\soa{\MG}{x} = \soa{\MG'}{x} = y$. If also $\sfat_2(\MG) = \sfat_2(\MG') = \sfat_2(\MF|_{\ba})$, then
  \begin{equation}
    \label{eq:linf-1}
\left\| \soaf{\MG'}- \soaf{\MG} \right\|_\infty \leq 1.
\end{equation}
\end{lemma}
\begin{proof}
  By Lemma \ref{lem:many-irred} applied to the classes $\MG, \MG'$, it holds that $\MG|_{\ba}$ and $\MG'|_{\ba}$ are 1-irreducible and satisfy $\sfat_2(\MG|_{\ba}) = \sfat_2(\MG'|_{\ba}) = \sfat_2(\MG) = \sfat_2(\MG') = \sfat_2(\MF|_{\ba})$. 

  If there were some $x \in \MX$ together with $k,k' \in [\K]$ so that $|k-k'| \geq 2$ so that
  $$
\sfat_2(\MG|_{\ba \cup \{ (x,k)\}}) = \sfat_2(\MG), \qquad \sfat_2(\MG'|_{\ba \cup \{(x,k')\}}) = \sfat_2(\MG'), %
$$
and since $\MG, \MG' \subset \MF$, we would have that %
$$
\sfat_2(\MF|_{\ba \cup \{ (x,k)\}}) = \sfat_2(\MF|_{\ba \cup \{(x,k')\}}) = \sfat_2(\MF|_{\ba}),
$$
which is a contradiction to Lemma \ref{lem:consec-k}.
\end{proof}

  Lemma \ref{lem:close-reps} uses Lemma \ref{lem:reduced-equal} to show that any class $\MH$ belonging to one of the sets $\irred{\ell_{r,t}}{d-t}{\MF}$ constructed in \filterstep is close in $\ell_\infty$ norm to its representative $\repl{\MH}$.
\begin{replemma}{lem:close-reps}
  Fix inputs $\MF, (\ell_{r,t})_{r,t \geq 0}, r_{\max}$ to \filterstep. For any $0 \leq r \leq r_{\max}, 0 \leq t \leq d$, and any $\MH \in \irred{\ell_{r,t}}{d-t}{\MF}$, we have that $\| \soaf{\MH} - \soaf{\repl{\MH}} \|_\infty \leq 1$.
\end{replemma}
\begin{proof}[Proof of Lemma \ref{lem:close-reps}]
  Fix some $\MH \in \irred{\ell_{r,t}}{d-t}{\MF}$ so that either $r = r_{\max}$ or $\MH \not \in \irred{\ell_{r+1,t}}{d-t}{\MF}$, and recall that $d-t = \sfat_2(\MH)$. If, in the iteration of the for loop in step \ref{it:for-mh} when the given $\MH$ is considered (which corresponds to the value $r$), the branch in step \ref{it:set-l} is taken, then we have $\| \soaf{\MH} - \soaf{\repl{\MH}} \|_\infty = 0$. The nontrivial case is that the branch in step \ref{it:found-l} is taken: in this case, choose $\ML \in \CL_{d-t}$ and $\ba \subset \MX \times [\K]$ so that $|\ba| \leq \ell_{r,t} - 1$, $\sfat_2(\MF|_{\ba}) = d-t$ and so that for all $(x,y) \in \ba$, $\soa{\ML}{x} = \soa{\MH}{x} = y$. 

  Certainly $\MH$ is $\ell_{r,t}$-irreducible. The same holds for $\ML$, since the only classes that have been added to $\CL_{d-t}$ at the time when $\MH$ is reached in step \ref{it:for-mh} must belong to $\irred{\ell_{r',t}}{d-t}{\MF}$ for some $r' \geq r$, and for all $r' \geq r$, we have $\ell_{r',t} \geq \ell_{r,t}$. We also have $\sfat_2(\ML) = \sfat_2(\MH) = d-t$ since this is the case for all elements of $\CL_{d-t}$.

By Lemma \ref{lem:reduced-equal} with $\MG = \MH, \MG' = \ML, \ell = \ell_{r,t}$, it follows that since $\ell_{r,t} -1\geq |\ba|$, we have that
$$
\| \soaf{\MH} - \soaf{\ML} \|_\infty \leq 1,
$$
as desired.
\end{proof}

\begin{replemma}{lem:at-most-one}
Fix inputs $\MF, (\ell_{r,t})_{r,t \geq 0}, r_{\max}$ to \filterstep. For any $0 \leq t \leq d$ and $0 \leq r \leq r_{\max}$, and any $\ba \subset \MX \times [\K]$ with $|\ba| \leq \ell_{r,t} - 1$ so that $\sfat_2(\MF|_{\ba}) = d-t$, there is at most one element $\ML \in \CL_{d-t} \cap \irred{\ell_{r,t}}{d-t}{\MF}$ so that for all $(x,y) \in \ba$, $\soa{\ML}{x} = y$.
\end{replemma}
\begin{proof}[Proof of Lemma \ref{lem:at-most-one}]
Suppose for the purpose of contradiction there were two distinct $\ML, \ML' \in \CL_{d-t} \cap \irred{\ell_{r,t}}{d-t}{\MF}$ so that for all $(x,y) \in \ba$, $\soa{\ML}{x} = \soa{\ML'}{x} = y$. By construction all elements of $\CL_{d-t}$ are elements of $\irred{\ell_{r,t}}{d-t}{\MF}$ for some $r$. Suppose (without loss of generality) that $\ML'$ is considered after $\ML$ in the for loop in step \ref{it:for-mh} of \filterstep. Since $|\ba| \leq \ell_{r', t}- 1$ for all $r' \geq r$, when $\ML'$ is considered in the for loop in step \ref{it:for-mh} of \filterstep, we would not add $\ML'$ to $\CL_{d-t}$ and could instead set $\repl{\ML'} \gets \ML$.
\end{proof}

\subsection{Proofs for the \soafilter algorithm}
\begin{lemma}
  \label{lem:m-irred}
  Fix $\MF \subset [\K]^\MX$, and in the context of the algorithm \soafilter, consider any $0 \leq j \leq d$ and $1 \leq s \leq d$, and set $r = r_{\max} - jr_0 - 1$. For any $\ba \in \CQ_{j,s}$, letting $t := d - \sfat_2(\MF|_\ba)$, it holds that $\MF|_\ba$ is $\ell_{r,t}$-irreducible.
\end{lemma}
\begin{proof}
Given $\ba \in \CQ_{j,s}$, let $\MH := \MF|_{\ba}$. There is some $\ba' \in \CQ_{j,s-1}$ so that, letting $\MH' := \MF|_{\ba'}$, there is some $y \in [\K]$ and leaf $v$ of the tree $\bx\^{\MH', (x_{\ba'},y)}$ so that $\ba = \ba' \cup \ba(v)$ (see step \ref{it:bav-added} of \soafilter). Let $t' := d - \sfat_2(\MH')$. Since the tree $\bx\^{\MH',(x_{\ba'},y)}$ is a reducing tree with respect to $\MH'$ for the  pair $(x_{\ba'}, y)$ and the sequence $(\ell_{r,t+t'})_{0 \leq t \leq d-t'}$, we have that $\MH'|_{\ba(v)} = \MF|_\ba$ is $\ell_{r,(\sfat_2(\MH') - \sfat_2(\MH))+t'}$-irreducible, i.e., $\ell_{r,t}$-irreducible (see Definition \ref{def:reducing-tree}).
\end{proof}

\begin{lemma}
  \label{lem:a-size}
  Fix $\MF \subset [\K]^\MX$, and in the context of the algorithm \soafilter consider any $0 \leq j \leq d$ and $1 \leq s \leq d$, and let $r = r_{\max} - jr_0 - 1$. Then the following statements hold: %
  \begin{enumerate}
  \item \label{it:leaf-bound} For any $\ba \in \CQ_{j,s}$, let $t := d - \sfat_2(\MF|_{\ba})$; then $|\ba| \leq \sum_{t'=0}^{t-1} \ell_{r,t'}$.%
  \item \label{it:node-bound} For any $\ba \in \CQ_{j,s}$, let $\MH := \MF|_\ba$, $t := d-  \sfat_2(\MH)$, and consider any of the reducing trees $\bx\^{\MH, (x_{\ba}, y)}$ constructed in step \ref{it:make-tree} of \soafilter, and any leaf $v$ of $\bx\^{\MH, (x_{\ba},y)}$. Then for any $t < \tilde t \leq d - \sfat_2(\MH|_{\ba(v)})$, there is some node $v'$ of $\bx\^{\MH, (x_{\ba},y)}$ which is an ancestor of $v$ (or is $v$ itself) and so that $\sfat_2(\MH|_{\ba(v')}) \leq d-\tilde t$ and $|\ba \cup \ba(v')| \leq \sum_{t'=0}^{\tilde t - 1} \ell_{r,t'}$. 
  \end{enumerate}
\end{lemma}
\begin{proof}
  Fix any $j$, let $r = r_{max} - jr_0 - 1$, and write $\CQ_j := \bigcup_{0 \leq s \leq d} \CQ_{j,s}$. 
  We begin with the proof of item \ref{it:leaf-bound}, which we establish via induction on $t$; the base case $t=0$ is immediate since the only element $\ba \in \CQ_j$ with $\sfat_2(\MF|_\ba) = d$ is $\ba = \emptyset$. %
  Suppose the statement of the lemma holds for all $\ba \in \CQ_j$ with $\sfat_2(\MF|_\ba) > d-t_0$, for any $t_0 \geq 0$. Now, for any $0 \leq s \leq d$, fix any $\ba \in \CQ_{j,s}$ with $\sfat_2(\MF|_\ba) = d-t_0$. By construction of $\CQ_{j,s}$, there is some $\ba' \in \CQ_{j,s-1}$, together with some $(x_{\ba'},y) \in \MX \times [\K]$, so that the following holds. Let us set $\MH:= \MF|_\ba, \ \MH' := \MF|_{\ba'}$ and $t_0' := d - \sfat_2(\MH') < t_0$; then for some leaf $v$ of the reducing tree $\bx\^{\MH',(x_{\ba'},y)}$ (which is defined with respect to the sequence $(\ell_{r,t_0' + t'})_{0 \leq t' \leq d - t_0'}$), we have that $\MH = \MH'|_{\ba(v)}$. By definition of a reducing tree, we have that
  $$
|\ba(v)| \leq \sum_{q = 0}^{(d-t_0') - (d-t_0) - 1} \ell_{r,q + t_0'} = \sum_{t' = t_0'}^{t_0-1} \ell_{r,t'}.
$$
By the inductive hypothesis, %
it holds that $|\ba'| \leq \sum_{t'=0}^{t_0' - 1} \ell_{r,t'}$. Then %
$$
| \ba' \cup \ba(v)| \leq \sum_{t'=0}^{t_0- 1} \ell_{r,t'},
$$
which establishes part \ref{it:leaf-bound}. 

Next we establish part \ref{it:node-bound}. Fix $\ba, \bx\^{\MH, (x_{\ba},y)}, v$ as in the statement of the lemma, and consider any $t < \tilde t \leq d - \sfat_2(\MH|_{\ba(v)})$. %
By the definition of a reducing tree there is some node $v'$ of $\bx\^{\MH, (x_{\ba},y)}$ which is an ancestor of $v$ so that $\sfat_2(\MH|_{\ba(v')}) \leq d - \tilde t$ and so that $\height(v') \leq \sum_{t'=0}^{\tilde t - t- 1} \ell_{r,t' + t}$. (If $\tilde t = d - \sfat_2(\MH|_{\ba(v)})$ we may just choose $v' = v$.) Using part \ref{it:leaf-bound}, we obtain that
$$
|\ba \cup \ba(v')| \leq \sum_{t'=0}^{t-1} \ell_{r,t'} + \sum_{t'=0}^{\tilde t - t - 1} \ell_{r,t'+t} = \sum_{t'=0}^{\tilde t - 1} \ell_{r,t'}.
$$
\end{proof}

Finally we are ready to establish the main strong stability result of \soafilter.
\begin{replemma}{lem:soafilter-lstar}%
  Fix any positive integer $\bar\ell$. Suppose that $\MG \subset \MF$ is nonempty, $\hat g \in [\K]^\MX$, that %
  $ \| \soaf{\MG} - \hg \|_\infty \leq \chi$ for some $\chi > 0$, and that $\MG$ is $(\bar\ell \cdot (d+3)^d)$-irreducible. %
  Then there is some $\bar\ell$-irreducible $\ML^\st \subset \MF$, depending only on $\MG$, so that $\| \soaf{\ML^\st} - \soaf{\MG} \|_\infty \leq \taumax + 1$ and so that $\ML^\st \in \repf{\hg}$, where $\repf{\hg}$ is the output of \soafilter when given as inputs $\MF$, $\hat g$, $r_{\max} = \rmax, \ \tau_{\max} = \taumax$ and the sequence $\ell_{r,t} := \bar\ell \cdot (r+2)^t$ for $0 \leq r \leq \rmax$, $0 \leq t \leq d$.

  Moreover, all $\ML \in \CR_{\hat g}$ satisfy $\| \soaf{\ML} - \hat g \|_\infty \leq \taumax$ and are $\bar \ell$-irreducible.
\end{replemma}
\begin{proof}
The final statement of the lemma follows from step \ref{it:remove-faraway-filter} of \soafilter. 
  
We proceed to prove the remainder of the lemma. For $0 \leq \tau \leq \taumax$ and $2 \leq r \leq \rmax$, define
  \begin{equation}
    \label{eq:mrtau}
\mu(r,\tau) := \max_{(\MH, \ell) \in \CG_{r,\tau}} \left\{ \sfat_2(\MH) \right\},
\end{equation}
where
\begin{equation}
  \label{eq:giant-max}
  \CG_{r,\tau} := \left\{ (\MH, \ell_{r,t}) : \parbox{10.3cm}{\centering\text{$\MH \subset \MF$ is $\ell_{r,t}$-irreducible and a finite restriction subclass of $\MF$, } \\ \text{ where $t = d - \sfat_2(\MH)$, %
      and $\| \soaf{\MH} - \soaf{\MG} \|_\infty \leq \tau$.}} \right\}.
  \end{equation}
  Since $\MG$ is $\ell_{\rmax,d}$-irreducible, and for all $t,r$ we have $\ell_{r,t} \leq \ell_{\rmax,d}$, we have that $(\MG, \ell_{r,t}) \in \CG_{r,\tau}$ for $t = d - \sfat_2(\MG)$ and all $0 \leq r \leq \rmax, 0 \leq \tau \leq \taumax$, i.e., $\CG_{r,\tau}$ is nonempty and so $\mu(r,\tau)$ is well-defined. Thus, for all $r,\tau$ in this range, it holds that for fixed $r$, $\tau \mapsto \mu(r,\tau)$ is a non-decreasing function of $\tau$, and for fixed $\tau$, $r\mapsto \mu(r,\tau)$ is a non-increasing function of $r$ (since for any $t$, $r \mapsto \ell_{r,t}$ is an increasing function). By Lemma \ref{lem:pigeonhole}, there is some $r^\st, \tau^\st$ with $r^\st = \rmax - j^\st,\ \tau^\st =  (2+ 2\chi)j^\st$ for some $0 \leq j^\st \leq d$, %
  so that $\mu(r^\st, \tau^\st) = \mu(r^\st - 1, \tau^\st + 2 + 2\chi)$. 

  Now choose some $(\MH^\st, \ell^\st)$ which achieves the maximum in (\ref{eq:mrtau}) for $r=r^\st, \ \tau = \tau^\st$; letting $t^\st = d - \sfat_2(\MH^\st)$, we have that $\ell^\st = \ell_{r^\st, t^\st}$. Let $\repl{\cdot}$ be the mapping defined as the output of \filterstep with the input class $\MF$, the sequence $(\ell_{r,t})_{0\leq r \leq r_{\max}, 0 \leq t \leq d}$, and $r_{\max} = d+1$ (these are exactly the parameters used in Step \ref{it:call-filterstep} of \soafilter). Now set $\ML^\st = \repl{\MH^\st} \in \CL_{d-t^\st}\cap \irred{\ell_{r^\st, t^\st}}{d-t^\st}{\MF}$; note that this is well-defined since $\MH^\st \in \irred{\ell_{r^\st, t^\st}}{d - t^\st}{\MF}$. %

  By definition of $\MH^\st$ we have that
\begin{equation}
\| \soaf{\MH^\st} - \soaf{\MG} \|_\infty \leq \tau^\st.\nonumber
\end{equation}
By Lemma \ref{lem:close-reps}, the fact that $\| \soaf{\MG} - \hg \|_\infty \leq \chi$ (by assumption), and the triangle inequality, it follows that
\begin{equation}
  \label{eq:taust-p2}
\| \soaf{\ML^\st} - \hg \|_\infty \leq \tau^\st + 1 + \chi.
\end{equation}
Next consider the execution of \soafilter (Algorithm \ref{alg:soafilter}) in the iteration of the for loop in line \ref{it:for-j} corresponding to $j=j^\st$ (and with input $\hg$ and $\ell_{r,t}, r_{\max}, \tau_{\max}$ as in the lemma statement; note that we have $\tau_0 = 2 + 2\chi, r_0 = 1$ in the context of \soafilter). In particular, in this iteration of the loop we have $\tau = \tau^\st + 2 + \chi, r = r^\st -1$. We define a particular sequence $\hat \ba_0\in \CQ_{j^\st,0}, \hat \ba_1 \in \CQ_{j^\st,1}, \ldots,  \hat \ba_{\hat s} \in \CQ_{j^\st,\hat s}$, for some $\hat s \leq d+1$ (to be defined below). First set $\hat \ba_0 = \emptyset$. Given the choice of $\hat \ba_s \in \CQ_{j^\st, s}$, for any $s \geq 0$, define $\hat \ba_{s+1}$ as follows: consider the iteration of the for loop over $\CQ_{j^\st,s}$ (i.e., the bullet point in step \ref{it:iterate-s}) for which $\ba = \hat \ba_s \in \CQ_{j^\st,s}$. If it holds that $\| \soaf{\MF|_{\hat \ba_s}} - \hg \|_\infty \leq \tau^\st + 2 + \chi$, meaning that the branch in step \ref{it:found-close} is taken, then set $\hat s = s$ (in which case $\hat \ba_{s+1}$ is not defined). 
Otherwise, on step \ref{it:choose-y} on the iteration of the for loop corresponding to $\ba = \hat \ba_s$, choose $y  = \soa{\ML^\st}{x_{\hat \ba_s}}$ (which is of distance at most $\tau^\st + 1 + \chi = (\tau^\st + 2 + \chi) - 1$ from $\hg({x_{\hat \ba_s}})$). Then let $v$ be the unique leaf of the reducing tree $\bx\^{\MF|_{\hat \ba_s}, (x_{\hat \ba_s},y)}$ corresponding to $\soaf{\ML^\st}$ in the sense that for all $(x', y') \in \ba(v)$, $\soa{\ML^\st}{x'} = y'$. Now set $\hat \ba_{s+1} := \hat \ba_s \cup {\ba(v)} \in \CQ_{j^\st, s+1}$ (again we use that for each such pair $(x',y')$, $|\hg({x'}) - y'| \leq (\tau^\st +2 +\chi)-1$). Notice that the definition of $\hat \ba_{s+1}$ from $\hat \ba_s$ above relies on the fact that $\MF|_{\hat \ba_s}$ is nonempty for each $s$; we will establish that this is case below, which will show that the $\hat \ba_s$ are well-defined for $0 \leq s \leq \hat s$. %
Finally, if there is no $0 \leq s \leq d$ so that $\| \soaf{\MF|_{\hat \ba_s}} - \hg \|_\infty \leq \tau^\st + 2 + \chi$, then define $\hat s = d+1$ (we will show that this will not be the case).

We claim that (a) for each $0 \leq s \leq \hat s$, all $\hat \ba_s \in \CQ_{j^\st, s}$ are well-defined, (b) $\hat s \leq d$, and (c) $\sfat_2(\MF|_{\hat \ba_{\hat s}}) = d-t^\st$. (Recall that $d-t^\st = \sfat_2(\ML^\st) = \sfat_2(\MH^\st) = \mu(r^\st, \tau^\st) = \mu(r^\st - 1, \tau^\st +4)$.) We show this in several steps:
\begin{itemize}
\item We begin by showing that for all $s \leq \min\{\hat s, d\}$, it holds that $\sfat_2(\MF|_{\hat \ba_{s}}) \geq \sfat_2(\ML^\st)$. This immediately implies that $\hat \ba_s$ is well-defined for all $0 \leq s \leq \min\{ \hat s, d \}$, since the fact that $\sfat_2( \MF|_{\hat \ba_s}) \geq 0$ implies that $\MF|_{\hat \ba_s}$ is nonempty. %
  Suppose that this is not the case; then choose $s < \hat s$ as large as possible so that $\sfat_2(\MF|_{\hat \ba_s}) \geq \sfat_2(\ML^\st)$ (in particular, $\hat \ba_s$ is well-defined and $\MF|_{\hat \ba_s}$ is nonempty). Let $y = \soa{\ML^\st}{x_{\hat\ba_s}}$. Let $v$ be the unique leaf of the tree $\bx\^{\MF|_{\hat \ba_s}, (x_{\hat \ba_s},y)}$ corresponding to $\soaf{\ML^\st}$ in the sense that for all $(x',y') \in \ba(v)$, we have $\soa{\ML^\st}{x'} = y'$. By definition of $s$ and of $\hat \ba_{s+1}$ we must have that $\sfat_2(\MF|_{\hat \ba_s \cup \ba(v)}) < \sfat_2(\ML^\st) \leq \sfat_2(\MF|_{\hat \ba_s})$. By part \ref{it:node-bound} of Lemma \ref{lem:a-size} with $\tilde t = t^\st + 1 = \sfat_2(\ML^\st) + 1$ and $\ba = \hat \ba_s$, there is some node $v'$ of the tree $\bx\^{\MF|_{\hat \ba_s}, (x_{\hat \ba_s},y)}$ which is an ancestor of $v$ and satisfies $\sfat_2(\ML^\st|_{\hat \ba\cup \ba(v')}) \leq \sfat_2(\MF|_{\hat \ba\cup\ba(v')}) < \sfat_2(\ML^\st)$ as well as $|\hat \ba_s \cup \ba(v')| \leq \sum_{t'=0}^{t^\st} \ell_{r^\st - 1, t'}$.
  Now notice that for each pair $(x', y') \in \hat \ba_s \cup \ba(v')$, we have that $\soa{\ML^\st}{x'} = y'$ by construction. 
  But since $\ML^\st$ is $\ell_{r^\st, t^\st}$-irreducible, this is a contradiction in light of Lemma \ref{lem:many-irred} and the fact that
  $$
\sum_{t'=0}^{t^\st} \ell_{r^\st-1, t'} \leq \ell_{r^\st, t^\st}
$$
for all possible $r^\st \geq 1, t^\st \geq 0$ for our choice of $\ell_{r,t} = \bar\ell \cdot (r+2)^t$.
\item Next we show that $\hat s \leq d$ (which implies that $\| \soaf{\MF|_{\hat\ba_{\hat s}}} - \hg \|_\infty \leq \tau^\st + 2 + \chi$). To do this we note that since the tree $\bx\^{\MF|_{\hat \ba_s}, (x_{\hat \ba_s}, y)}$ used to define $\hat \ba_{s+1}$ from $\hat \ba_s$ is a reducing tree for the class $\MF|_{\hat \ba_s}$, we must have that $\sfat_2(\MF|_{\hat \ba_{s+1}}) < \sfat_2(\MF|_{\hat \ba_s})$, and so for $s \leq \hat s$, $\sfat_2(\MF|_{\hat \ba_s}) \leq d-s$. If it is not the case that $\hat s \leq d$ (i.e., $\hat s = d+1$), then by the previous item for $s = d$, we have that $0 \geq \sfat_2(\MF|_{\hat \ba_d}) \geq \sfat_2(\ML^\st)$,  which implies that $\sfat_2(\MF|_{\hat \ba_d}) = \sfat_2(\ML^\st) = 0$ since $\ML^\st$ is nonempty. In particular, by Lemma \ref{lem:0dim-irred}, $\ML^\st, \MF|_{\hat \ba_d}$ are $\ell$-irreducible for all $\ell \in \BN$. By Lemma \ref{lem:reduced-equal} with $\ba = \hat \ba_d$, $\MG = \ML^\st, \MG' = \MF|_{\hat \ba_d}$, since for all $(x',y') \in \hat \ba_d$, we have $\soa{\ML^\st}{x'} = \soa{\MF|_{\hat \ba_d}}{x'} = y'$, it follows that $\| \soaf{\MF|_{\hat \ba_d}} - \soaf{\ML^\st} \|_\infty \leq 1$. Together with the triangle inequality and (\ref{eq:taust-p2}), this gives $\| \soaf{\MF|_{\hat \ba_d}} - \hat g \|_\infty \leq \tau^\st + 2 + \chi$. But this means that in step \ref{it:found-close} of \soafilter, it holds that $\| \soaf{\MF|_{\hat \ba_d}} - \hat g \|_\infty \leq \tau = \tau^\st + 2 + \chi$, and thus the branch in that step is taken, i.e., we set $\hat s = d$. This shows it cannot be the case that $\hat s = d+1$, as desired.
\item Finally we show that $\sfat_2(\MF|_{\hat \ba_{\hat s}}) \leq \mu(r^\st - 1, \tau^\st +2 + 2\chi)$. By definition of $\mu(\cdot, \cdot)$ it suffices to show that $\MF|_{\hat \ba_{\hat s}} \in \CG_{r^\st-1, \tau^\st+2 + 2\chi}$. By the definition of $\hat s$ and the fact that $\hat s \leq d$, we have that $\| \soaf{\MF|_{\hat \ba_{\hat s}}} - \hg \|_\infty \leq \tau^\st + 2 + \chi$, and thus $\| \soaf{\MF|_{\hat \ba_{\hat s}}} - \soaf{\MG} \|_\infty \leq \tau^\st + 2 + 2\chi$. By Lemma \ref{lem:m-irred}, we have that $\MF|_{\hat \ba_{\hat s}}$ is $\ell_{r^\st-1,t}$-irreducible for $t = d - \sfat_2(\MF|_{\hat \ba_{\hat s}})$. Hence $\MF|_{\hat \ba_{\hat s}} \in \CG_{r^\st-1,\tau^\st+2 + 2\chi}$, and thus $\mu(r^\st-1,\tau^\st+2 + 2\chi) \geq \sfat_2(\MF|_{\hat \ba_{\hat s}})$.
  \item From the first and second items above it follows that $\sfat_2(\MF|_{\hat \ba_{\hat s}}) \geq \sfat_2(\ML^\st)$, and the third item above shows that $\sfat_2(\MF|_{\hat \ba_{\hat s}}) \leq \mu(r^\st - 1, \tau^\st + 2 + 2\chi) = \mu(r^\st, \tau^\st) = \sfat_2(\ML^\st)$. Thus $\sfat_2(\MF|_{\hat \ba_{\hat s}}) = \sfat_2(\ML^\st) = d-t^\st$. 
\end{itemize}
Part \ref{it:leaf-bound} of Lemma \ref{lem:a-size} gives that $|\hat \ba_{\hat s}| \leq \sum_{t'=0}^{t^\st-1} \ell_{r^\st-1,t'} < \ell_{r^\st-1,t^\st} < \ell_{r^\st, t^\st}$. By Lemma \ref{lem:at-most-one}, there is at most one choice of $\ML \in \CL_{d- t^\st}$ so that $\ML$ is $\ell_{r^\st,t^\st}$-irreducible and for each $(x,y) \in \hat \ba_{\hat s}$, $\soa{\ML}{x} = y$. Notice that $\ML^\st$ is one such choice of $\ML$. Thus $\ML^\st$ must be added to $\repf{\hg}$ in step \ref{it:found-close} of \soafilter when $\hat \ba_{\hat s}$ is considered in the for loop.

By Lemma \ref{lem:reduced-equal} with $\ba = \hat \ba_{\hat s}$, $\MG = \ML^\st, \MG' = \MF|_{\hat \ba_{\hat s}}$, since $|\hat \ba_{\hat s}| < \ell_{r^\st - 1, t^\st} $, $\sfat_2(\ML^\st) = \sfat_2(\MF|_{\hat \ba_{\hat s}})$, $\ML^\st$ and $\MF|_{\hat \ba_{\hat s}}$ are both $\ell_{r^\st - 1, t^\st}$-irreducible, and for all $(x,y) \in \hat \ba_{\hat s}$, $\soa{\ML^\st}{x} = \soa{\MF|_{\hat \ba_{\hat s}}}{x} = y$, we have that $\| \soaf{\ML^\st} - \soaf{\MF|_{\hat \ba_{\hat s}}} \|_\infty \leq 1$. Together with $ \| \soaf{\MF|_{\hat \ba_{\hat s}}} - \soaf{\MG} \|_\infty \leq \tau^\st + 2 + 2\chi$ and $\tau^\st \leq (2+ 2\chi)d$, we get that $\| \soaf{\ML^\st} - \soaf{\MG} \|_\infty \leq \taumax + 1$. Moreover, since $\| \soaf{\MF|_{\hat \ba_{\hat s}}} - \hat g \|_\infty \leq \tau^\st + 2 + \chi \leq \taumax - 1$, we have that $\soaf{\ML^\st}$ is not eliminated from $\CR_{\hat g}$ in step \ref{it:remove-faraway-filter} of \soafilter.
\end{proof}

\begin{lemma}
  \label{lem:rep-size}
  In the algorithm \soafilter, we have the following upper bound on the size of the output set $\repf{\hg}$:
  $$
|\repf{\hg}| \leq \sum_{r=0}^{r_{\max}} \K^{\sum_{t'=0}^{d-1} \ell_{r,t'}}.
$$
In particular, for the choice $r_{\max} = \rmax$ and $\ell_{r,t} = \bar\ell \cdot (r+2)^t$ (for any $\bar\ell \in \BN$), we get
$$
|\repf{\hg}| \leq \K^{\bar\ell \cdot (d+4)^{d}}.
$$
\end{lemma}
\begin{proof}
  Fix any $0 \leq j \leq d$ considered in the for loop on step \ref{it:for-j} of \soafilter. Let $\tau = j\tau_0 + 3, r = r_{\max} - jr_0 - 1$. For accounting purposes, we define the following tree $T$ whose non-leaf nodes are labeled by elements of $\MX$ (the tree $T$ does {\it not} satisfy the requirements of Definitions \ref{def:xv-tree} or \ref{def:augmented}). The root of the tree $T$ is labeled by the point $x_{\emptyset}$ defined in step \ref{it:choose-xs} of \soafilter corresponding to $\emptyset \in \CQ_{j,0}$ (in the event that this step is never reached, then at most a single element is added to $\repf{\hg}$ in \soafilter for the value of $j$ under consideration). We will call some of the nodes of $T$ {\it special}; the root is special. Each special node of $T$ is labeled by some $x_\ba$ corresponding to the execution of step \ref{it:choose-xs} in \soafilter for some $0 \leq s \leq d$ and $\ba \in \CQ_{j,s}$. For each special node $u$ we define its descendents inductively as follows. %
  The (immediate) children of $u$ in $T$ are defined as follows: $u$ has at most $2\tau - 1 \wedge \K$ children, corresponding to each of the elements $y$ of $\{ k -\tau+1 \vee 0, \ldots, k + \tau - 1 \wedge \K \}$, where $k = \hg(x_\ba)$. Each such child corresponding to some such $y$ is labeled by the unique child of the root of the reducing tree $\bx\^{\MF|_\ba, (x_\ba, y)}$. Then we append the reducing tree $\bx\^{\MF|_\ba , (x_\ba, y)}$ (except its root) to $T$ via this child. The leaves of a reducing tree are not labeled by elements of $\MX$, but we label some leaves $v$ of $\bx\^{\MF|_\ba, (x_\ba, y)}$ as follows. For any leaf $v$ of $\bx\^{\MF|_\ba, (x_\ba, y)}$, if $\ba \cup \ba(v)$ is not added to $\CQ_{j,s+1}$ in step \ref{it:bav-added}, then $v$ (viewed as a node of $T$) is defined to be a leaf of $T$, in which case we do not assign it a label. Otherwise, we have that $\ba':= \ba \cup \ba(v) \in \CQ_{j,s+1}$; if either of the branches in steps \ref{it:check-empty} or \ref{it:found-close} are taken when $\ba'$ is considered in the for loop (for the value $s+1$), then $v$ has no children in the tree $T$ (i.e., is a leaf of $T$) and again is assigned no label. Otherwise, $v$ is labeled by the element $x_{\ba'}$ defined in step \ref{it:choose-xs}, in which case we say that $v$ is special and we repeat the process described above with $\ba'$ replacing $\ba$. Notice that the construction of $T$ maintains the following property: for each special node $v$ of $T$ which is labeled by $x_{\ba}$, the ancestor set of $v$ in $T$ is exactly $\ba$. 

By construction of $T$, each element added to $\repf{\hg}$ in step \ref{it:found-close} of \soafilter for the value of $j$ under consideration corresponds to a distinct leaf of the tree $T$, whose ancestor set is given by some $\ba \in \CQ_{j,s}$ for some $0 \leq s \leq d$. So it suffices to bound the number of such leaves of $T$. %
Note that each node of $T$ has at most $\K$ children; indeed, the special nodes of $T$ have at most $2\tau - 1 \wedge \K \leq \K$ children, and the remaining nodes are identified with nodes of various $\K$-ary reducing trees. 
Moreover, the depth (i.e., distance to the root) of any leaf of $T$ whose ancestor set is given by some $\ba \in \CQ_{j,s}$ for some $s \leq d$ is at most $|\ba| \leq \sum_{t'=0}^{d-1} \ell_{r,t'}$, by part \ref{it:leaf-bound} of Lemma \ref{lem:a-size}. Thus the number of leaves of $T$ is at most $\K^{\sum_{t'=0}^{d-1} \ell_{r,t'}}$. Hence
  $$
|\repf{\hg}| \leq \sum_{r=0}^{r_{\max}} \K^{\sum_{t'=0}^{d-1} \ell_{r,t'}},
$$
and for the choice $\ell_{r,t} = \bar\ell \cdot (r+2)^t$ and $r_{\max} = \rmax$, this number is at most
$$
\sum_{r=0}^{d+1} \K^{\bar\ell \cdot (r+2)^{d}} \leq \K^{\bar\ell \cdot (d+4)^{d}}.
$$
\end{proof}

\begin{lemma}
  \label{lem:pigeonhole}
  Fix positive integers $A,B,d$, and let $\mu : \{ 0, 1, \ldots, A(d+1)\} \times \{ 0, 1, \ldots, B(d+1) \} \ra \BZ$ be a function so that $0 \leq \mu(a,b) \leq d$ for all $a,b$ and so that for each fixed $b$, $a \mapsto \mu(a,b)$ is non-decreasing and for each fixed $a$, $b \mapsto \mu(a,b)$ is non-decreasing. Then there is some $0 \leq i \leq d$ so that the pair $(a,b) := (Ai, Bi)$ satisfies $\mu(a,b) = \mu(a+A,b+B)$.
\end{lemma}
\begin{proof}
Consider the $d+1$ pairs $(0,0), (A,B), (2A,2B), \ldots, (A(d+1), B(d+1))$. If for each $0 \leq i \leq d+1$, $\mu(Ai,Bi) \neq \mu(A(i+1), B(i+1))$, then we have $0 \leq \mu(0,0) < \mu(A,B) < \cdots < \mu(A(d+1),B(d+1)) \leq d$, which is impossible since $\mu(a,b)$ is an integer for all $a,b$ in the domain of $\mu$.
\end{proof}

\section{\RegLearn: Private learning algorithm for regression}
\label{sec:reglearn}
In this section we combine the procedures described in the previous sections to produce an algorithm for privately learning a real-valued hypothesis class. At a high level, our algorithm \RegLearn (Algorithm \ref{alg:reglearn}) proceeds as follows: given a class $\MH \subset [-1,1]^\MX$ and samples from a distribution $Q$ supported on $\MX \times [-1,1]$, it first  discretizes $\MH$ as described in Section \ref{sec:pac}: to avoid confusion with notation in other sections, we denote the discretization parameter as $\bar \eta > 0$. In particular, we set $\MF := \disc{\MH}{\bar\eta} \subset [\K]^\MX$ (with $\K = \lceil 2/\bar\eta \rceil$) and $P := \disc{Q}{\bar\eta}$, so that $P$ is a distribution over $\MX \times [\K]$. We then use Algorithm \ref{alg:reduce-tree} applied to the class $\MF$ to learn a hypothesis $\hat g \in [\K]^\MX$ with low population error with respect to $\disc{Q}{\bar\eta}$ and which satisfies the ``weak stability'' guarantee of Lemma \ref{lem:approx-stability}. Using Algorithm \ref{alg:soafilter} we then produce a set of hypotheses $\repf{\hat g}$, satisfying the ``strong stability'' guarantee of Lemma \ref{lem:soafilter-lstar}. Repeating this procedure sufficiently many times using independent datasets drawn from the distribution $Q$ and using the sparse selection procedure of Proposition \ref{prop:sparse-selection}, we may finally produce a regressor in $[-1,1]^\MX$ which is differentially private.

\begin{algorithm}[!htp]
  \caption{\bf \RegLearn}\label{alg:reglearn}
  \KwIn{Parameters $\ep, \delta, \bar \eta, \beta \in (0,1)$, irreducibility parameter $\bar \ell \in \BN$, i.i.d. samples $(x,y) \in \MX \times [-1,1]$ from a distribution $Q$, hypothesis class $\MH\ \subset [-1,1]^\MX$.}
  \arxiv{  \begin{enumerate}[leftmargin=14pt,rightmargin=20pt,itemsep=1pt,topsep=1.5pt]}
    \colt{ \begin{enumerate}[leftmargin=20pt,rightmargin=20pt,itemsep=1pt,topsep=1.5pt]}
  \item Set $\MF := \disc{\MH}{\bar \eta}$, and write $\K := \lceil 2/\bar \eta \rceil$, so that $\MF \subset [\K]^\MX$.

    \label{it:define-params} Set $m \gets \frac{C\bar \ell (2\sfat_2(\MF)+6)^{\sfat_2(\MF)+4} \log^2 \left(\frac{1}{\ep \delta \beta \bar \eta}\right)}{\ep \bar \eta^2}$, $n_0 \gets C_0 \cdot \frac{\fat_{c_0 \bar \eta}(\MH) \log(1/\bar \eta) + \log(4m/\beta)}{\bar \eta^2}$, $n \gets n_0m$, $\alpha_\Delta \gets 18$, where $C_0, c_0$ are the constants of Corollary \ref{cor:fat-uc-disc}, and $C > 0$ is a sufficiently large constant.

    Also set $\ell' \gets \max \left\{ \bar \ell \cdot (d+3)^d, C_0 \K^2(d \log \K + 1) \right\} $, where $C_0$ is the constant of Corollary \ref{cor:2fat-uc-disc}.
  
  \item \label{it:define-alpha1} %
Let $n_1 = \frac{C_0 \cdot \fat_{c_0 \bar \eta}(\MH) \log(1/\bar \eta) + \log(8/\beta)}{\ep \bar \eta^2}$, where $C_0, c_0$ are the constants of Corollary \ref{cor:fat-uc-disc}. Set $T_{n_1} \sim Q^{n_1}$ to be an independent sample from the distribution $Q$ of size $n_1$. Set
    $$
\hat \eta := \inf_{f \in \MF} \left\{ \err{\disc{\hat Q_{T_{n_1}}}{\bar \eta}}{f} \right\} + \Lap\left( \frac{2\K}{\ep n_1} \right).
$$
to be the sum of the smallest achievable empirical error on $T_{n_1}$ and a Laplace random variable with scale $2\K/(\ep n_1)$. {\it ($\hat \eta$ is a private estimate of the optimal error achievable by a classifier in $\MF$, which is neededd to apply \ReduceTreeReg.)}

Then set $\alpha_1 :=  \hat \eta  + \alpha_\Delta / 2 + d \cdot \alpha_\Delta$. 
  \item For $1 \leq j \leq m$:\label{it:do-reducetree}
    \begin{enumerate}
      \item Let $S_{n_0} \sim Q^{n_0}$ be an independent sample from the distribution $Q$.
      \item Run the algorithm \ReduceTreeReg with the class $\MF$, distribution $\disc{\hat Q_{S_{n_0}}}{\bar \eta}$, $n = n_0$ and the parameters $\alpha_1, \alpha_\Delta, \ell'$ defined in steps \ref{it:define-params} and \ref{it:define-alpha1}.

        Let its output set $\hat \MS$  (defined in (\ref{eq:redtree-output})) be denoted by $\hat \CS\^j$. 
      \end{enumerate}
    \item For $1 \leq j \leq m$:
      \begin{enumerate}
      \item Set $\CR\^j \gets \emptyset$.  {\it ($\CR\^j$ will hold hypotheses of the form $g : \MX \ra [\K]$.)}
    \item For each hypothesis $\hat g \in \hat \CS\^j$, apply the algorithm $\soafilter$ to the hypothesis $\hat g : \MX \ra [\K]$, with the other inputs as follows:
      the hypothesis class is $\MF$, the sequence $\ell_{r,t}$ is given by $\bar \ell \cdot (r+2)^t$, parameters $\tau_{\max} = 12 \cdot (\sfat_2(\MF) + 1),\ r_{\max} = \sfat_2(\MF) + 1$.

    \item  Denote the output set of \soafilter by $\repf{\hat g}$; for each $\ML \in \repf{\hat g}$, add  $\soaf{\ML}$ to the set $\hat \CR\^j$.\label{it:soafilter-rg-output}
    \end{enumerate}
  \item \label{it:define-hath} Run the $(\ep, \delta)$-differentially private $(m, s)$-sparse selection protocol of Proposition \ref{prop:sparse-selection} with sparsity $s = \K^{C \bar \ell (2\cdot \sfat_2(\MF) + 6)^{\sfat_2(\MF) + 2} \K^2 \cdot \sfat_2(\MF) \log \K}$ on the sets $\hat \CR\^1, \ldots, \hat \CR\^m$; the universe $\MU$ for the sparse selection protocol is equal to the set of all $\soaf{\ML}$, for $\ML \subset \MF$ irreducible. Denote its output by $\soaf{\hat \ML} : \MX \ra [\K]$, for some $\hat \ML \subset \MF$. Output the class $\hat \ML$, as well as the function $\hat h : \MX \ra [-1,1]$, defined by
    $$
\hat h(x) := -1 + \frac{2}{\K} \cdot (\soa{\hat \ML}{x} - 1).
    $$
\end{enumerate}
\end{algorithm}

The below theorem states the main guarantee for the algorithm \RegLearn:
\begin{theorem}
  \label{thm:reglearn}
There are constants $c_0 \leq 1, C \geq 1, C_1 \geq 1$ so that the following holds.\footnote{In particular, $c_0$ is the corresponding constant of Corollary \ref{cor:fat-uc-disc} and $C_1$ is the corresponding constant of Corollary \ref{cor:2fat-uc-disc}.} Suppose we are given $\MH \subset [-1,1]^\MX$, as well as $\ep, \delta, \bar\eta, \beta \in (0,1)$ and $\bar \ell \in \BN$. For
  $$
n = C \cdot \frac{\bar \ell \cdot \fat_{c_0 \bar \eta}(\MH) \cdot (2\cdot \sfat_{\bar \eta}(\MH)+6)^{\sfat_{\bar \eta}(\MH)+5} \log^3 \left( \frac{\sfat_{\bar \eta}(\MH) \cdot \bar \ell}{\ep \delta \beta \bar \eta}\right)}{\ep \bar \eta^4},
$$
if the algorithm \RegLearn (Algorithm \ref{alg:reglearn}) takes as input $n$ i.i.d. samples $(x_1, y_1), \ldots, (x_n, y_n)$ from any distribution $Q$ on $\MX \times [-1,1]$, then it is $(\ep, \delta)$-differentially private and its output hypothesis $\hat h$ satisfies
$$
\Pr_{(x_1, y_1), \ldots, (x_n, y_n)} \left[ \err{Q}{\hat h} \leq \inf_{h \in \MH} \left\{ \err{Q}{h} \right\} + 30 (\sfat_{\bar \eta}(\MH) +2) \cdot \bar \eta + 2C_1\bar \eta\right] \geq 1-\beta.
$$
Moreover, under the same $(1-\beta)$-probability event, the class $\hat \ML \subset [\lceil 2/\bar\eta\rceil]^{\MX}$ output by \RegLearn is $\bar \ell$-irreducible.
\end{theorem}
\begin{proof}
  In the proof we will often refer to the values $n_0, m, \bar \eta, \alpha_\Delta, \alpha_1, \K, \MF$ which are set in steps \ref{it:define-params} through \ref{it:define-alpha1} of \RegLearn. Throughout the proof we will write $d := \sfat_2(\MF) \leq \sfat_{\bar \eta}(\MH)$ (Lemma \ref{lem:fat-disc}). Since our choice of $n_0$ satisfies
  $$
n_0 \geq C_0 \cdot \frac{\fat_{c_0 \bar \eta}(\MH) \log(1/\bar \eta) + \log(4m / \beta)}{\bar \eta^2},
$$
where $c_0, C_0$ are the constants of Corollary \ref{cor:fat-uc-disc}, then by Corollary \ref{cor:fat-uc-disc} and the choice of $\alpha_\Delta = 18$, we have that
\begin{align*}
\Pr_{S_{n_0} \sim Q^{n_0}} \left[
  \substack{\text{$E_{\good}$ holds for the dataset $S_{n_0}$} \\ \text{and the distribution $\disc{Q}{\bar \eta}$}}
  \right] &= \Pr_{S_{n_0} \sim Q^{n_0}} \left[ \sup_{f \in \MF} \left| \err{\disc{Q}{\bar \eta}}{f} - \err{\disc{\hat Q_{S_{n_0}}}{\bar \eta}}{f} \right| \leq \alpha_\Delta/6 \right] \geq 1-\frac{\beta}{4m}.
\end{align*}
(Recall the definition of $E_{\good}$ in (\ref{eq:def-egood}).)

For $1 \leq j \leq m$, let $S_{n_0}\^j := \{(x_1\^j, y_1\^j), \ldots, (x_{n_0}\^j, y_{n_0}\^j) \}$ be the dataset of size $n_0$ drawn i.i.d.~from $Q$ in the $j$th iteration of step \ref{it:do-reducetree} of \RegLearn. For convenience of notation let $\hat Q\^j := \hat Q_{S_{n_0}\^j} = \frac{1}{n_0} \sum_{i=1}^{n_0} \delta_{(x_i\^j, y_i\^j)}$ denote the empirical measure over $S_{n_0}\^j$. Then by the union bound the probability that $E_{\good}$ holds for each of the datasets $S_{n_0}\^j$ is at least $1-\beta/4$, i.e.,
\begin{equation}
  \label{eq:allj-unif-conv}
\Pr \left[ \forall j \in [m] :  \sup_{f \in \MF} \left| \err{\disc{Q}{\bar \eta}}{f} - \err{\disc{\hat Q\^j}{\bar \eta}}{f} \right| \leq \alpha_\Delta/6 \right] \geq 1-\frac{\beta}{4}.
\end{equation}
Let $E_0$ be the event inside the probability above, namely that $E_{\good}$ holds for each $S_{n_0}\^j$.

The bulk of the proof of Theorem \ref{thm:reglearn} is to show the following claims:

The first, Claim \ref{clm:alpha}, shows that $\alpha_1$ in step \ref{it:define-alpha1} in of \RegLearn is differentially private and is with high probability an upper bound on the optimal error with respect to the true distribution $Q$:
\begin{claim}[Privacy and accuracy of $\alpha_1$]
  \label{clm:alpha}
  The value $\alpha_1$ produced in step \ref{it:define-alpha1} of \RegLearn is $(\ep, 0)$-differentially private as a function of the dataset $T_{n_1}$ (and thus the entire dataset of $n$ samples used by \RegLearn). Moreover, $\alpha_1$, satisfies the following:
  \begin{equation}
    \label{eq:alpha1-accuracy}
\Pr \left[ \inf_{f \in \MF} \left\{ \err{\disc{Q}{\bar \eta}}{f} \right\} + \alpha_\Delta \geq \alpha_1 - d \cdot \alpha_\Delta \geq \inf_{f \in \MF} \left\{ \err{\disc{Q}{\bar \eta}}{f} \right\}+ \alpha_\Delta / 6 \right] \geq 1 - \beta/4.
  \end{equation}
\end{claim}
\begin{claim}
  \label{clm:reduce-and-filter}
  There is an event $E_1$ that occurs with probability at least $1-\beta/2$ over the randomness of the dataset and the algorithm, so that under $E_0 \cap E_1$, \RegLearn outputs a class $\hat\ML \subset \MF$ which is $\bar \ell$-irreducible and satisfies $\hat \ML \in \CR\^j$ for some $1 \leq j \leq m$.
\end{claim}

\begin{claim}
  \label{clm:soas-good-error}
Let $C_0, C_1$ be the constants of Corollary \ref{cor:2fat-uc-disc}. Suppose $\ell' \geq C_0 \K^2 (d \log(\K) + 1)$. Under the event $E_1 \cap E_0$, the output $\hat h$ of \RegLearn satisfies
\begin{equation}
  \label{eq:hath-good-error}
\err{Q}{\hat h} \leq \inf_{h \in \MH} \left\{ \err{Q}{h} \right\} + 30 (d+2) \bar \eta + 2C_1\bar \eta.
  \end{equation}
\end{claim}

Assuming Claims \ref{clm:alpha}, \ref{clm:reduce-and-filter} and \ref{clm:soas-good-error}, we complete the proof of Theorem \ref{thm:reglearn}. By Claim \ref{clm:soas-good-error}, under the event $E_0 \cap E_1$ (which holds with probability at least $1-\beta$), we have that the output hypothesis $\hat h : \MX \ra [-1,1]$ of \RegLearn satisfies (\ref{eq:hath-good-error}). Moreover, by Claim \ref{clm:reduce-and-filter}, under $E_0 \cap E_1$, the class $\hat \ML$ output by \RegLearn is $\bar \ell$-irreducible.

Next we argue that the outputs $(\hat \ML, \hat h)$ of \RegLearn are $(\ep, \delta)$-differentially private as a function of its input dataset (which consists of the disjoint union of the datasets $T_{n_1}, S_{n_0}\^1 ,\ldots, S_{n_0}\^m$, which we denote as $R$). Let us consider two neighboring datasets $R, R'$. If they differ in a sample corresponding to $T_{n_1}$, then we have that for any event $E$, $\Pr_{R}[(\hat \ML, \hat h) \in E] \leq e^{\ep} \cdot \Pr_{R'} [(\hat \ML, \hat h) \in E]$ by the $(\ep, 0)$-differential privacy of $\alpha_1$ (Claim \ref{clm:alpha}) and the post-processing lemma for differential privacy \cite[Proposition 2.1]{dwork_algorithmic_2013} (since for fixed $S_{n_0}\^1, \ldots, S_{n_0}\^m$, $(\hat \ML, \hat h)$ are randomized functions of $\alpha_1$). Otherwise, $R, R'$ differ in a sample corresponding to one of $S_{n_0}\^1, \ldots, S_{n_0}\^m$. Then the $(\ep, \delta)$-differential privacy guarantee of the sparse selection protocol of Proposition \ref{prop:sparse-selection} guarantees that for any fixed $\alpha_1$, for any event $E$, $\Pr_{R}[(\hat \ML, \hat h) \in E] \leq e^{\ep} \cdot \Pr_{R'} [(\hat \ML, \hat h) \in E] + \delta$. This establishes that $(\hat \ML, \hat h)$ are differentially private as a function of $R$.

Summarizing, letting $d = \sfat_2(\MF) \leq \sfat_{\bar \eta}(\MH)$ and $d' := \fat_{c_0 \bar \eta}(\MH)$ (where $c_0$ is the constant of Corollary \ref{cor:fat-uc-disc}), the sample complexity of \RegLearn is
\begin{align*}
  n_0 \cdot m + n_1 & \leq C \cdot \frac{\bar \ell (2d+6)^{d+4} \log^2 \left( \frac{1}{\ep \delta \beta \bar \eta}\right) \cdot \left( d' \log(1/\bar \eta) + \log(m/\beta)\right)}{\ep \bar \eta^4} \\
                    & \leq C' \cdot \frac{\bar \ell  d' (2d+6)^{d+5} \log^3 \left( \frac{d \bar \ell}{\ep \delta \beta \bar \eta}\right)}{\ep \bar \eta^4},
\end{align*}
where $C, C'$ are sufficiently large constants.

It only remains to prove Claims \ref{clm:alpha}, \ref{clm:reduce-and-filter}, and \ref{clm:soas-good-error}, which we do so below.
\begin{proof}[Proof of Claim \ref{clm:alpha}]
Let $C_0 \geq 1, c_0 \leq 1$ be the constants of  Corollary \ref{cor:fat-uc-disc}; then by Corollary \ref{cor:fat-uc-disc}, as long as %
  \begin{equation}
    \label{eq:tn0-lb}
n_1 \geq C_0 \cdot \frac{\fat_{c_0 \bar \eta}(\MH) \log(1/\bar \eta) + \log(8/\beta)}{\bar \eta^2},
\end{equation}
we have
$$
\Pr_{T_{n_1} \sim Q^{n_1}} \left[ \sup_{f \in \MF} \left| \err{\disc{Q}{\bar \eta}}{f} - \err{\disc{\hat Q_{T_{n_1}}}{\bar \eta}}{f} \right| > \frac{\alpha_\Delta}{6} \right] \leq \beta/8. 
$$
Let $Y$ denote the random variable drawn according to $\Lap(2\K/(\ep n_1))$ in step \ref{it:define-alpha1} of \RegLearn. Then $\Pr[|Y| > 2\K t / (\ep n_1)] = \exp(-t)$ for all $t > 0$, and in particular, as long as
\begin{equation}
  \label{eq:laplace-nlb}
  n_1 \geq C_1 \cdot \frac{\log(1/\beta)}{\ep \bar \eta}
  \end{equation}
for a sufficiently large constant $C_1$, it holds that $\Pr\left[ |Y| > \frac{\alpha_\Delta}{6} \right] \leq \beta/8$. 
  
Under the event that both $\sup_{f \in \MF} \left| \err{\disc{Q}{\bar \eta}}{f} - \err{\disc{\hat Q_{T_{n_1}}}{\bar \eta}}{f} \right| \leq \alpha_\Delta / 6$ and $|Y| \leq \alpha_\Delta / 6$, which holds with probability at least $1-\beta/4$, we get that
$$
\inf_{f \in \MF} \left\{ \err{\disc{Q}{\bar \eta}}{f} \right\} + \frac{5\alpha_\Delta}{6}  \geq \inf_{f \in \MF} \left\{\err{\disc{\hat Q_{T_{n_1}}}{\bar \eta}}{f}\right\} + Y + \frac{\alpha_\Delta}{2} \geq \inf_{f \in \MF} \left\{ \err{\disc{Q}{\bar \eta}}{f} \right\} + \frac{\alpha_\Delta}{6} .
$$
Note that the choice of $n_1$ in step \ref{it:define-params} ensures that both (\ref{eq:tn0-lb}) and (\ref{eq:laplace-nlb}) hold (as long as the constant $C$ is sufficiently large). Recalling that $\hat \eta = \inf_{f \in \MF} \left\{\err{\disc{\hat Q_{T_{n_1}}}{\bar \eta}}{f}\right\} + Y$ and $\alpha_1 - d \cdot \alpha_\Delta = \hat \eta + \alpha_\Delta/2$, we get that (\ref{eq:alpha1-accuracy}) holds.

To see the differential privacy of $\alpha_1$, note that the function that maps $T_{n_1} = \{ (x_1, y_1), \ldots, (x_{n_1}, y_{n_1}) \}$ to $\inf_{f \in \MF}\left\{\err{\disc{\hat Q_{T_{n_1}}}{\bar \eta}}{f}\right\}$ has sensitivity at most $\K / n_1$, since $|f(x) - y| \leq \K$ for each $(x,y) \in \MX \times [\K]$ and $f \in \MF$. Since $Y \sim \Lap((\K / n_1) \cdot (2 / \ep))$, we get that $\alpha_1$ is $(\ep/2, 0)$-differentially private as a function of the dataset $T_{n_1}$. 
\end{proof}

\begin{proof}[Proof of Claim \ref{clm:reduce-and-filter}]
  Recall that $\MF = \disc{\MH}{\bar \eta}$ and $P = \disc{Q}{\bar \eta}$, as well as $d = \sfat_2(\MF) \leq \sfat_{\bar \eta}(\MH)$ (Lemma \ref{lem:fat-disc}). For $\alpha > 0$, $t \in [d+1]$, recall the definition of $\MM_{\alpha, t}$ in (\ref{eq:define-m}) (defined with respect to $\MF$ and $P$), and for those $\alpha, t$ for which $\MM_{\alpha, t}$ is nonempty, the definition of $\sigma_{\alpha, t}^\st$ in (\ref{eq:define-sigma-at}). By definition of $\hat \CS\^j$ (see (\ref{eq:redtree-output}) and step \ref{it:do-reducetree} of \RegLearn) and Lemma \ref{lem:approx-stability}, as long as $\MF_{\disc{\hat Q\^j}{\bar \eta}, \alpha_{d+1}}$ is nonempty, then under the event $E_0$, each $\CS\^j$ contains at least one hypothesis of the form $\soaf{\hat \MG}$, where $\| \sigma_{\alpha_t - \alpha_\Delta/2, t}^\st - \soaf{\hat \MG} \|_\infty \leq 5$. By the pigeonhole principle, some $t$ satisfies this property for at least $\lceil m / (d+1) \rceil$ sets $\CS\^j$; let us denote this $t$ by $t^\st$. We must also verify that $\MF_{\lceil \hat Q\^j \rceil_{\bar \eta}, \alpha_{d+1}}$ is nonempty; to do so, let $E_{1,0}$ be the event that
  \begin{equation}
    \label{eq:event-e10}
  \inf_{f \in \MF} \left\{ \err{\disc{Q}{\bar \eta}}{f} \right\} + \alpha_\Delta \geq \alpha_1 - d \cdot \alpha_\Delta \geq \inf_{f \in \MF} \left\{ \err{\disc{Q}{\bar \eta}}{f} \right\}+ \alpha_\Delta / 6.
  \end{equation}
  By Claim \ref{clm:alpha}, the probability that $E_{1,0}$ holds (over the choices of the algorithm \RegLearn) is at least $1 - \beta/4$. Then noting that $\alpha_{d+1} = \alpha_1 - d \cdot \alpha_\Delta$ and using (\ref{eq:allj-unif-conv}), we get that $\MF_{\disc{Q\^j}{\bar \eta}, \alpha_{d+1}}$ is nonempty under the event $E_0 \cap E_{1,0}$. 

  Since $\ell_t \geq \ell'$ for all $t \geq 1$ (step \ref{it:define-kt} of \ReduceTreeReg), it holds from (\ref{eq:define-sigma-at}) and (\ref{eq:define-m}) that $\sigma_{\alpha_{t^\st} - \alpha_\Delta/2, t^\st}^\st$ is of the form $\soaf{\MG}$ for some $\MG \subset \MF$ which is $\ell_{t^\st}$-irreducible, and thus $\ell'$-irreducible. By Lemma \ref{lem:soafilter-lstar} with $\chi = 5$, as long as $\ell' \geq \bar\ell \cdot (d+3)^d$, there is some $\ML^\st \subset \MF$ which is $\bar\ell$-irreducible, depending only on $\MG$, so that
  for any $\hat g$ satisfying $\| \hat g - \sigma_{\alpha_{t^\st} - \alpha_\Delta/2, t^\st}^\st \|_\infty \leq 5$, $\ML^\st \in \CR_{\hat g}$, where $\CR_{\hat g}$ is as in step \ref{it:soafilter-rg-output} of \RegLearn. Thus, among the sets $\CR\^1, \ldots, \CR\^m$, there are at least $\lceil m / (d+1) \rceil$ of them containing $\ML^\st$.

  By Lemma \ref{lem:s-size}, we have that for each $1 \leq j \leq m$, $| \hat \MS\^j | \leq \K^{\ell' \cdot 2^{d+1}}$. By Lemma \ref{lem:rep-size}, each element $\hat g \in \CS\^j$ gives rise to $| \CR_{\hat g} | \leq \K^{\bar\ell \cdot (d+4)^d}$ elements of $\CR_{\hat g}$, all of which are added to $\hat \CR\^j$. Thus, recalling the definition of $\ell'$ in step \ref{it:define-params} of \RegLearn, we have that
  $$
  | \hat \CR\^j| \leq \K^{\ell' \cdot 2^{d+1} + \bar\ell \cdot (d+4)^d} \leq \K^{\bar \ell \cdot (2d+6)^{d+2} + C \K^2 d \log \K} \leq \K^{C \bar \ell (2d+6)^{d+2} \K^2 d \log \K},
  $$
  where $C > 0$ is a sufficiently large constant.

Now choose $\nu > 0$ so that the $(m,  \K^{C \bar \ell (2d+6)^{d+2} \K^2 d \log \K})$-sparse selection protocol of Proposition \ref{prop:sparse-selection} (with universe $\MU$ given by the family of all $\soaf{\MG}$, where $\MG \subset \MF$ is a finite restriction subclass of $\MF$; this family must include all elements of $\CR\^j$, $1 \leq j \leq m$), has error at most $\nu$ on some event $E_{1,1}$ with probability at least $1 - \beta/4$. By \cite[Lemma 36]{ghazi_differentially_2020}, we may choose $\nu = \frac{C}{\ep} \log \left( \frac{m \K^{C \bar \ell (2d+6)^{d+2} \K^2 d \log \K}}{\ep \delta \beta} \right)$ for a sufficiently large constant $C$.

Now set $E_1 = E_{1,0} \cap E_{1,1}$. Then under the event $E_0 \cap E_1$, as long as $\nu < \lceil m / (d+1) \rceil$, the hypothesis $\hat \ML$ output by the sparse selection protocol belongs to $\hat \CR\^j$ for some $1 \leq j \leq m$. That $\hat \ML$ is $\bar \ell$-irreducible follows from the fact $\CR\^j$ is the union of output sets $\CR_{\hat g}$ of $\soafilter$, for various functions $g : \MX \ra [\K]$, and $\CR_{\hat g}$ consists of $\bar\ell$-irreducible classes (Lemma \ref{lem:soafilter-lstar}).

To ensure $\nu < \lceil m / (d+1) \rceil$, it suffices to have, for $C'$ a sufficiently large constant,
$$
m > \frac{C'd}{\ep} \cdot \left( \log(m) + \log\left( \frac{1}{\ep \delta \beta} \right)+ \bar \ell (2d+6)^{d+3} \K^2 \log^2 \K \right),
$$
for which it in turn suffices that
$$
m \geq \frac{C'' \bar \ell (2d+6)^{d+4} \log^2 \left(\frac{1}{\ep \delta \beta \bar \eta}\right)}{\ep \bar \eta^2},
$$
where we have used that $\K = \lceil 2/\bar \eta \rceil$, and $C''$ is a sufficiently large constant.

\end{proof}

\begin{proof}[Proof of Claim \ref{clm:soas-good-error}]
  By Claim \ref{clm:reduce-and-filter}, under the event $E_1 \cap E_0$, \RegLearn outputs a class $\hat \ML \in \CR\^j$ for some $1 \leq j \leq m$, which is $\bar \ell$-irreducible. For the remainder of the proof we assume that $E_1 \cap E_0$ holds and fix such a $j$. By Lemma \ref{lem:soafilter-lstar} (with $\chi = 5$), there is some $\hat g \in \CS\^j$ so that $\| \soaf{\hat \ML} - \hat g \|_\infty \leq 12(d+1)$. Set $\hat P\^j := \disc{\hat Q\^j}{\bar \eta}$. By definition, each element $\hat g \in \hat \CS\^j$ is of the form $\soaf{\MF_{\hat P\^j, \alpha_t - 2\alpha_\Delta/3}|_{\ba(v)}}$ for some $1 \leq t \leq d$ and some node $v$ of the tree $\hat \bx$ output by \ReduceTreeReg for which $\MF_{\hat P\^j, \alpha_t - 2\alpha_\Delta/3}|_{\ba(v)}$ is nonempty and $\ell'$-irreducible (see (\ref{eq:redtree-output})). Fix any such element, and write $\MJ := \MF_{\hat P\^j, \alpha_t - 2\alpha_\Delta/3}|_{\ba(v)}$. By definition we have that each $f \in \MJ \subset \MF_{\hat P\^j, \alpha_t - 2\alpha_\Delta/3}$ satisfies, under the event $E_1 \cap E_0$,
  \begin{equation}
    \label{eq:err-pj-ub}
\err{\disc{Q}{\bar \eta}}{f} \leq \err{\hat P\^j}{f} + \alpha_\Delta / 6 \leq \alpha_t - \alpha_\Delta / 2 \leq \alpha_1 - \alpha_\Delta / 2 \leq \inf_{f \in \MF} \left\{\err{\disc{Q}{\bar \eta}}{f}\right\} + (d+1) \alpha_\Delta,
\end{equation}
where the first inequality holds under $E_0$ (see (\ref{eq:allj-unif-conv})) and the final inequality follows from (\ref{eq:event-e10}), which holds under $E_1 \cap E_0$ (in particular, it holds under the event $E_{1,0}$ defined in the proof of Claim \ref{clm:alpha}, which is included in $E_1$).

Recall the definition of finite restriction subclasses of $\MF$ from Section \ref{sec:irreducibility}. Since $\MX$ is countable, the set of all finite restriction subclasses of $\MX$ is countable; thus the set of all finite unions of finite restriction subclasses of $\MF$ is countable as well. Define
$$
\tilde \MF = \MF \cup  \{ \soaf{\MG} :  \substack{\text{$\MG \subset \MF$, $\MG$ is nonempty, $(d+1)$-irreducible, } \\ \text{and a finite union of finite restriction subclasses of $\MF$}}\}.
$$
Then $\tilde \MF$ is countable, and Lemma \ref{lem:irred-sfat-bound} gives that $\fat_2(\tilde \MF) \leq \sfat_2(\tilde \MF) = d$. 

Let $C_0, C_1$ be the constants of Corollary \ref{cor:2fat-uc-disc}, and choose $n_2 \geq C_0 \K^2 \cdot (d \log(\K) + 1)$ (recall $\K = \lceil 2/\bar \eta \rceil$). By Corollary \ref{cor:2fat-uc-disc} applied to the class $\tilde \MF$, we have:
$$
\Pr_{S_{n_2} \sim Q^{n_2}} \left[ \sup_{\tilde f \in \tilde \MF} \left| \err{\disc{Q}{\bar \eta}}{\tilde f} - \err{\disc{\hat Q_{S_{n_2}}}{\bar \eta}}{\tilde f} \right| > C_1 \right] \leq 1/2.
$$
Choose some dataset 
$S_{n_2} \in (\MX \times [-1,1])^{n_2}$ so that $\sup_{\tilde f \in \tilde \MF} \left| \err{\disc{Q}{\bar \eta}}{\tilde f} - \err{\disc{\hat Q_{S_{n_2}}}{\bar \eta}}{\tilde f} \right| \leq C_1$ holds, and write $\hat P := \disc{\hat Q_{S_{n_2}}}{\bar \eta}$ as the discretization of the empirical distribution $\hat Q_{S_{n_2}}$. Then by (\ref{eq:err-pj-ub}), each $f \in \MJ$ satisfies
\begin{equation}
  \label{eq:err-hatp-ub}
\err{\hat P}{f} \leq \inf_{f \in \MF} \left\{ \err{\disc{Q}{\bar \eta}}{f} \right\} + (d+1)\alpha_\Delta + C_1.
\end{equation}

We next claim that $\err{\hat P}{\soaf{\MJ}} \leq  \inf_{f \in \MF} \left\{\err{\disc{Q}{\bar \eta}}{f}\right\} + (d+1) \alpha_\Delta + C_1$. Suppose  for the purpose of contradiction that this is not the case. Let us write $S_{n_2} = \{ (x_1, y_1), \ldots, (x_{n_2}, y_{n_2}) \}$. For $1 \leq i \leq n_2$, write $\tilde y_i := \soa{\MJ}{x_i}$. Since $\MJ$ is $\ell'$-irreducible and the definition of $\ell'$ in step \ref{it:define-params} of \RegLearn ensures $\ell' \geq n_2$, it holds that
$$
\sfat_2(\MJ|_{\{(x_1, \tilde y_1), \ldots, (x_{n_2}, \tilde y_{n_2})\}}) = \sfat_2(\MJ)  \geq 0.
$$
Thus there is some $f \in \MJ$ so that $f(x_i) = \tilde y_i = \soa{\MJ}{\tilde x_i}$ for $1 \leq i \leq n_2$. Thus $\err{\hat P}{\soaf{\MJ}} = \err{\hat P}{f} >  \inf_{f \in \MF} \left\{\err{\disc{Q}{\bar \eta}}{f}\right\} + (d+1) \alpha_\Delta + C_1$, which contradicts (\ref{eq:err-hatp-ub}). Since $\soaf{\MJ} \in \tilde \MF$ (as $\ell' \geq n_2 \geq d+1$), it follows from the choice of $S_{n_2}$ that
\begin{equation}
  \label{eq:j-discq-ub}
\err{\disc{Q}{\bar \eta}}{\soaf{\MJ}} \leq  \inf_{f \in \MF} \left\{\err{\disc{Q}{\bar \eta}}{f}\right\} + (d+1) \alpha_\Delta + 2C_1.
\end{equation}

Recalling that $\| \soaf{\hat \ML} - \soaf{\MJ}\|_{\infty} \leq 12 (d+1)$ and using (\ref{eq:j-discq-ub}), we get that
$$
\err{\disc{Q}{\bar \eta}}{\soaf{\hat \ML}} \leq  \err{\disc{Q}{\bar \eta}}{\soaf{\MJ}} + \| \soaf{\hat \ML} - \soaf{\MJ} \|_\infty \leq \inf_{f \in \MF} \left\{ \err{\disc{Q}{\bar \eta}}{f} \right\} + (d+1)(\alpha_\Delta + 12) + 2C_1.
$$
Finally, using (\ref{eq:disc-cont-rel}) with $\eta = \bar \eta$ and the definition of $\hat h : \MX \ra [-1,1]$ in step \ref{it:define-hath} of \RegLearn (which implies that $\soaf{\hat \ML} = \disc{\hat h}{\bar \eta}$), we get
\begin{align*}
  \err{Q}{\hat h}  & \leq \frac{2(1 + \err{\disc{Q}{\bar \eta}}{\soaf{\hat \ML}})}{\lceil 2/\bar \eta \rceil} \\
                   & \leq  \frac{2(1 + \inf_{f \in \MF} \left\{ \err{\disc{Q}{\bar \eta}}{f} \right\} + (d+1)(\alpha_\Delta + 12) + 2C_1)}{\lceil 2/\bar \eta \rceil} \\
                   & \leq \frac{2(1 + \frac{\lceil 2 / \bar \eta\rceil}{2} \cdot \inf_{h \in \MH} \left\{ \err{Q}{h} \right\} + 1 + (d+1)(\alpha_\Delta + 12) + 2C_1)}{\lceil 2/\bar \eta \rceil} \\
                   &\leq \inf_{h \in \MH} \left\{ \err{Q}{h} \right\} + \bar \eta \cdot (d+2)(\alpha_\Delta + 12) + 2C_1 \bar \eta \\
                   & = \inf_{h \in \MH} \left\{ \err{Q}{h} \right\} + 30 (d+2) \bar \eta + 2C_1 \bar \eta,
\end{align*}
where the last line follows from the choice of $\alpha_\Delta = 18$.
\end{proof}

\end{proof}

\end{document}